\documentclass[12pt]{msml2021} 

\usepackage[toc,page,header]{appendix}
\usepackage{minitoc}

\usepackage{booktabs}       
\usepackage{amsfonts}       

\usepackage{verbatim}

\usepackage{appendix}
\usepackage{dsfont}
\usepackage{xcolor}
\usepackage{bm}
\usepackage[normalem]{ulem}
\usepackage{enumitem}

\definecolor{darkred}{rgb}{0.75,0,0}
\definecolor{darkgreen}{rgb}{0.1,0.6,0.1}
\definecolor{darkblue}{rgb}{0.1,0.1,0.6}
\def\mt#1{{\textcolor{blue}{#1}}}

\def\mtmath#1{\mathbf{\textcolor{darkblue}{#1}}}
\newcommand{\cL}{\mathcal{L}}
\newcommand{\bbE}{\mathbb{E}}

\newcommand{\bbP}{\mathbb{P}}
\newcommand{\bbR}{\mathbb{R}}
\newcommand{\cE}{\mathcal{E}}

\newcommand{\cP}{\mathcal{P}}
\newcommand{\argmin}{\mathrm{argmin}}
\newcommand{\dd}{\mathrm{d}}
\newcommand{\Div}{\mathrm{div}}
\newcommand{\dist}{\mathrm{dist}}
\newcommand{\eps}{\varepsilon}

\newcommand{\iid}{\stackrel{\mathrm{iid}}{\sim}}
\newcommand{\Lip}{\mathrm{Lip}}

\newcommand{\one}{\mathds{1}}

\newcommand{\Vol}{\mathrm{Vol}}

\def\l{\left(}
\def\r{\right)}
\def\la{\left|}
\def\ra{\right|}

\def\lb{\left\{}
\def\rb{\right\}}
\def\rd{\right.}
\def\ls{\left[}
\def\rs{\right]}

\def\Cr#1{\textrm{\textnormal{C}}^{#1}}
\def\Lp#1{\textrm{\textnormal{L}}^{#1}}

\def\cEcon{\cE_{\mathrm{con}}}

\def\vx{{\bm{x}}}
\def\vy{{\bm{y}}}
\def\vz{{\bm{z}}}

\def \cR {{\mathcal{R}}}

\def \RR {{\mathbb{R}}}

\def \Wb {{\mathbf{W}}}

\usepackage{soul}

\title[Robust Certification for Laplace Learning on Geometric Graphs]{
Robust Certification for Laplace Learning on Geometric Graphs}
\usepackage{times}

\msmlauthor{%
 \Name{Matthew Thorpe} \Email{matthew.thorpe-2@manchester.ac.uk}\\
 \addr Department of Mathematics\\
 University of Manchester, Manchester, UK, M13 9PR
 \AND
 \Name{Bao Wang} \Email{wangbaonj@gmail.com}\\
 \addr Department of Mathematics\\
 Scientific Computing and Imaging Institute\\
 University of Utah, Salt Lake City, UT, 84112
}

\makeatletter
 \let\Ginclude@graphics\@org@Ginclude@graphics
\makeatother

\begin{document}

\maketitle

\begin{abstract}%
Graph Laplacian (GL)-based semi-supervised learning is one of the most used approaches for classifying nodes in a graph. Understanding and certifying the adversarial robustness of machine learning (ML) algorithms has attracted large amounts of attention from different research communities due to its crucial importance in many security-critical applied domains. There is great interest in the theoretical certification of adversarial robustness for popular ML algorithms. In this paper, we provide the first adversarial robust certification for the GL classifier. More precisely we quantitatively bound the difference in the classification accuracy of the GL classifier before and after an adversarial attack.
Numerically, we validate our theoretical certification results and show that leveraging existing adversarial defenses for the $k$-nearest neighbor classifier can remarkably improve the robustness of the GL classifier.
\end{abstract}

\begin{keywords}%
Graph Laplacian; Semi-supervised learning; Robust certification
\end{keywords}

\section{Introduction} \label{sec:Intro}
Let {$\Omega_N := \{\vx_i\}_{i=1}^N \subset \RR^d$} be a set of feature vectors with a subset of {$\Gamma_N := \{\vx_i\}_{i\in Z_N \subset [N]}$} 
being labeled. If {$i\in Z_N$} then {$\vx_i$} is labeled {$\ell(\vx_i) \in \bbR$} 
and we denote {$\ell_N := \ell |_{\Gamma_N}$}. 
The Graph Laplacian (GL) framework encodes the geometry of the feature vectors {$\Omega_N$} by constructing an undirected graph, {$G_N = (\Omega_N, \Wb_{N})$}, where {$\Omega_N$} forms the nodes of the graph and {$\Wb_{N} := (\Wb_{\vx, \vy})_{\vx, \vy \in \Omega_N}$} is the set of edge weights with {$\Wb_{\vx, \vy}$} being the weight of the edge between {$\vx$} and {$\vy$}. 
The graph Dirichlet energy is defined by
{\[ {\cE(u; \Omega_N) = \sum_{\vx, \vy\in \Omega_N} \Wb_{\vx, \vy} (u(\vx)-u(\vy))^2,} \]}%
where $u$ is a function defined on the nodes $\Omega_N$ of the graph. 
We can then predict the label for the unlabeled data by solving the following constrained energy minimization problem
{\begin{equation} \label{eq:Setting:Variational:Var}
{\text{minimize } \cE(u; \Omega_N) \text{ over } u:\Omega_N\to\bbR \text{ subject to } u(\vx) =\ell_N(\vx) \, \forall \vx\in\Gamma_N.}
\end{equation} }%
Laplacian regression is the solution to \eqref{eq:Setting:Variational:Var}.
To go from regression to (binary) classification one thresholds $u$, e.g. if the classes are represented by $\{0,1\}$ then the GL classifier predicts the label
$1$ if {$u(\vx)\geq 1/2$}, and $0$ otherwise. 
{Note that the GL classifier classifies any unlabeled data leveraging both labeled and unlabeled data.} As a comparison, for any unlabeled $\vx$, the $k$-nearest neighbor ($k$NN) classifier classifies $\vx$ with the most common label amongst its {labeled} nearest neighbors.


The GL classifier has been successfully used for semi-supervised data classification \citep{wang2006semi,zhou2004learning,zhu03}, image processing \citep{buades2006neighborhood,gilboa2009nonlocal,shi17}, improving robustness and accuracy of deep neural nets (DNNs) \citep{wang2018deep,wang2019graph}, etc. Direct application of GL classification with Gaussian \citep{belkin2004semi} or locally linear embedding weights \citep{roweis2000nonlinear} for the above tasks may cause inference inconsistency in the low labeling ratio regime. To resolve this dilemma, many regularisation strategies have been developed to adapt GL to the ultra-low ratio of the labeled training data, e.g., scaling the weights \citep{shi17,shi2018error} of the labeled data and the $p$-Laplacian \citep{calder18,rios2019algorithms,zhou05}.

Despite the tremendous success of machine learning (ML) algorithms, they are generally vulnerable to adversarial attacks \citep{szegedy2013intriguing}. 
The adversarial vulnerability of ML algorithms raises concerns in applications to security-critical domains, such as: autonomous cars \citep{Akhtar:2018,Attack:Tesla}, 
medical imaging \citep{finlayson2019adversarial},
and national defense \citep{hoadley2018artificial}. 
Many algorithms have been recently proposed to improve robustness of ML including adversarial training \citep{Goodfellow:2014AdversarialTraining,madry2017towards}, augmenting training data with unlabeled instances \citep{carmon2019unlabeled}, and noise injection 
\citep{wang2019resnets}.
Nevertheless, there is a lack of theoretical understanding of adversarial issues of ML models. In this paper, we focus on \emph{theoretical analysis of the conditions that guarantee adversarial robustness of the GL classifier for semi-supervised learning (SSL).}

\subsection{Our Contribution}
{A classifier is said to be certifiably robust in classifying $\vx$, if the classification result remains constant provided the perturbation on $\vx$ is within a ball, e.g., $\ell_2$-ball, of radius $r$. In this paper, we provide the first certification 
of the adversarial robustness of the GL classifier under the $\ell_2$-norm. 
Our theory shows that within a certain adversarial attack regime, the GL classifier with} {$O(k)$ edges per node} 
{is intrinsically more robust than the $k$NN classifier. We show that to achieve certified robustness, the GL method needs significantly} {fewer} {nearest neighbors, with a small computational overhead.} 
{Our theoretical result resonates with the finding that unlabeled data can improve the robustness of ML algorithms \citep{carmon2019unlabeled} and provides a feasible avenue to explain the observation that GL-based activation function remarkably improves DNNs' robustness \citep{wang2019graph}.} 
We summarize these high probability results in Table~\ref{tab:Theory:Summary}, where $N$ and $M$ are the total number of data and the number of unlabeled data respectively, $k$ is the 
number of nearest neighbors involved in $k$NN and the approximate order of edges per node for the GL classifier, $r$ is the maximum allowed adversarial perturbation measured in the $\ell_2$-norm, and $\kappa$ is the condition number of the matrix $\Wb_{N}$. 
We point out, however, that the results for the GL classifier in Table~\ref{tab:Theory:Summary} are a special case and in particular one can reduce the number of neighbors $k$ at the cost of reducing the probability (going from high probability bounds to low probability bounds).  
Note that if a constant fraction of the data is labeled i.e., {$(N-M)/N $} is constant, then {$k=\Omega(\log N)$} for the GL classifier.
We will numerically verify these theoretical results with the existing benchmark experiments in Section~\ref{sec:Exp}.
{More detail on how we extracted these bounds from our theoretical results is given in Remark~\ref{rem:Tab1Bounds}.}
\begin{table}[!ht]
\fontsize{9.5}{9.5}\selectfont
\centering
\caption{High probability robustness guarantees and computational complexity of GL vs. $k$NN.
}
\label{tab:Theory:Summary}
\centering
\begin{tabular}{ccccc}
\toprule
Classifier  & $k$ & Assumption on $r$ & Computational Complexity & Reference\\
\midrule
$k$NN  & \boldsymbol{$\Omega(\sqrt{N\log N})$} & None  & $O(kM\log (N-M))$  & \cite{wang2017analyzing}\cr
GL   & \boldsymbol{$\Omega\left(\frac{N\log N}{N-M}\right)$} & $r \leq c\sqrt{\frac{N-M}{N}}\left(\frac{\log N}{N-M}\right)^{\frac{1}{d}}$  & $O(kN\log N + Nk\sqrt{\kappa})$  &  {\bf This Work} \cr
\bottomrule
\end{tabular}
\end{table}

\subsection{Additional Related Works} 
{
The first theoretical characterisation of the number of nearest neighbors required for a robust $k$NN classifier appeared in~\citep{wang2017analyzing}, where the authors also proposed a robust one nearest neighbor approach.
We apply the robust characterisation used by \cite{wang2017analyzing} and develop a robust certification for the GL classifier in SSL.}




To prove robustness, we connect with large data results and we mention several here. When the labeling rate is low Laplacian regularisation becomes degenerate and the label function $u$ becomes nearly constant with sharp spikes at the labeled points~\citep{elalaoui16,nadler09,slepcev19}.
The degeneracy can be avoided by either using $p$-Dirichlet energies, with $p>d$~\citep{calder18AAA,elalaoui16,slepcev19}, by increasing the label rate $(N-M)/N$~\citep{calder20AAA}, or by reweighting the Laplacian in order to gain more regularity~\citep{calder18bAAA,shi2018error}.
Similar results hold for the game theoretic $p$-Laplacian~\citep{calder18,calder18AAA}.
In addition, pointwise convergence of Laplacians has been considered several times, for example~\citep{belkin2007convergence,calder18,calder20AAA,SpecRatesTrillos,GTSSpectralClustering,hein05,singer06}.

\subsection{Organization}
We organize this paper as follows: In Section~\ref{sec:Setting}, we present the main theory on the {certified} robustness of the GL classifier. In Section~\ref{sec:Complexity}, we analyze the computational complexity of the GL classifier. We verify the robustness of the GL classifier in different settings and compare it with {the} $k$NN {classifier} in Section~\ref{sec:Exp}. This paper ends with some concluding remarks in Section~\ref{sec:Conc}. Technical proofs and some more experimental details and results are provided in the appendix.

\subsection{Notation}
We denote 
vectors/matrices by lower/upper case bold face letters. 
Given two sequences $\{a_n\}$ and $\{b_n\}$, we write $a_n=O(b_n)$ if there exists a positive constant $C$ 
such that $a_n \leq C b_n$; and $a_n = \Omega(b_n)$ if for large enough $n$, $b_n$ is at least $ca_n$ for some constant $c$. 
{Throughout $0<c\leq C<+\infty$ will be arbitrary constants (independent of data realisations and all other parameters but possibly depending on dimension and the density of the data generating distribution) and may change value from line-to-line.}
We denote the set $\{1, 2, \cdots, N\}$ by $[N]$. 


\section{Main Theory}\label{sec:theory} \label{sec:Setting}
\subsection{Preliminaries and Assumptions}\label{subsec:Setting:Variational}
To {certify} the 
robustness of the GL classifier, we make the following assumptions on the dataset:
\begin{enumerate}
\item[(A1)] $\Omega\subset \bbR^d$ is open connected and bounded with Lipschitz boundary;
\item[(A2)] $\vx_i\iid\mu\in\cP(\Omega)$ where $\mu$ has density $\rho\in \Cr{2}(\Omega)$ that is bounded below by a positive constant, i.e. $\inf_{\vx\in\Omega}\rho(\vx)=:\rho_{\min} >0$;
\item[(A3)] $\bbP(\vx\in \Gamma_N|\vx\in\Omega_N) = \bbP(i\in Z_N|\vx_i\in \Omega_n) = \beta$ and if $\vx\in\Gamma_N$, then $\vx$ is labeled as $\ell(\vx)$ for a Lipschitz function $\ell:\Omega\to\bbR$. 
\end{enumerate}

{It will be convenient to define $\ell_N = \ell|_{\Gamma_N}$.} Note that $\beta$ is the {probability of a data point being labeled} and so (in the notation of Table~\ref{tab:Theory:Summary}) {we can make the formal association $\beta \sim (N-M)/N$.} 
For convenience, we introduce the following constrained graph Dirichlet energy functional:
\[ {\cEcon(u; D_N) = \lb \begin{array}{ll} \cE(u; \Omega_N) & \text{if } u(\vx) = \ell_N(\vx) \, \forall \vx\in\Gamma_N \\ +\infty & \text{else.} \end{array} \rd} \]
The Euler-Lagrange equation corresponding to minimizing $\cEcon(\cdot; D_N)$ is 
{\begin{align*}
\cL_N(u; \Omega_N)(\vx) & = 0 && \text{for } \vx\in \Omega_N\setminus\Gamma_N \\
u(\vx) & = \ell_N(\vx) && \text{for } \vx\in\Gamma_N,
\end{align*}}
where $\cL_N(\cdot; \Omega_N)$ is the graph Laplacian defined by
\[{ \cL_N(u; \Omega_N)(\vx) = \sum_{\vy\in\Omega_N} \Wb_{\vx, \vy} \l u(\vx) - u(\vy)\r.} \]

We have made explicit the dependence of the domain $\Omega_N$ on the functionals $\cE$, $\cEcon$ and the operator $\cL_N$. Although this notation may feel cumbersome at this stage, it will aid clarity when we have two sets of data; the original dataset $\Omega_N$ and the (adversarially-) perturbed dataset $\hat{\Omega}_N$.

{We will consider \emph{Geometric Random} graphs.
This construction involves weighting edges between all pairs of nodes as a function of the distance between nodes (and we say there is no edge between two nodes if the edge weight is zero).
We use a parameter $\eps$, which is often chosen relative to $N$, to control the length scale in the graph.
This is summarised below:}
\begin{enumerate}
\item[(A4)] 
{$\Wb_{\vx, \vy}=\Wb_{\eps, \vx, \vy}$ where}
$\Wb_{\eps, \vx, \vy} = \eta_\eps(|\vx-\vy|)$ and $\eta_\eps = \frac{1}{\eps^d}\eta(\cdot/\eps)$ and $\eta:[0,+\infty)\to[0,+\infty)$ is non-increasing, positive, $\eta(t)\geq 1$ for all $t\leq 1$ and $\eta(t)=0$ for all $t\geq 2$.
In addition, either $\eta$ is Lipschitz continuous, or $\eta(t) = \one_{t\leq 1}$.
\end{enumerate}

{We note that whilst we use the geometric random graph construction in (A4) and we use the $k$NN graph in our experiments. The parameters $k$ and $\eps$ are related as follows $k\sim N\eps^d$ (cf Lemma~\ref{lem:app:BallBound}). There are additional technical challenges when addressing the $k$NN constructions, however, we believe our results carry through to this setting (see also Remark~\ref{rem:Tab1Bounds} below).}


The assumptions in (A4) allow us to bound the degrees of nodes and, letting $\hat{\vx},\hat{\vy}$ be the adversarial perturbations of $\vx, \vy$, show that either (i) $\Wb_{\vx, \vy}$ is always close to $\Wb_{\hat{\vx},\hat{\vy}}$ (when $\eta$ is Lipschitz) or (ii) we can control the number of $\vx, \vy$ such that $\Wb_{\vx, \vy}$ is not close to $\Wb_{\hat{\vx}, \hat{\vy}}$ (when $\eta = \one_{\cdot\leq 1}$).

\subsection{Robustness of Semi-Supervised Learning with Graph Laplacian}\label{subsec:Setting:Robust}
In this subsection, we give a theoretical bound of the following question: 
{\emph{how is the classification estimate affected if an adversary replaces the clean dataset $D_N=(\Omega_N,\ell_N)$ with a new, corrupted, dataset $\hat{D}_N=(\hat{\Omega}_N,\hat{\ell}_N)$?}} Following~\cite{wang2017analyzing}, we assume that the adversary can corrupt features by adding a small perturbation to the unlabeled data; the question of robustness under poisoning attacks \citep{dalvi2004adversarial,lowd2005good} is an interesting question we leave open.
We assume the adversary can corrupt the unlabeled data by moving each point a maximum distance of $r$ in $\ell_2$-norm. That is, the adversary can replace the set $\Omega_N$ with a corrupted dataset $\hat{\Omega}_N$ by, for each $i=1, \dots, N$, choosing $\hat{\vx}_i\in B(\vx_i, r)$ thus defining $\hat{\Omega}_N = \{\hat{\vx}_i\}_{i=1}^N$. Here, and in the sequel, $\hat{\vx}, \hat{\vx}_i$ is understood to be a perturbation of $\vx, \vx_i$, respectively. Although the labels are not perturbed, the domain of the labeling function $\ell_N$ is, i.e. the perturbed domain is $\hat{\Gamma}_N = \{\hat{\vx}_i\}_{\{i\,:\,x_i\in\Gamma_N\}}$, and so we define $\hat{\ell}_N:\hat{\Gamma}_N\to \bbR$ by $\hat{\ell}_N(\hat{\vx}) = \ell_N(\vx)$ for all $\hat{\vx}\in\hat{\Gamma}_N$.
{Note that $\hat{\ell}_N(\hat{\vx}) = \ell_N(\vx)$ is precisely the condition that the adversary doesn't corrupt labels.}

A \emph{learning strategy} is a map from the dataset $D_N=(\Omega_N,\ell_N)$ to a function $u:\Omega_N\to \bbR$. For example, in the previous section we defined the learning strategy 
\begin{equation} \label{eq:Setting:Robust:LapStrat}
{
D_N=(\Omega_N, \ell_N)\mapsto u(\cdot; D_N) := \argmin_{u: \Omega_N\to\bbR}\cEcon(\cdot; D_N).
}
\end{equation}
This is the learning strategy we will analyse.

Given a dataset $D_N = (\Omega_N, \ell_N)$ and a perturbation $\hat{D}_N=(\hat{\Omega}_N, \hat{\ell}_N)$ we will compare $u(\cdot; D_N)$ with $u(\cdot; \hat{D}_N)$ by $|u(\vx; D_N) - u(\hat{\vx}; \hat{D}_N)|$. The $\Lp{\infty}$ 
distance between $u(\cdot; D_N)$ and $u(\cdot; \hat{D}_N)$ can be defined as $\max_{\vx\in\Omega_N} |u(\vx; D_N) - u(\hat{\vx}; \hat{D}_N)|$.

We let $\delta>0$ be a prescribed tolerance then the robustness radius is the smallest $r$ such that it is possible to perturb $u(\cdot; D_N)$ by more than $\delta$. More precisely, we define the $\delta$-robustness radius below which is a modification of the robustness radius in \citep{wang2017analyzing}.
\begin{definition}
{\bf $\delta$-Robustness Radius.}
\label{def:Setting:Robust:RobRadCorDat}
{Let $D_N\mapsto u(\cdot; D_N)$ be a learning strategy.
The $\delta$-robustness radius $\cR_\delta(\Omega^\prime, u, D_N)$ of $u$ over a subset $\Omega^\prime\subset \Omega$ given the data $D_N$ is the smallest radius $r$ such that $\sup_{\vx\in\Omega_N} |u(\vx; D_N)-u(\hat{\vx}; \hat{D}_N)|>\delta$ where $|\hat{\vx}-\vx|<r$ for all $\vx\in\Omega^\prime$, i.e.}
\[ {\cR_\delta(\Omega^\prime, u, D_N) = \inf_{r>0} \lb\forall \vx_i\in\Omega_N\cap\Omega^\prime\, \exists \hat{\vx}_i \in B(\vx_i, r) \text{ s.t. } \sup_{\vx\in\Omega_N} |u(\vx; D_N)-u(\hat{\vx}; \hat{D}_N)|>\delta \rb.} \]
\end{definition}

We prove $\delta$-robustness over $\Omega^\prime$ in order to avoid problems at the boundary $\partial\Omega$.
In particular, we take $\Omega^\prime$ such that $\dist(\Omega^\prime,\partial\Omega)$ is sufficiently large.
We believe our arguments can be extended to the boundary but the techniques to do so are more involved and will involve estimates between the GL and its continuum analogue at the boundary.
{In particular, our proof uses a bound between the graph Laplacian and its continuum analogue, for which there are quantitative bounds away from the boundary, e.g.~\cite{singer06,calder18}.
Near the boundary the bound between the graph Laplacian and its continuum counterpart deteriorates to $O(1)$, i.e. there are currently no established rates of convergence close to the boundary, see~\cite{calder20AAA}.}

Our main theoretical results are the following, the proofs can be found in Appendix~\ref{Appendix:proof}.


\begin{theorem}
{\bf $\delta$-Robustness of GL-based 
Regression.}
\label{thm:Setting:Robust:RobRadCorDat}
{Under Assumptions (A1-A4)}
define $u$ by~\eqref{eq:Setting:Robust:LapStrat}.
There exists constants $C_0>0$, $\eps_0>0$, $C>c>0$ such that if $\eps\in (0,\eps_0)$, $r\in (0,r_{\max})$ where $r_{\max} = c\sqrt{\beta}\eps$, $\beta\in[\eps^2,1]$ and $\Omega^\prime\subset\Omega$, with $\dist(\Omega^\prime,\partial\Omega)>C_0\beta^{-\frac12}\eps\log\l\beta^{\frac12}\eps^{-1}\r$, then $\cR_\delta(\Omega^\prime,u,D_N) \geq r$ with probability at least $1-CNe^{-cN\beta\eps^d}$ where
\begin{equation} \label{eq:MainRes:delta}
{\delta = \frac{C\eps}{\sqrt{\beta}} \log\l\frac{\sqrt{\beta}}{\eps}\r.}
\end{equation}
\end{theorem}

\begin{remark}
\label{rem:Tab1Bounds}
The comparison with $k$NN given in Table~\ref{tab:Theory:Summary} can be derived from the above theorem as follows.
With probability at least $1-CNe^{-cN\eps^d}$ the number of neighbors in an $\eps$ connected graph scales as $N\eps^d$ (cf Lemma~\ref{lem:app:BallBound}); hence $k\sim N\eps^d$.
Now to achieve a high probability convergence rate we require that $(N\beta\eps^d)/(\log N)$ is large, which gives a lower bound on $\eps$.
Choosing $\eps$ as small as possible then implies that $k\gg (\log N)/\beta$.
Since $\beta\sim(N-M)/N$ then we arrive at the form of the bound stated in Table~\ref{tab:Theory:Summary}. 
{Moreover, we believe the above theorem can be generalised to include the $k$NN graph construction: $\Wb_{\vx,\vy} = \Wb_{N,k,\vx,\vy}$ where $\Wb_{N,k,\vx,\vy} = \frac{N}{k}\one_{\vx\sim_k\vy}$ and $\one_{\vx\sim_k\vy}=1$ if $\vx$ is a $k$NN of $\vy$ (or vice versa) and $\one_{\vx\sim_k\vy}=0$ otherwise.
Formally, we conjecture that if one substitutes $\eps = \l\frac{k}{N}\r^{\frac{1}{d}}$ then Theorem~\ref{thm:Setting:Robust:RobRadCorDat} continues to hold with $k$NN weights, i.e. $\cR_\delta(\Omega^\prime,u,D_N) \geq r$ with probability at least $1-CNe^{-ck\beta}$ where
{\[ \delta = \frac{C k^{\frac{1}{d}}}{N^{\frac{1}{d}}\sqrt{\beta}} \log\l\frac{N^{\frac{1}{d}}\sqrt{\beta}}{k^{\frac{1}{d}}}\r. \] }}
\end{remark}

\begin{remark}
{Theorem~\ref{thm:Setting:Robust:RobRadCorDat} shows the $\delta$-robustness of GL-based regression up to { $r_{\max}= c\sqrt{\beta}\eps$}.
We can restate this in terms of the number of labels, {$N-M$}, by using the formal scaling {$\beta\sim{(N-M)}/{N}$}, so that {$r_{\max}= c\eps\sqrt{{(N-M)}/{N}}$}. In particular, the number of labels increases the $\delta$-robustness following a square-root law.}
\end{remark}

Typically, one uses Laplacian regularisation for labeling by projecting the solution $u$ of~\eqref{eq:Setting:Variational:Var} onto the set of labels.
For simplicity we consider the binary  classification problem, that is we seek a function $v:\Omega_N\to\{0,1\}$ where $0$ and $1$ are the two classes.
As is common, we define
\[{ v(\vx;D_N) = \lb \begin{array}{ll} 1 & \text{if } u(\vx;D_N)\geq\frac12 \\ 0 & \text{else.} \end{array} \rd} \]

\begin{corollary}
\label{cor:MainRes:LapClass}
Let $\delta$ be given by~\eqref{eq:MainRes:delta}. 
In addition to the assumptions of Theorem~\ref{thm:Setting:Robust:RobRadCorDat} we assume
\begin{equation} \label{eq:MainRes:ExAss}
{\sup_{\xi>0} \frac{1}{\xi} \Vol\l\lb \vx \,:\, \frac12 - \xi \leq \ell(\vx) \leq \frac12 + \xi \rb\r \leq A,}
\end{equation}
for some $A>0$, then there exists $\Omega_\delta\subset \Omega$ such that $\mu(\Omega\setminus\Omega_\delta)\leq C\delta$ and $\cR_0(\Omega_\delta,v,D_N)\geq r$ with probability at least $1-CNe^{-cN\beta\eps^d}$.
\end{corollary}

The classification decision boundary is $\{\vx\,:\,u(\vx)=\frac12\}$, which (c.f. Theorem~\ref{thm:MainRes:LargeDataLimit}) is approximately the set $\{\vx\,:\,\ell(\vx)=\frac12\}$.
The additional assumption in equation~\eqref{eq:MainRes:ExAss} is in order to ensure that the set where $\ell$ is close to $\frac12$ can be controlled.
When $\ell(\vx)$ is sufficiently far from $\frac12$ then we obtain $v(\vx;D_N) = v(\hat{\vx};\hat{D}_N)$.


The proof of Theorem~\ref{thm:Setting:Robust:RobRadCorDat} and Corollary~\ref{cor:MainRes:LapClass} is given in Appendix~\ref{Appendix:proof}, and relies on a quantitative bound between solutions of~\eqref{eq:Setting:Variational:Var} and the true function $\ell$. In particular, if the data points are close to being iid then we can use the result in~\cite{calder20AAA} to infer a high probability bound between $u(\cdot;D_N)$ and $\ell$.
Our proof shows that if $r<r_{\max}$ then we can consider the perturbed data points $\hat{\Omega}_N=\{\hat{\vx}_i\}_{i=1}^N$ to be close to iid and hence apply the result to infer a high probability bound between $u(\cdot;\hat{D}_N)$ and $\ell$. 
In fact, we can show the following result, and the proof is also given in  Appendix~\ref{Appendix:proof}.

\begin{theorem}
\label{thm:MainRes:LargeDataLimit}
Under the assumptions of Theorem~\ref{thm:Setting:Robust:RobRadCorDat} we have that $\max_{\hat{\vx}\in\hat{\Omega}_N}|u(\hat{\vx};\hat{D}_N) - \ell(\vx)| \leq \delta$ with probability at least $1-CNe^{-cN\beta\eps^d}$ where $\delta$ is given by~\eqref{eq:MainRes:delta}.
\end{theorem}

Theorem~\ref{thm:Setting:Robust:RobRadCorDat} quantifies the robustness of GL regression. More precisely, it establishes an upper bound of the maximum allowed adversarial perturbation under which the GL solution, $u$, of \eqref{eq:Setting:Robust:LapStrat} is close to the ground truth label function $\ell(\vx)$ with high probability. Furthermore, Theorem~\ref{thm:MainRes:LargeDataLimit} implies that after an adversarial attack the solutions to GL regression remain close to the true solution with a quantifiable bound. In Corollary~\ref{cor:MainRes:LapClass} we infer the robustness of GL classification.

\section{Computational Complexity Analysis}\label{sec:Complexity}
We ignore the common pre-processing time for both $k$NN and GL.
The computational complexity of $k$NN is dominated by nearest neighbor searching, 
and the total computational complexity of searching for the nearest neighbors for all unlabeled points is $O(kM\log (N-M) )$ 
\citep{muja2014scalable}.
For the GL classifier, if we use the top $k (k\ll N)$-nearest neighbors, 
the total computational complexity for constructing the weight matrix would be $O(kN\log N)$. The additional computational complexity of GL comes from solving a sparse linear system of the size $N\times N$, which can be solved by using the conjugate gradient method in $O(Nk\sqrt{\kappa})$ time, with $\kappa$ being the condition number of $\Wb_{N}$ \citep{shewchuk1994introduction}. Hence, the total computational complexity of the GL classifier is $O(kN\log N + Nk\sqrt{\kappa})$. Table~\ref{tab:CPU:time-memory} lists a comparison of CPU time and peak RAM consumption for MNIST 1v7 classification with different numbers of nearest neighbors ($k$) being used, and all the experiments are done on an Intel(R) Xeon(R) CPU E5-P2690 0 @ 2.90GHz. We provide the detailed experimental settings in 
Section~\ref{sec:Exp}. 
GL is slightly more computationally expensive than $k$NN, but for the most used $k$ the computational overhead is not an obstacle. Moreover, GL classifier can achieve at least comparable results to $k$NN with a much smaller $k$.

\begin{table}[t!]
\fontsize{9.0}{9.0}\selectfont
\centering
\caption{CPU time and memory cost of GL vs. $k$NN classifiers for MNIST 1v7 classification (without attack), and the dataset is described in Subsection~\ref{subsec:datasets}. 
}
\label{tab:CPU:time-memory}
\centering
\begin{tabular}{c|cccc|cccc}
\toprule
Classifier & \multicolumn{4}{c}{$k$NN} & \multicolumn{4}{c}{GL} \\
\midrule
Number of Nearest Neighbors (k) & 5 & 10 & 15 & 20 & 5 & 10 & 15 & 20 \\
\midrule
Peak Memory (MB) &\ \ \  382\ \ \  &\ \ \  382\ \ \  &\ \ \  382\ \ \  &\ \ \  382\ \ \  &\ \ \  439\ \ \  &\ \ \  445\ \ \  &\ \ \  448\ \ \  &\ \ \  449\ \ \ \\
CPU Time (second) & 2.55 & 2.62 & 2.70 & 2.89 & 2.99 & 3.27 & 3.89 & 5.07\\
\bottomrule
\end{tabular}
\end{table}

\section{Experiments} \label{sec:Exp}
We consider performance of GL classifier and its enhanced variants in classifying different datasets under adversarial attacks to numerically validate: 1) our certification results in Table~\ref{tab:Theory:Summary}, 2) the efficacy of adversarial defenses, and 3) the advantages of the GL classifier over the $k$NN classifier. In all experiments below, we construct $\Wb_N$ in the same way as that used in \cite{shi17}.

\subsection{Datasets, Classifiers, \& Attacks}\label{subsec:datasets}
We use the {same} benchmark datasets, adversarial attacks, and defenses as {those} used in \cite{wang2017analyzing}, where the authors analyzed robustness of $k$NN. Below we give a brief description of these baselines.
\paragraph{Datasets.} We consider three benchmarks: Halfmoon, MNIST 1v7, and Abalone \citep{Dua:2019}. Halfmoon is a randomly generated 2D synthetic dataset in which we randomly generate 2000 points and 1000 points, with a standard deviation of 0.2, as the training and test set, respectively. For the MNIST 1v7, we randomly select 500 images for each of digit 1 and 7 to form the training set and the same size for the test set. For Abalone, we use 500/100 samples for training/test. We also generate a validation set for each of the three benchmarks with the same size as the test set.

\paragraph{Adversarial Attacks.} 
We apply the same white-box (WB) and black-box (BB) attacks as that used in \cite{wang2017analyzing}. In particular, we consider the following two WB attacks:
\begin{itemize}[leftmargin=*]
\item {\bf Direct attack (DA).} Given a perturbation of $\ell_2$ magnitude $r$ and the training dataset $S$ (which might be an augmentation or pruning of $D_N$), the adversarial example of $\vx$ is $\vx_{\rm adv} = \vx + r(\vx-\vx')/|\vx-\vx'|$, where $\vx'$ is the nearest neighbor of $\vx$ in $S$ that is labeled differently from $\vx$. {In our experiments, we vary the value of $r$ to change the maximum $\ell_2$ norm of the adversarial perturbation.}

\item {\bf Kernel substitution attack (KSA).} {KSA} attacks a surrogate kernel classifier, trained on the same training set as that of GL classifier, using the fast gradient sign method {with the target maximum $\ell_2$ norm of the adversarial perturbation, see}  \citep{Goodfellow:2014AdversarialTraining}.
\end{itemize}

In {the} BB attacks, we adopt three substitute classifiers: kernel classifier (Kernel), logistic regression (LR), and neural net (NN). The adversary trains these substitute models using the method of \cite{papernot2017practical} and generate adversarial examples by attacking surrogate classifiers. We use the same setting as that used in \cite{wang2017analyzing} for both WB and BB attacks.

\subsection{{An Approximated Numerical Robust Certification for the GL Classifier}}
According to Table~\ref{tab:Theory:Summary}, GL classifier is robust when the number of nearest neighbors that used to construct $\Wb_N$ satisfies {$k\geq C(N\log N)/(N-M)$} and {$r\leq c \sqrt{(N-M)/N}\sqrt[d]{(\log N)/(N-M)}$}, where $c$ and $C$ are two constants that are independent of $N$ and $M$; we will numerical verify this in this subsection. In particular, we first perform a grid search with a small amount of labeled data, say one-fifth of the whole labeled data, for each benchmark to estimate the parameters $c$ and $C$. Then, we apply the obtained $c$ and $C$ to compute the smallest number of nearest neighbors $k$ and the largest perturbation $r$ such that the GL classifier is guaranteed to be robust.
Finally, we apply the above five attacks to the GL classifier with twenty different values of maximum perturbation uniformly sampled from $[0, r]$.
For each attack value we do 20 independent runs. 
Table~\ref{tab:certify-on-Abalone}
show{s} the accuracies of GL classifier on three benchmarks and the theoretical values (assisted with grid search on small amount of labeled data) of $k$ and $r$ with different amount of labeled data $N-M$. These results show that for any benchmark with a given amount of labeled data and $k$ that is used to construct $\Wb_N$, the GL classifier's accuracy is consistent under different adversarial attacks using different attack strengths, and these results confirm our theoretical robustness of GL classifier.

\renewcommand\arraystretch{1.3}
\setlength{\tabcolsep}{4pt}
\begin{table*}[t!]
\caption{Accuracies of GL classifier for Abalone, Halfmoon, and MNIST classification with different number of labeled data $N-M$. For different $N-M$, the GL classifier is robust when the adversarial attack does not exceed $r$ provided the number of nearest neighbors that used to construct $\Wb_N$ is at least $k$. (Unit: \%)}\label{tab:certify-on-Abalone}
\centering
\fontsize{9.5pt}{0.5em}\selectfont
\begin{tabular}{ccccccccc}
\hline
$N-M$ & $k$ & $r$ & No Attack & WB-DA & WB-KSA & BB-LR & BB-Kernel  & BB-NN \\
\hline
\multicolumn{9}{c}{Abalone}\\
\hline
100 & 18 & 0.0186 & $64.7\pm 0.48$ & $64.6\pm 0.53$ & $64.2\pm 1.61$ & $64.2\pm 0.66$ & $64.5\pm 0.87$ & $64.3\pm 0.43$\\
200 & 14 & 0.0196 & $69.6\pm 0.72$ & $69.2\pm 0.92$ &$69.8\pm 0.85$ & $69.7\pm 0.93$ & $68.6\pm 1.55$ & $68.9\pm 0.86$\\
300 & 13 & 0.0198 & $67.5\pm 0.36$& $67.1\pm 0.51$ & $67.1\pm 1.36$ & $66.6\pm 0.39$ & $68.1\pm 1.18$ & $67.1\pm 0.82$\\
400 & 13 & 0.0197 & $69.5\pm 0.87$ & $69.0\pm 0.93$ & $69.1\pm 1.50$ & $69.8\pm 1.01$ & $68.8\pm 1.72$ & $69.4\pm 1.13$\\
500 & 13 & 0.0196 & $70.2\pm 0.99$ & $70.1\pm 1.09$ & $69.8\pm 0.91$ & $69.8\pm 0.82$ & $69.4\pm 0.88$ & $69.8\pm 1.21$\\
\hline
\multicolumn{9}{c}{Halfmoon}\\
\hline
400 & 22  & 0.0216 & $97.0\pm 0.08$ & $96.8\pm 0.05$ & $96.9\pm 0.12$ & $96.6\pm 0.07$ & $96.8\pm 0.09$ & $96.6\pm 0.11$\\
800 & 15  & 0.0195 &$96.6\pm 0.15$ & $96.7\pm 0.17$ & $96.4\pm 0.08$ & $96.7\pm 0.08$ & $96.5\pm 0.10$ & $96.6\pm 0.14$\\
1200 & 13 & 0.0177 &$96.8\pm 0.06$ & $96.7\pm 0.11$ & $96.8\pm 0.05$ & $96.7\pm 0.11$ & $97.0\pm 0.16$ & $97.0\pm 0.06$\\
1600 & 11 & 0.0165 &$96.7\pm 0.07$ & $96.5\pm 0.09$ & $96.7\pm 0.16$ & $96.9\pm 0.08$ & $96.7\pm 0.18$ & $96.7\pm 0.08$\\
2000 & 10 & 0.0156 &$96.6\pm 0.10$ & $96.2\pm 0.07$ & $96.6\pm 0.10$ & $96.4\pm 0.15$ & $96.5\pm 0.08$ & $96.6\pm 0.16$\\
\hline
\multicolumn{9}{c}{MNIST}\\
\hline
200 & 19 & 0.407 &$98.9\pm 0.11$ & $98.7\pm 0.12$ & $98.8\pm 0.18$ & $98.8\pm 0.09$ & $98.8\pm 0.05$ & $98.8\pm 0.11$\\
400 & 11& 0.532 &$98.7\pm 0.11$ & $98.7\pm 0.15$ & $98.6\pm 0.13$ & $98.7\pm 0.14$ & $98.6\pm 0.12$ & $98.7\pm 0.05$\\
600 & 9 & 0.609 &$98.8\pm 0.09$ & $98.8\pm 0.11$ & $98.9\pm 0.10$ & $98.8\pm 0.14$ & $98.9\pm 0.13$ & $98.8\pm 0.11$\\
800 & 7 & 0.663 &$99.0\pm 0.13$ & $98.7\pm 0.10$ & $99.0\pm 0.12$ & $98.9\pm 0.10$ & $98.8\pm 0.05$ & $98.8\pm 0.12$\\
1000 & 7& 0.703 &$98.9\pm 0.13$ & $98.9\pm 0.14$ & $98.7\pm 0.11$ & $98.7\pm 0.11$ & $98.9\pm 0.13$ & $98.9\pm 0.09$\\
\hline
\end{tabular}
\end{table*}

\subsection{The Effects of Adversarial Defenses}\label{subsec:robust-acc}
\cite{wang2017analyzing} propose to enhance the robustness of $k$NN by i) augmenting the training set with $\vx_{\rm adv}$ generated from WB-DA, and resulting in the classifier ATNN; ii) augmenting the training data with adversarial examples crafted by all the above attacks and leads to the classifier ATNN-ALL; iii) pruning the training set such that the pruned training set is $a$-separated \cite{wang2017analyzing} and gives the RobustNN classifier. In the above $k$NN-based classifiers, if we replace $k$NN by GL, we get four more classifiers: GL, ATGL, ATGL-ALL, RobustGL. For GL-based classifiers 
we adapt the same setting as that used in \citep{wang2017analyzing} for $k$NN-based classifiers, except the number of nearest neighbors, $k$, in constructing $\Wb_N$. 
We use the values of $k$ in Table~\ref{tab:certify-on-Abalone} for the whole labeled training set, i.e., 13, 10, and $7$, respectively, for Abalone, Halfmoon, and MNIST.

Figure~\ref{fig:acc:wb-wb-kerbel} shows the results ($20$ runs) of the above eight classifiers for three datasets classification under the WB DA and KSA attacks with different perturbations $r$. We provide the results of these classifiers under BB attacks in Appendix~\ref{Appendix:RobustAcc:BlackBox:Attack}. We vary $r$ from $0$ to $0.04$, to $0.2$, and to $4$ for these three datasets classification, respectively. For these three datasets, GL-based classifiers are always more accurate than $k$NN-based classifiers with or without defense under different WB attacks. Furthermore, as $r$ increases, the improvement becomes more significant. The training data pruning-based adversarial defense remarkably improves $k$NN-based classifiers in all cases (Fig.~\ref{fig:acc:wb-wb-kerbel} (a), (b), (c), (d), and (f){) with the exception of} the Halfmoon {dataset} under the KSA attack, and improves GL-based classifiers when classifying the Abalone dataset under the 
KSA attack with large $r$
(Fig.~\ref{fig:acc:wb-wb-kerbel} (d)). Similarly, both data augmentation methods 
can usually improve classifiers' robustness, especially for MNIST 1v7.

\begin{figure}[!ht]
\centering
\begin{tabular}{ccc}
\hskip -0.65cm\includegraphics[clip, trim=0cm 0cm 0cm 0cm,width=0.31\columnwidth]{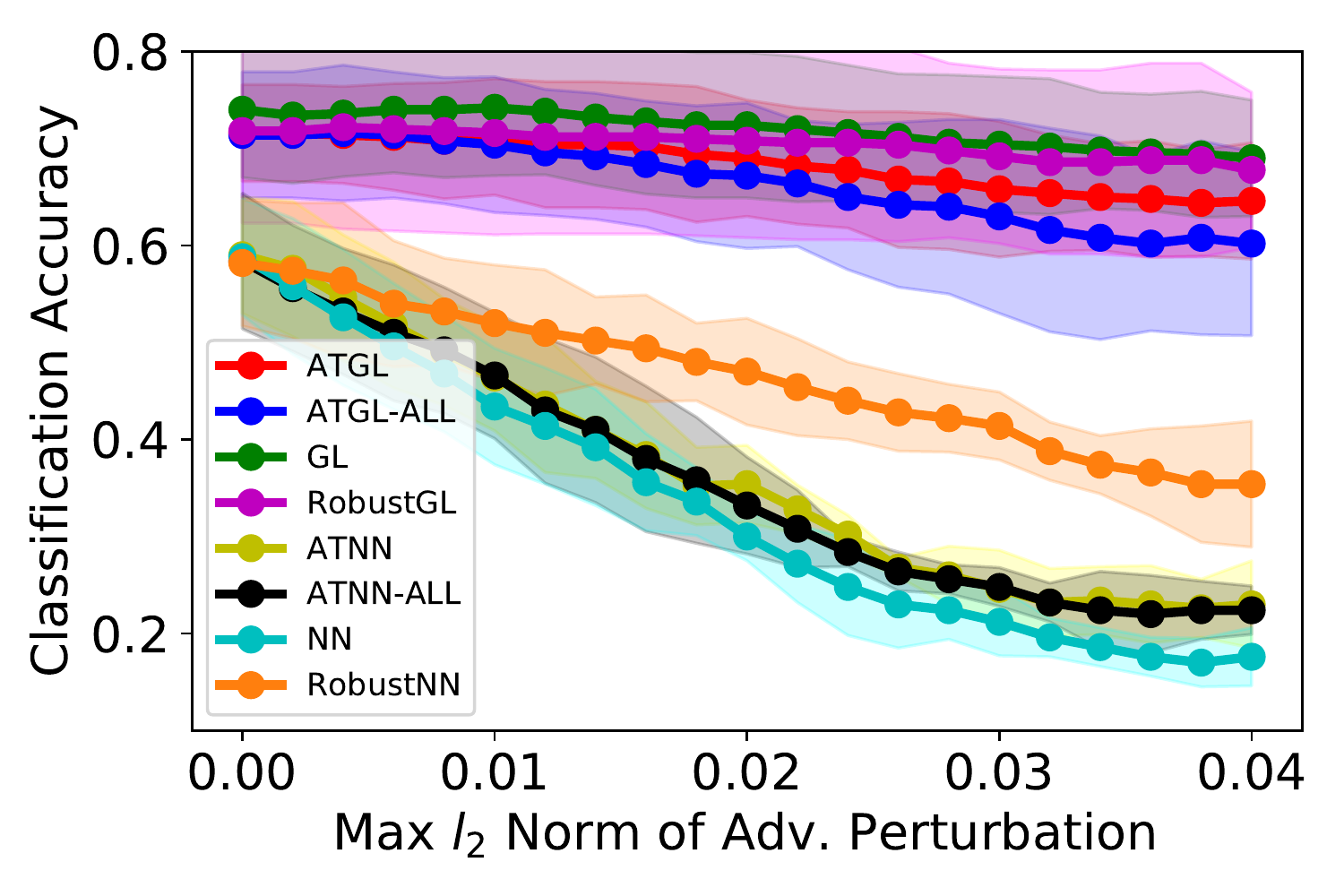}&
\hskip -0.4cm\includegraphics[clip, trim=0cm 0cm 0cm 0cm,width=0.31\columnwidth]{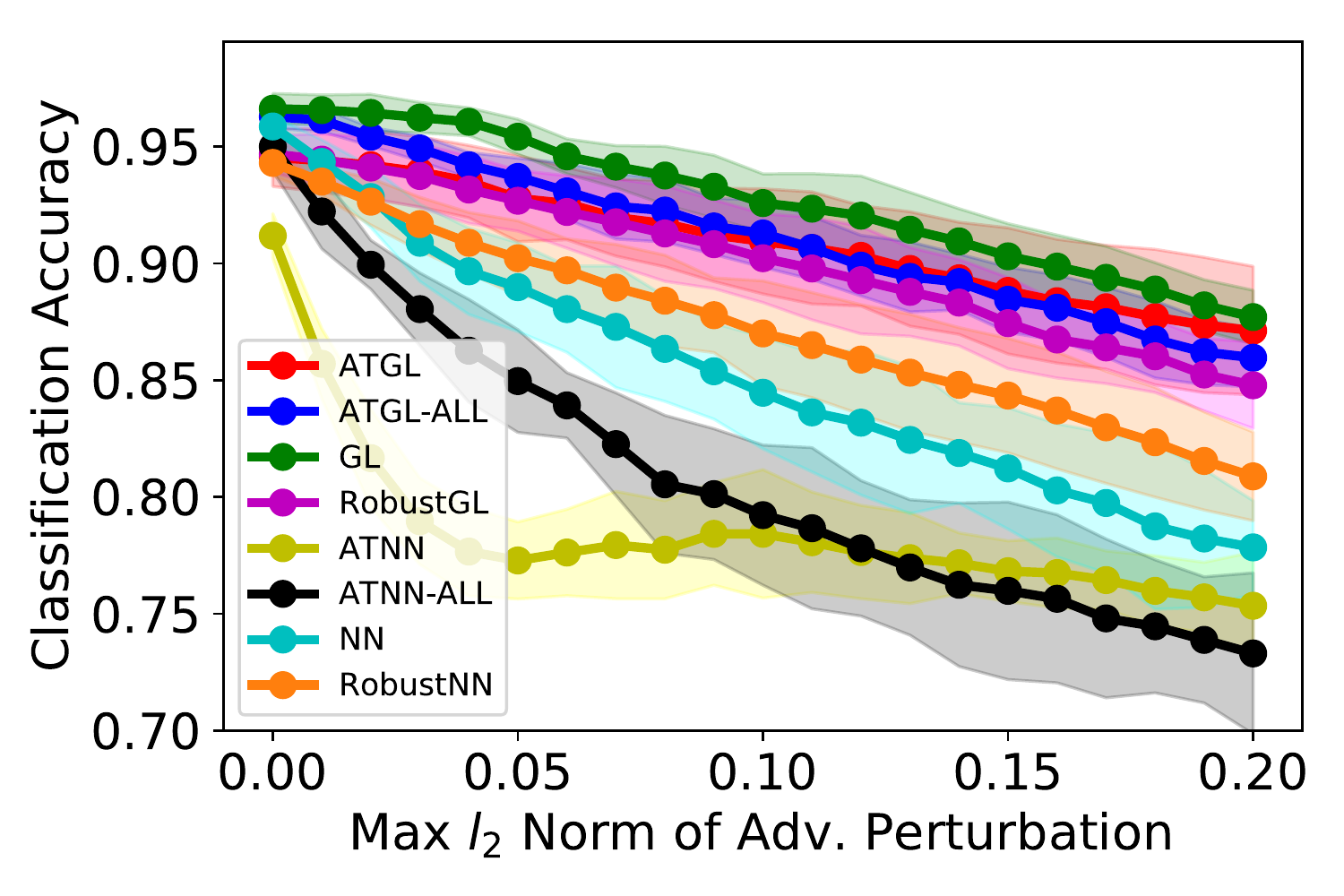}&
\hskip -0.4cm\includegraphics[clip, trim=0cm 0cm 0cm 0cm,width=0.31\columnwidth]{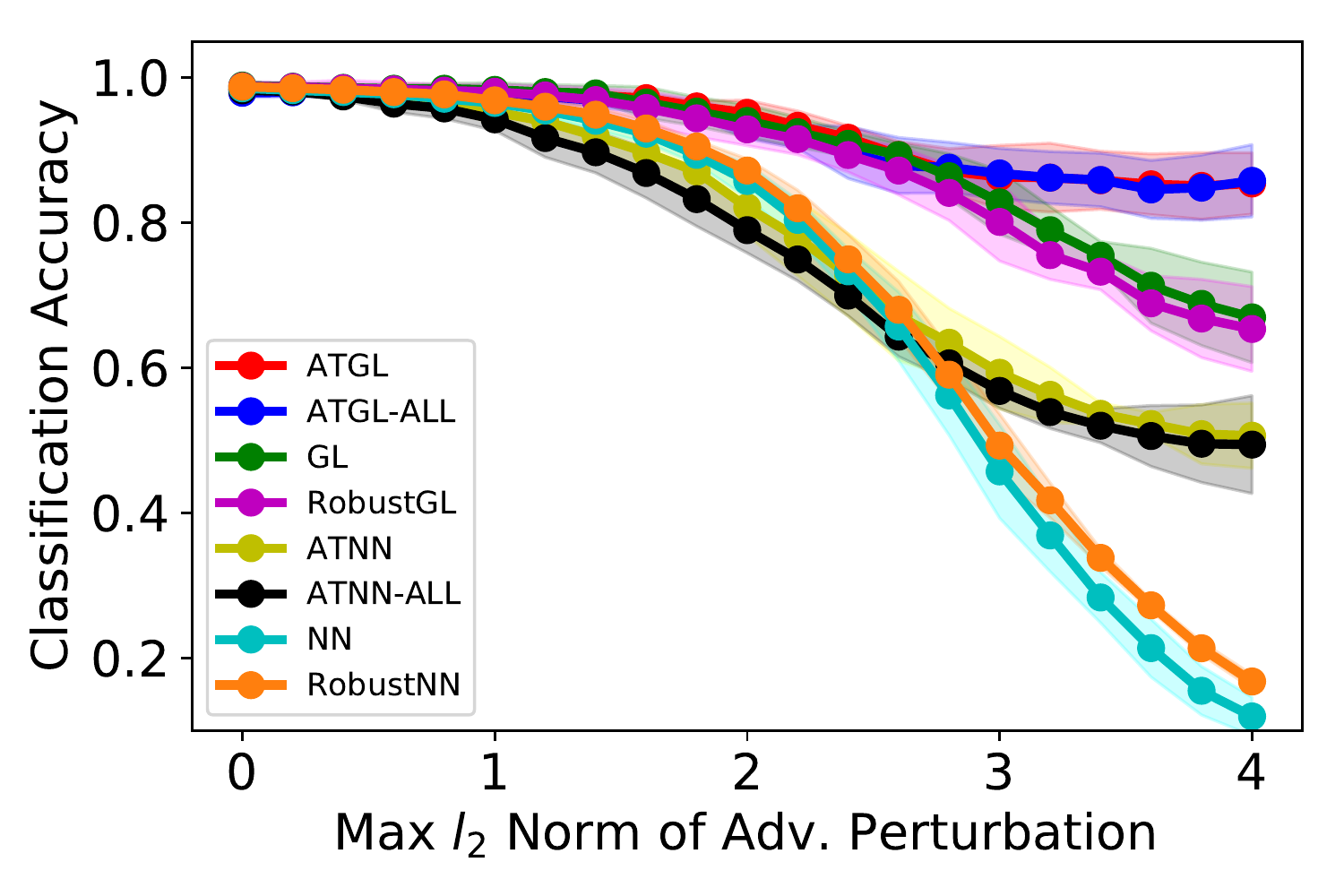}\\
{\footnotesize (a) Abalone, DA} & {\footnotesize (b) Halfmoon, DA} & {\footnotesize (c) MNIST 1v7, DA} \\
\hskip -0.65cm\includegraphics[clip, trim=0cm 0cm 0cm 0cm,width=0.31\columnwidth]{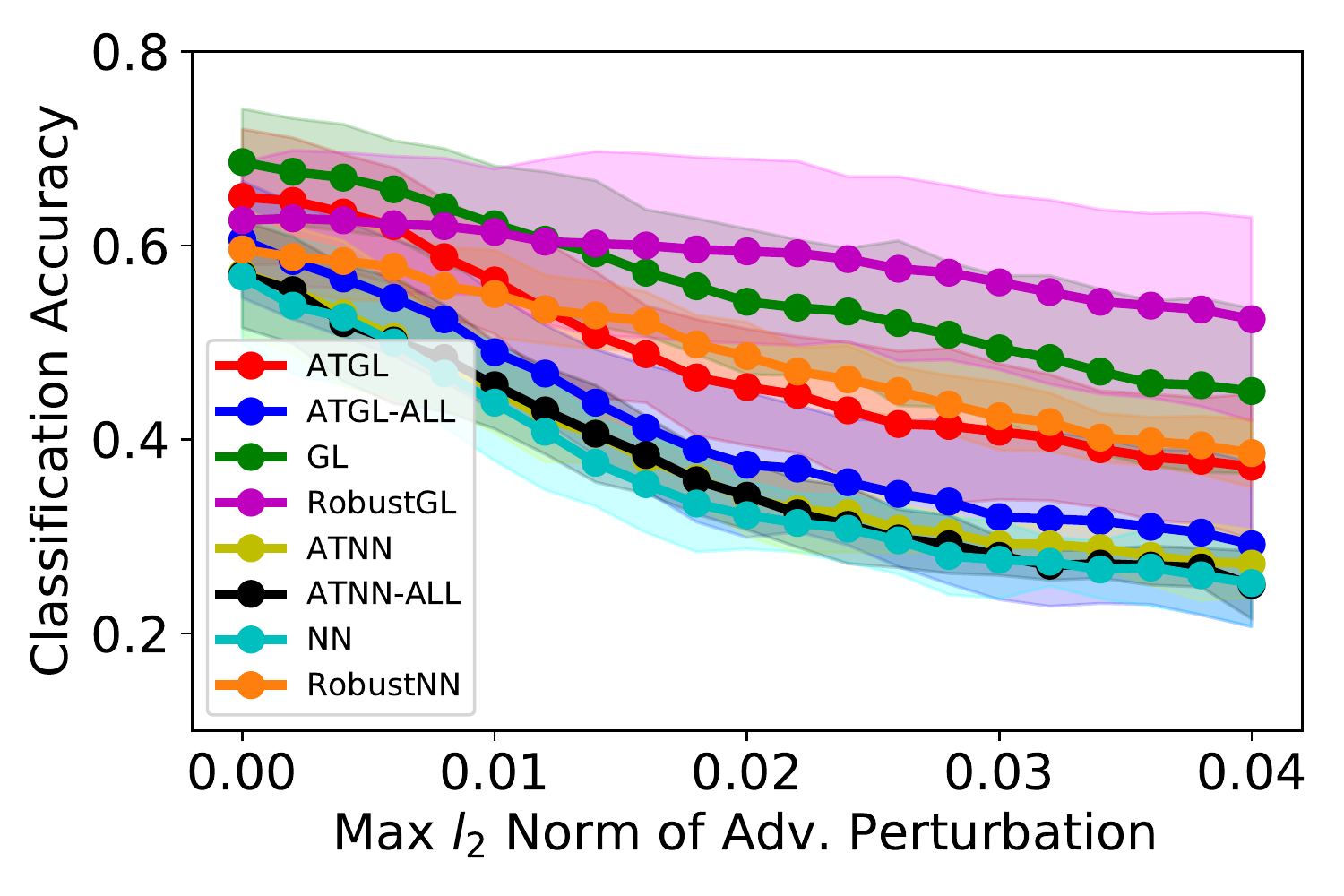}&
\hskip -0.4cm\includegraphics[clip, trim=0cm 0cm 0cm 0cm,width=0.31\columnwidth]{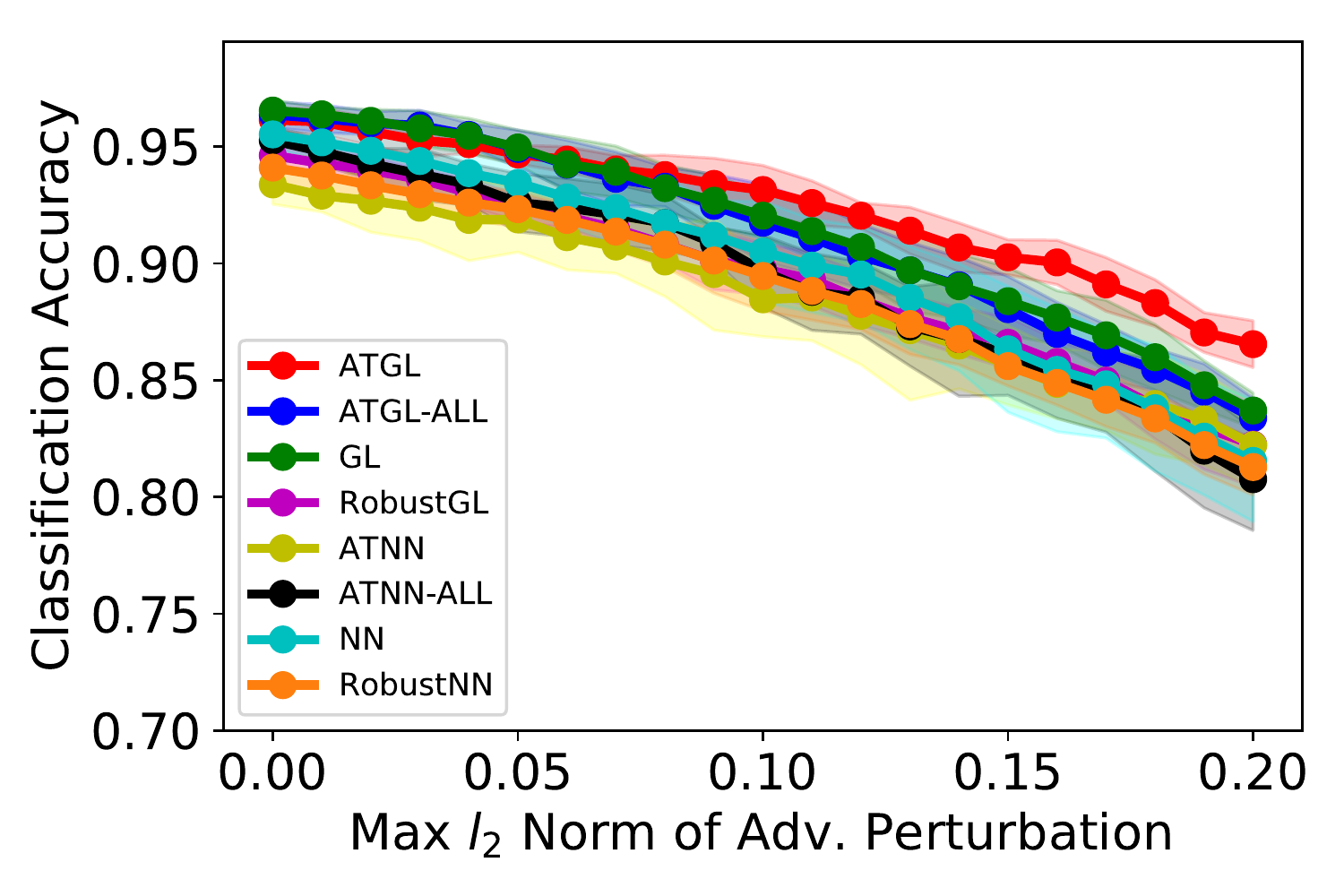}&
\hskip -0.4cm\includegraphics[clip, trim=0cm 0cm 0cm 0cm,width=0.31\columnwidth]{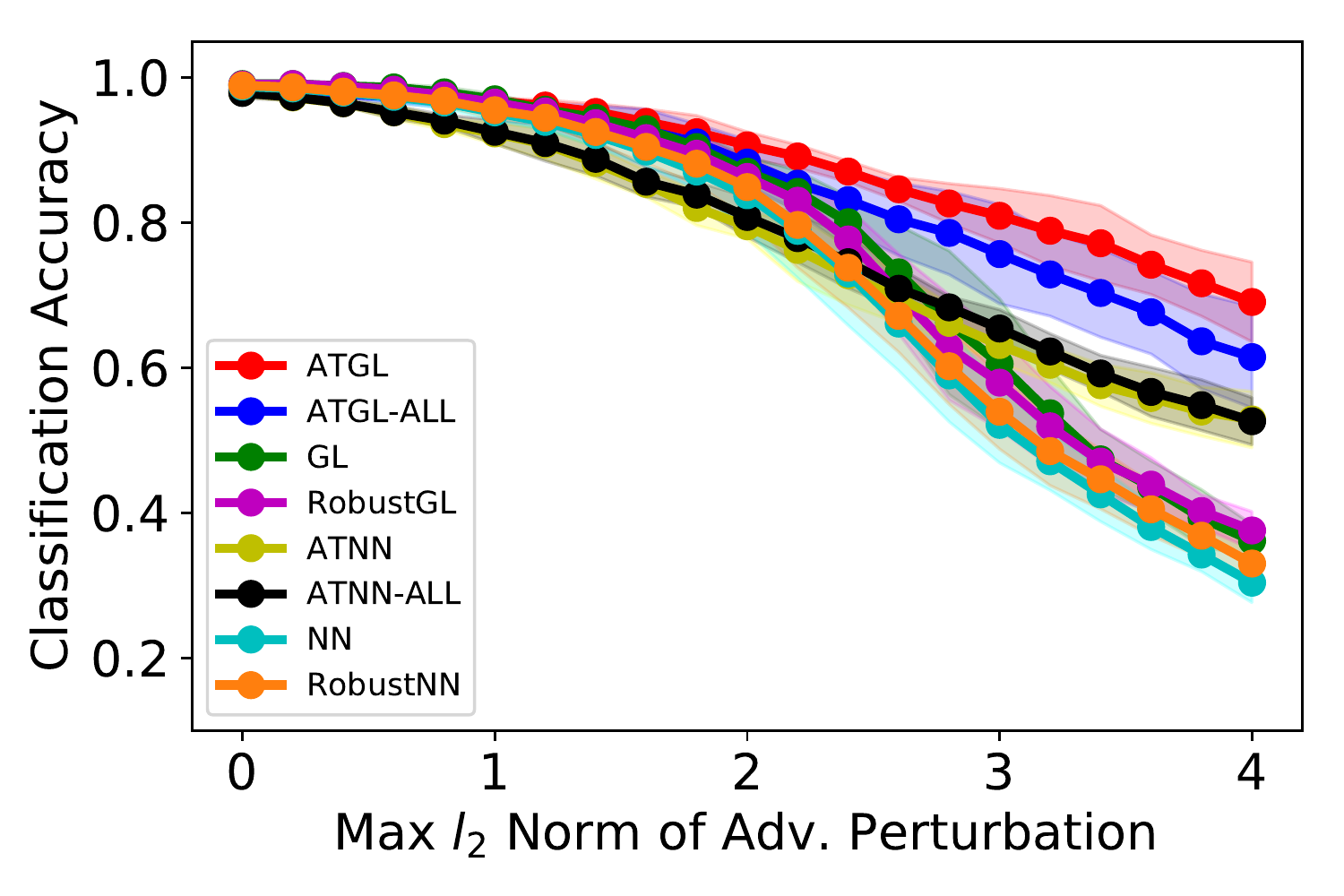}\\
{\footnotesize (d) Abalone, KSA} & {\footnotesize (e) Halfmoon, KSA} & {\footnotesize (f) MNIST 1v7, KSA} \\
\end{tabular}
\caption{Robust accuracies of GL vs. $k$NN classifiers for three datasets classification under WB attacks with different maximum perturbation measured in $\ell_2$-norm. GL-based classifiers are consistently more accurate than $k$NN-based classifiers. (Best viewed on a computer screen.)}
\label{fig:acc:wb-wb-kerbel}
\end{figure}

\begin{figure}[!ht]
\centering
\begin{tabular}{ccc}
\hskip -0.6cm\includegraphics[clip, trim=0cm 0cm 0cm 0cm, width=0.31\columnwidth]{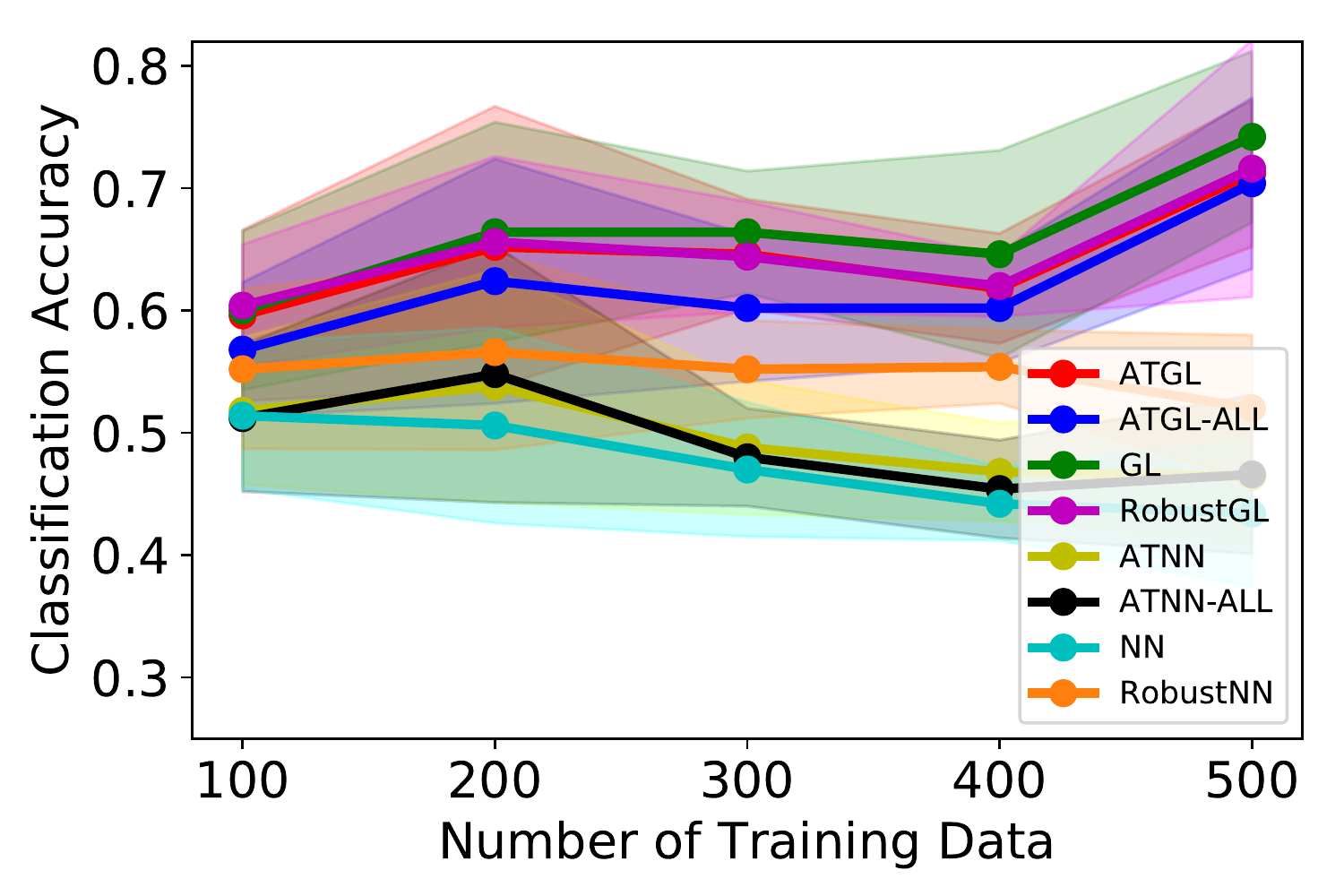}&
\hskip -0.4cm\includegraphics[clip, trim=0cm 0cm 0cm 0cm, width=0.31\columnwidth]{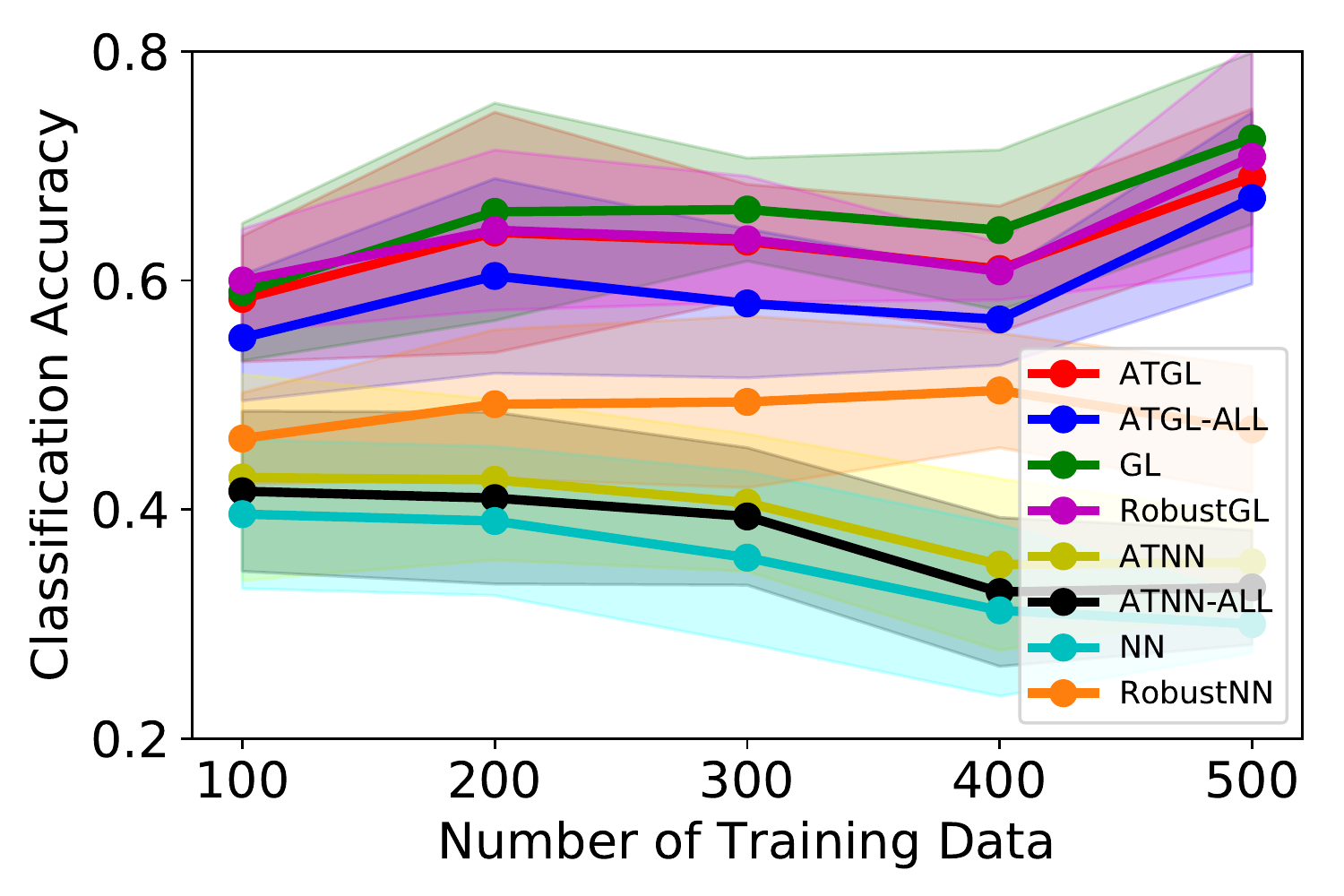}&
\hskip -0.4cm\includegraphics[clip, trim=0cm 0cm 0cm 0cm, width=0.31\columnwidth]{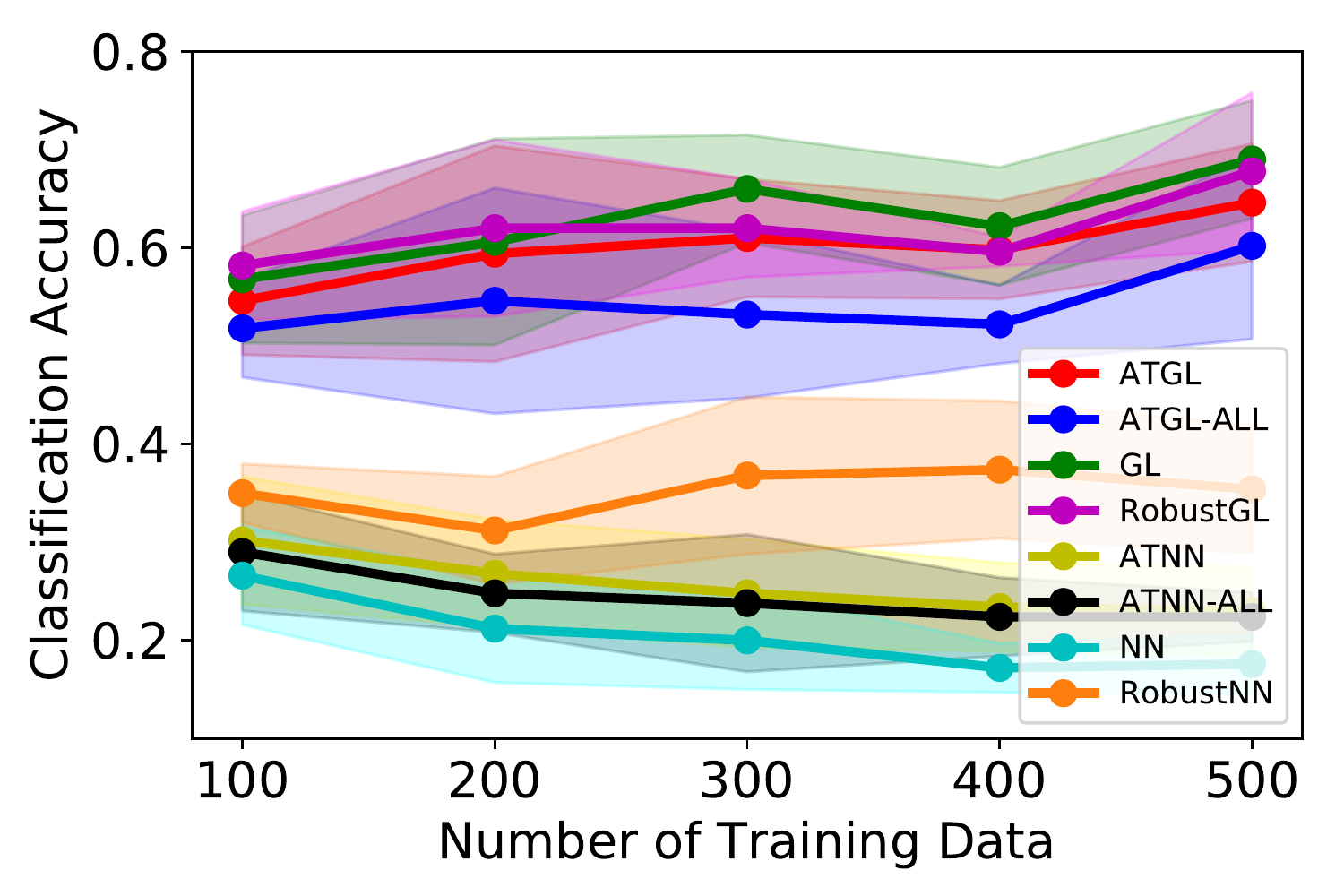}\\
DA, $r=0.01$ & DA, $r=0.02$ & DA, $r=0.04$ \\
\hskip -0.6cm\includegraphics[clip, trim=0cm 0cm 0cm 0cm, width=0.31\columnwidth]{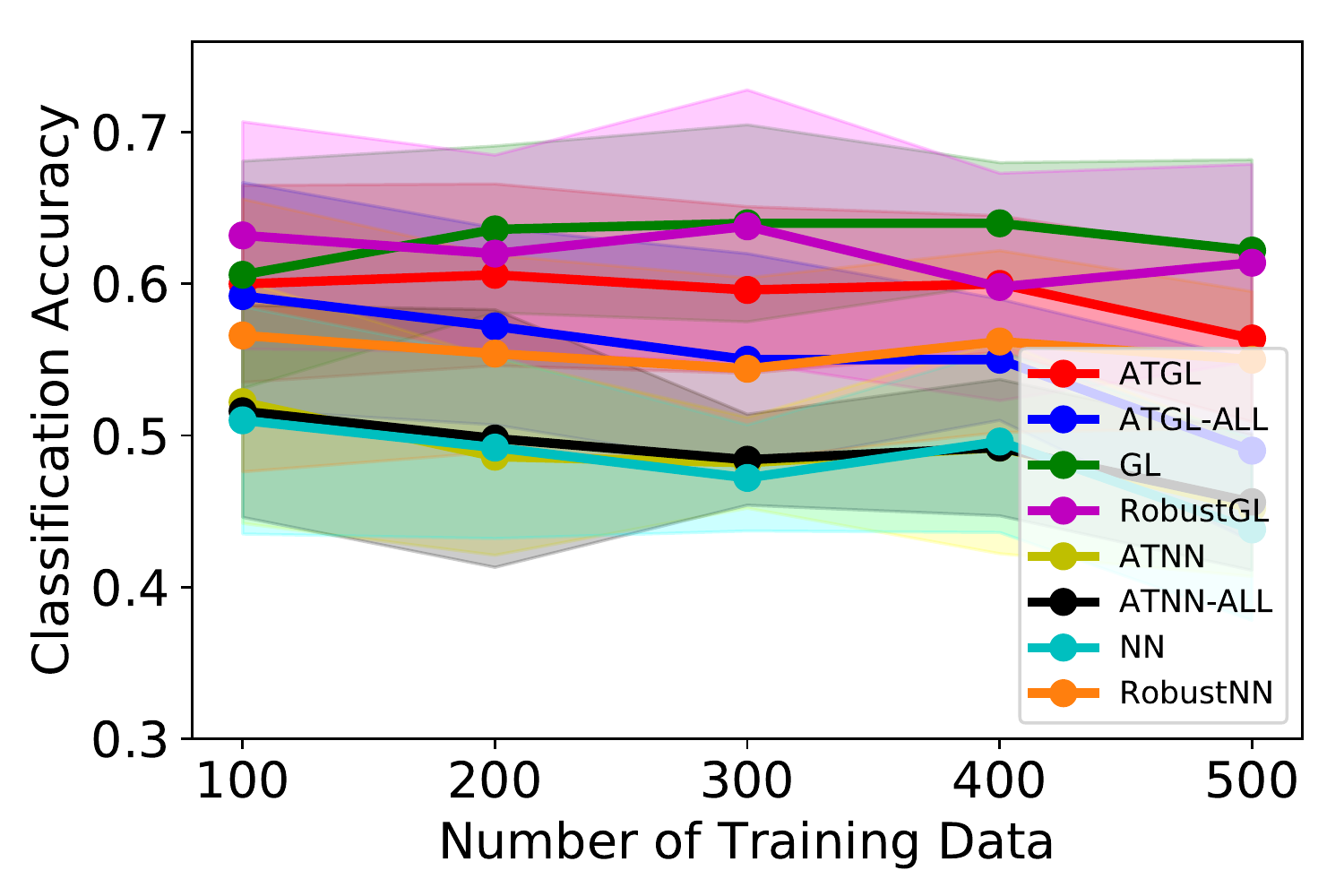}&
\hskip -0.4cm\includegraphics[clip, trim=0cm 0cm 0cm 0cm, width=0.31\columnwidth]{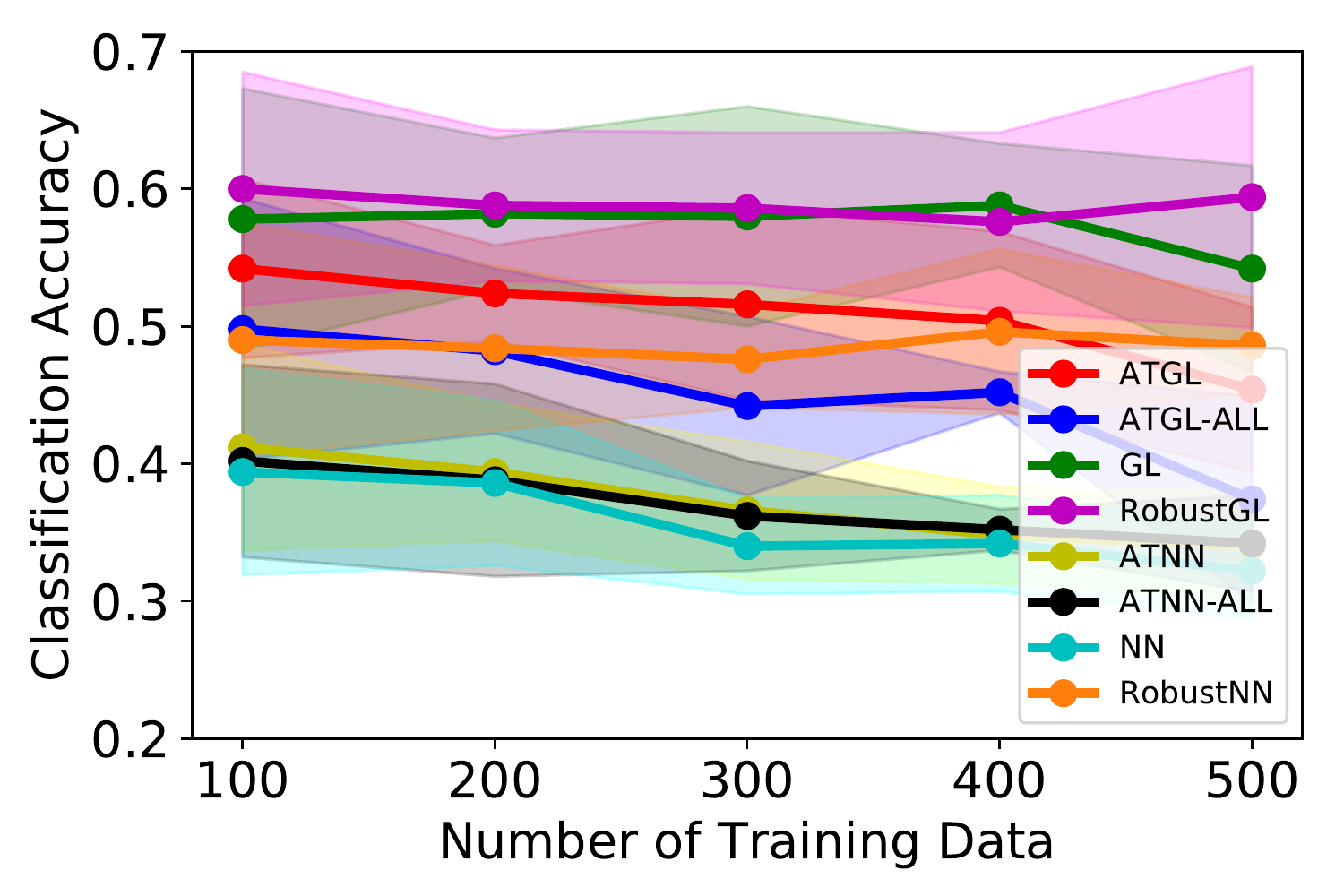}&
\hskip -0.4cm\includegraphics[clip, trim=0cm 0cm 0cm 0cm, width=0.31\columnwidth]{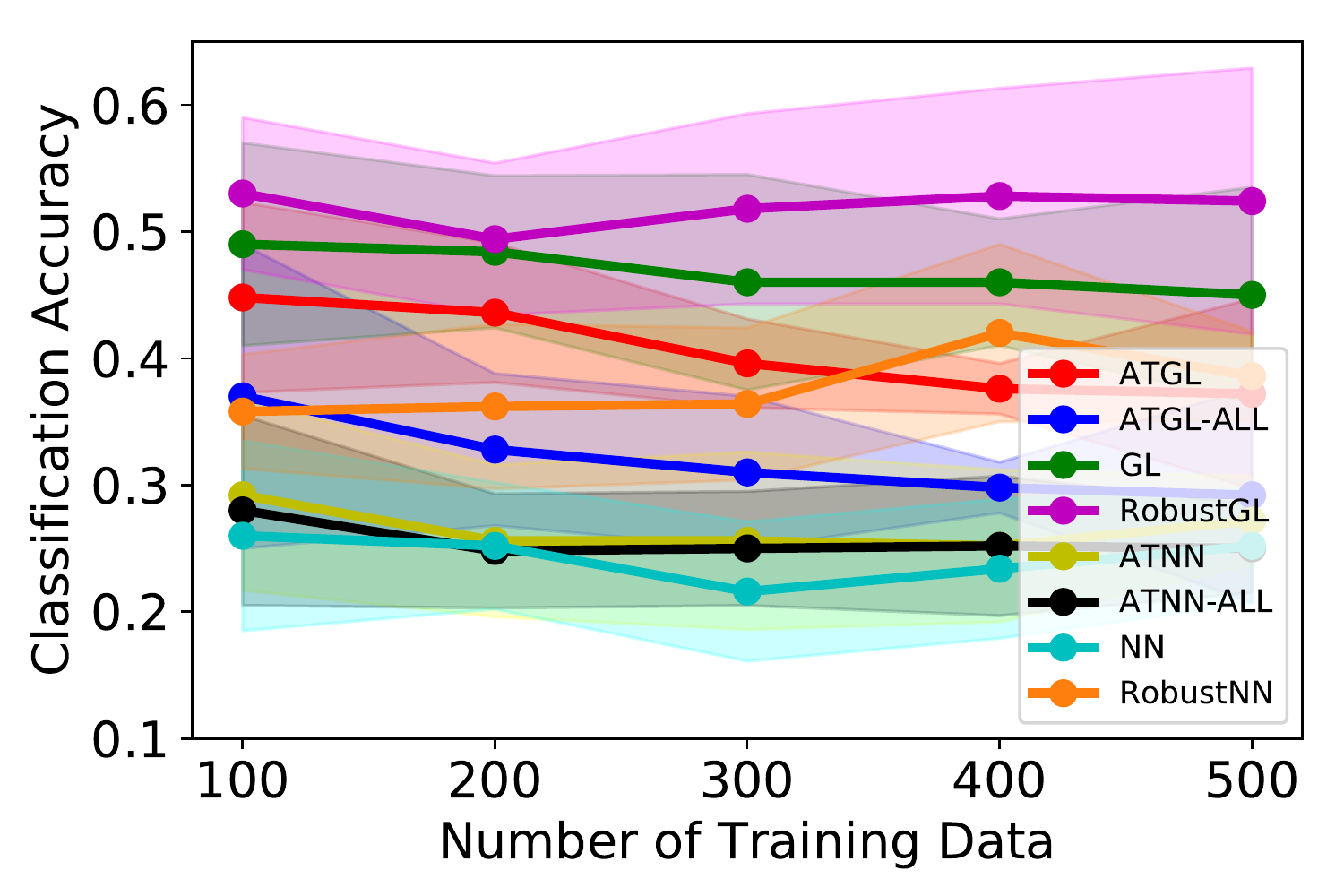}\\
KSA, $r=0.01$ & KSA, $r=0.02$ & KSA, $r=0.04$ \\
\end{tabular}
\caption{Robust accuracies of GL vs. $k$NN classifiers, trained by different number of training data, for classifying the Abalone under the WB attacks. As the number of training data increases, GL classifiers becomes more robust, while $k$NN classifiers are not. (Best viewed on a computer screen.)
}
\label{fig:acc:diff:number:training-data:WBAttack:Abalone}
\end{figure}

\subsection{Robust Accuracy with Different Number of Training Data}
Our theory indicates that for a given $k$, the robustness of GL classifier depends on the number of labeled data. In this subsection, we study the effects of the number of labeled data on the classifiers' robustness. 
We present the results of 
classifying the Abalone dataset under WB attacks with $r=0.01, 0.02$ and $0.04$, respectively, in Fig.~\ref{fig:acc:diff:number:training-data:WBAttack:Abalone}. The BB attacks are discussed in Appendix~\ref{appendix:RAcc:vs:Different-num-training-data}. In general, under these three choices of $r$, 
the GL-based classifiers tend to be more robust as the number of training data increases.

\section{Conclusions}\label{sec:Conc}
In this paper, we gave the first rigorous analysis of the adversarial robustness of the Graph Laplacian (GL)-based semi-supervised learning algorithm under evasion attacks.
Theoretically, we showed the sample limit that guarantees adversarial robustness of the GL classifier and its theoretical advantages over the $k$-nearest neighbor ($k$NN) classifier. 
We have also empirically shown that the robustness of the GL classifier can be remarkably improved by adversarial defenses and increasing the amount of training data. {Our theoretical results echo the observation that unlabeled data can improve the robustness of machine learning algorithms \citep{carmon2019unlabeled} and provides a potential explanation for the observation that GL-based activation function remarkably improves neural nets' robustness \citep{wang2019graph}.} Many interesting problems are remaining, for instance, 
how to develop a theory of robustness of GL classification under other types of the adversary, e.g., poisoning attacks?

{The results in this paper are built on large data limits and there is an increasing body of literature that could provide the tools to analyse such problems as spectral clustering, e.g.~\citep{SpecRatesTrillos,singer06}, and Cheeger cuts, e.g.~\citep{trillos2020graph}.
The main difficulty in applying these results is that one needs to be able to quantify how the ``closeness'' of the samples to being independent and identically distributed samples affects the labeling.
In this work, the bound on the maximum adversarial distance means that the perturbed data are still close to being independent and identically distributed and so the large data theory still applies.
Extending these results to other settings is of great interest and, in particular, we expect that different methods will see different robustness scalings as we believe the regularity of the solutions will play a role. 
}



\acks{This material is based on research sponsored by the NSF grant DMS-1924935 and DMS-1952339, and the DOE grant  DE-SC0021142.}

\clearpage


\clearpage
\appendix

\part{Appendices} 
The appendices are structured as follows. In Section~\ref{Appendix:proof}, we proof the results in Section~\ref{sec:theory}. In Section~\ref{Appendix:RobustAcc:BlackBox:Attack}, we numerically study the robustness of the GL-based classifier under the black-box attacks. In Section~\ref{appendix:RAcc:vs:Different-num-training-data}, we empirically study the effects of the number of training data in the robustness of the GL-based classifiers and contrast it to the $k$NN-based classifiers. In Section~\ref{sec:Vis:BB:Attack}, we visualize the adversarial examples of GL-based classifiers under different adversarial attacks.

\addcontentsline{toc}{section}{Appendix} 

\parttoc 

\section{Proofs of the Main Results}\label{Appendix:proof}

We define the following quantities:
\begin{align*}
d_{N,\eps}(\vx;\Omega_N) & = \sum_{\vy\in\Omega_N} \Wb_{\vx,\vy} = \sum_{\vy\in\Omega_N} \eta_\eps(|\vy-\vx|) \\
p_{N,\eps}(\vx;\Gamma_N) & = \sum_{\vy\in\Gamma_N} \Wb_{\vx,\vy} = \sum_{\vy\in\mtmath{\Gamma_N}} \eta_\eps(|\vy-\vx|) \\
\cL\varphi(\vx) & = \frac{\sigma_\eta}{\rho(\vx)}\Div(\rho^2\nabla\varphi)(\vx)
\end{align*}
where $\sigma_\eta = \int_{B(0,2)} \eta(|\vz|) |\vz_1|^2 \, \dd \vz$.
We also define the constant $C_\eta = \int_{B(0,2)} \eta(|\vz|) \, \dd \vz$. The value $d_{N,\eps}(\vx;\Omega_N)$ is the degree of the node at $\vx$ and $p_{N,\eps}(\vx;\Omega_N)$ is degree of the node at $\vx$ using only labeled data.
We also define the ``thickened'' boundary $\partial_\eps\Omega$ by
\[ \partial_\eps\Omega =\lb \vx\in\Omega \,:\, \dist(\vx,\partial\Omega)<\eps\rb. \]
We now recall Theorem 3.11 from~\cite{calder20AAA} which forms the basis for our proofs.

\begin{theorem}
\label{thm:app:Conv}
\cite[Theorem 3.11]{calder20AAA}
Let $\eps\in (0,1)$ and $\beta\in [\eps^2,1]$, $\ell:\Omega\to\bbR$ be Lipschitz, $\rho\in\Cr{2}(\Omega)$, $\Omega$ satisfy (A1), and $W_{\vx,\vy}$ is constructed as in (A4).
Assume there exists $C>c>0$ such that $\Omega_N^\prime = \{\vx_i^\prime\}_{i=1}^N\subset\Omega$ and $\Gamma_N^\prime\subseteq\Omega_N^\prime$ satisfy 
\begin{align}
\la d_{N,\eps}(\vx;\Omega^\prime_N) - C_\eta N\rho(\vx)\ra & \leq CN\sqrt{\beta} \label{eq:app:cond1} \\
d_{N,\eps}(\vx;\Omega_N^\prime) & \geq cN \label{eq:app:cond2} \\
p_{N,\eps}(\vx;\Gamma_N^\prime) & \geq cN\beta \label{eq:app:cond3} \\
\la \frac{1}{N\eps^2} \cL_N(\varphi;\Omega_N^\prime)(\vx) - \cL\varphi(\vx) \ra & \leq C\|\varphi\|_{\Cr{3}(\bar{\Omega})} \frac{\sqrt{\beta}}{\eps} && \forall \varphi\in\Cr{3}(\bar{\Omega}) \label{eq:app:cond4}
\end{align}
for all $\vx\in\Omega_N^\prime\setminus \partial_{2\eps}\Omega$.
There exists $C_0>0$ and $\bar{C}$ such that if $u^\prime$ satisfies
\begin{align*}
\cL_N(u^\prime;\Omega_N^\prime) & = 0 && \forall \vx\in\Omega_N^\prime\setminus\Gamma_N^\prime \\
u^\prime(\vx) & = \ell(\vx) && \forall \vx\in\Gamma_N^\prime
\end{align*}
then
\[ \max_{\vx\in\Omega_N^\prime\setminus\partial_\tau\Omega} \la u^\prime(\vx) - \ell(\vx)\ra \leq \frac{\bar{C}\eps}{\sqrt{\beta}} \log\l\frac{\sqrt{\beta}}{\eps}\r \]
where $\tau = \frac{C_0\eps}{\sqrt{\beta}} \log\l\frac{\sqrt{\beta}}{\eps}\r$.
\end{theorem}

The theorem is stated for any collection of data points $\Omega_N^\prime=\{\vx_i^\prime\}_{i=1}^N$ 
{and, under the assumptions, gives a quantitative bound between the GL solution and the label $\ell$.
The general idea behind the proof of Theorem~\ref{thm:app:Conv} is to exploit the connection between solutions of the GL-based classifier with an appropriately defined random walk that terminates when it hits a labeled data point.
In particular, the distance of the random walk after $k$ steps from its starting position is like $\eps\sqrt{k}$ (the size of each step is $\sim\eps$ so this coincides with the usual random walk bounds in continuum domains) and the random walk will terminate after approximately $k\sim N/|\Gamma_N^\prime|$ steps.
By~\eqref{eq:app:cond3} the probability of a data point being labelled is a fraction of $\beta$.
Putting all this together one gets that the error of the GL-based solution should be on the order of $\epsilon/\sqrt{\beta}$.
A more careful treatment gives the extra logarithmic terms and makes precise the high probability bound, we refer to~\cite{calder20AAA} for details.}

We will apply the theorem to the dataset $\Omega_N$ and the adversarially-perturbed domain $\hat{\Omega}_N$. 
In~\cite{calder20AAA} it is shown that the conditions (\ref{eq:app:cond1}-\ref{eq:app:cond4}) hold with high probability when the data points $\vx_i$ are iid, we recall this result in Lemma~\ref{lem:app:iidAss}.

To make notation easier we let
\begin{align*}
u & = \argmin \, \cEcon(\cdot;D_N) & \hat{u} & = \argmin \, \cEcon(\cdot;\hat{D}_N) \\
\cL_N & = \cL_N(\cdot;\Omega_N) & \hat{\cL}_N & = \cL_N(\cdot;\hat{\Omega}_N) \\
d_{N,\eps} & = d_{N,\eps}(\cdot;\Omega_N) & \hat{d}_{N,\eps} & = d_{N,\eps}(\cdot;\hat{\Omega}_N) \\
p_{N,\eps} & = p_{N,\eps}(\cdot;\Gamma_N) & \hat{p}_{N,\eps} & = p_{N,\eps}(\cdot;\hat{\Gamma}_N).
\end{align*}

We are unable to apply Theorem~\ref{thm:app:Conv} directly to the adversarially-perturbed problem. This is because the function $\hat{\ell}_N$ is not Lipschitz continuous (and depends on $N$).
To control the adversarially-perturbed problem we use stability of Laplace's equation.
In particular, we let $\hat{w}$ satisfy
\begin{align*}
\hat{\cL}_N \hat{w}(\hat{\vx}) & = 0 && \forall\hat{\vx}\in\hat{\Omega}_N\setminus\hat{\Gamma}_N \\
\hat{w}(\hat{\vx}) & = \ell(\hat{\vx}) && \forall\hat{\vx}\in\hat{\Gamma}_N.
\end{align*}
We first show that $\hat{u}$ and $\hat{w}$ are close.

\begin{lemma}
\label{lem:app:LapStab}
Assume the graph $\hat{G}_N$, which consists of nodes $\hat{\Omega}_N$ and edges between any two nodes $\hat{\vx}$, $\hat{\vy}$ for which $\Wb_{\eps,\hat{\vx},\hat{\vy}}>0$, is connected.
Let $\hat{\ell}_N(\hat{\vx})=\ell(\vx)$ where $|\hat{\vx}-\vx|\leq r$ for all $\vx\in\Omega_N$ and suppose that $\ell$ is Lipschitz continuous.
Then,
\[ \max_{\hat{\vx}\in\hat{\Omega}_N} \la \hat{u}(\hat{\vx}) - \hat{w}(\hat{\vx}) \ra\leq \Lip(\ell) r. \]
\end{lemma}

\begin{proof}
Let $\hat{v} = \hat{u} - \hat{w}$.
Then $\hat{v}$ satisfies
\begin{align*}
\hat{\cL}_N \hat{v}(\hat{\vx}) & = 0 && \forall \hat{\vx}\in\hat{\Omega}_N\setminus\hat{\Gamma}_N \\
\hat{v}(\hat{\vx}) & = \ell(\vx) - \ell(\hat{\vx}) && \forall \hat{\vx}\in\hat{\Gamma}_N.
\end{align*}
By the maximum principle (for example see~\cite[Theorem 3]{calder18})
\[ \max_{\hat{\vx}\in\hat{\Omega}_N} \hat{v}(\hat{\vx}) = \max_{\hat{\vx}\in\hat{\Gamma}_N} \hat{v}(\hat{\vx}) \leq \Lip(\ell) r. \]
Similarly, by the minimum principle we have $\min_{\hat{\vx}\in\hat{\Omega}_N} \hat{v}(\hat{\vx})\geq -\Lip(\ell)r$.
Combining the two bounds we can conclude the result.
\end{proof}

Next we recall that the conditions~\ref{thm:app:Conv} hold in the unperturbed domain.

\begin{lemma}
\label{lem:app:iidAss}
Let Assumptions (A1-A4) hold.
There exists $C>c>0$ such that if $\eps\in (0,1)$ and $\beta\in [\eps^2,1]$ then
\begin{align*}
\la d_{N,\eps}(\vx) - C_\eta N\rho(\vx)\ra & \leq CN\sqrt{\beta} \\
d_{N,\eps}(\vx) & \geq cN \\
p_{N,\eps}(\vx) & \geq cN\beta \\
\la \frac{1}{N\eps^2} \cL_{N}\varphi(\vx) - \cL\varphi(\vx) \ra & \leq C\|\varphi\|_{\Cr{3}(\bar{\Omega})} \frac{\sqrt{\beta}}{\eps} && \forall \varphi\in\Cr{3}(\bar{\Omega})
\end{align*}
for all $\vx\in\Omega_N\setminus\partial_{2\eps}\Omega$ with probability at least $1-CNe^{-cN\beta\eps^d}$.
\end{lemma}

\begin{proof}
The first inequality holds by choosing $\delta=\sqrt{\beta}$ in~\cite[Theorem~5]{calder18} (and noting that in the proof one establishes the bound with probability at least $1-CNe^{-cN\beta\eps^d}$).
The second inequality holds by~\cite[Propositions 3.5 and 3.8]{calder20AAA}.
The third inequality holds by Remark~\ref{rem:app:LabelDegreeSpecCase} below. 
The fourth inequality holds by choosing $\delta=\frac{\sqrt{\beta}}{\eps}$ in~\cite[Theorem~5]{calder18}.
\end{proof}

To prove the bounds for the perturbed model we will use the following preliminary result.

\begin{lemma}
\label{lem:app:BallBound}
Let $A_N\subset\Omega_N$ satisfy $\bbP(\vx\in A_N) = \alpha\in [0, 1]$.
Then, there exists, $a>0$ and $C>c>0$ (independent of $\alpha$) such that for all $\tau\in(0,1]$ and $0<\vartheta\leq 1$
\begin{align*}
& \bbP\l (1-\vartheta-a\tau)C_\tau(\vx)N\alpha\leq \#\lb \vy\in A_N\,:\,|\vx-\vy|\leq \tau\rb \leq (1+\vartheta+a\tau)C_\tau(x)N\alpha, \, \forall \vx\in \Omega_N\r \\
& \hspace{4cm} \geq 1-CNe^{-cN\alpha\tau^d\vartheta^2}
\end{align*}
where $C_\tau(\vx) = \rho(\vx)\Vol(B(\vx,\tau)\cap\Omega)$.
Moreover, there exists $\tau_0>0$, $C_2>C_1>0$, such that, for all $\tau\in (0,\tau_0)$,
\[ \bbP\l C_1N\alpha\tau^d \leq \# \lb \vy\in A_N\,:\, |\vx-\vy|\leq \tau \rb \leq C_2 N\alpha\tau^d, \, \forall \vx\in \Omega_N\r \geq 1-CNe^{-cN\alpha\tau^d}. \]
\end{lemma}

\begin{remark}
\label{rem:app:DegreeSpecCase}
We can apply the above lemma to lower bound the number of labeled data points in $B(\vx,\eps)$. In particular, if $\vx\in\Omega\setminus\partial_{2\eps}\Omega$ then $C_\eps(\vx) = \rho(\vx)\Vol(B(0,1)) \eps^d$, and if $\eta=\one_{\cdot \leq 1}$ then $C_\eta = \Vol(B(0,1))$. We choose $\alpha = 1$ and $\tau=\eps$ to infer
\[ d_{N,\eps}(\vx) = \sum_{\vy\in\Omega_N} \Wb_{\vx,\vy} = \frac{1}{\eps^d} \#\lb \vy\in\Omega_N \,:\, |\vx-\vy|\leq \eps\rb \leq C_\eta N\rho(\vx) + O\l(\vartheta+\eps)N\r \]
with probability at least $1-CNe^{-cN\eps^d\vartheta^2}$.
Choosing $\vartheta=\sqrt{\beta}$ proves the first inequality in Lemma~\ref{lem:app:iidAss} for the special case $\eta=\one_{\cdot \leq 1}$.
\end{remark}

\begin{remark}
\label{rem:app:LabelDegreeSpecCase}
We can also use the above lemma to bound the number of labeled data points in $B(\vx,\eps)$. Let $\vx\in\Omega\setminus\partial_{2\eps}\Omega$, and choose $\alpha = \beta$.
Applying the second bound in Lemma~\ref{lem:app:BallBound}, and using that $\eta(t) \geq 1$ for all $t<1$, we have
\[ p_{N,\eps}(\vx) = \sum_{\vy\in\Omega_N} \Wb_{\vx,\vy} \geq \frac{1}{\eps^d} \# \lb \vy\in \Gamma_N\,:\, |\vx-\vy|\leq \eps \rb \geq C_1 N \beta \]
with probability at least $1-CNe^{-cN\beta\eps^d}$.
This proves the third inequality in Lemma~\ref{lem:app:iidAss}.
\end{remark}

\begin{proof}[Proof of Lemma~\ref{lem:app:BallBound}]
Fix $\vx\in\Omega_N$ and $\tau>0$.
Let $\xi_\vy = 1$ if $\vy\in A_N$ and $|\vx-\vy|\leq \tau$, and $\xi_\vy=0$ otherwise.
We can write
\[ \#\lb \vy\in A_N\,:\, \la \vx-\vy\ra\leq \tau\rb = \sum_{\vy\in\Omega_N} \xi_\vy. \]
By Bernstein's inequality
\[ \bbP\l\sum_{\vy\in\Omega_N}\l \xi_\vy - \bbE[\xi_\vy]\r \geq t\r \leq \exp\l-\frac{ct^2}{N\sigma^2+t}\r \]
for all $t>0$ and where
\[ \sigma^2 = \bbE\l\xi_\vy - \bbE[\xi_\vy]\r^2 \]
(note that the right hand side is independent of $\vy$).
Using the lower bound on the density $\rho$ of $\vx$ we infer
\begin{align*}
\bbP(\xi_\vy=1) & = \bbP\l B\l \vx,\tau\r\cap\Omega\r\bbP(\vy\in A_N) \\
 & = \alpha \Vol\l B(\vx,\tau)\cap\Omega\r \l \rho(\vx) + O(\tau)\r \\
 & = \alpha C_\tau(\vx)\l 1 +O(\tau)\r.
\end{align*}
So there exists $a>0$ such that
\[ \alpha C_\tau(\vx)\l 1-a\tau\r \leq \bbE[\xi_\vy] = \bbP(\xi_\vy=1) \leq \alpha C_\tau(\vx)\l 1+a\tau\r. \]
Hence,
\[ \sigma^2 \leq \bbE[\xi_\vy^2] = \bbE[\xi_\vy]  \leq C\alpha\tau^d \]
for some $C>0$.
We can then bound
\begin{align*}
& \bbP\l \#\lb \vy\in A_N\,:\,|\vx-\vy|\leq \tau\rb\leq t+C_\tau(\vx)N\alpha(1+a\tau)\r \\
& \hspace{2cm} \geq \bbP\l\sum_{\vy\in\Omega_N}\l\xi_\vy-\bbE[\xi_\vy]\r\leq t\r \\
& \hspace{2cm} \geq 1-\exp\l-\frac{ct^2}{N\sigma^2+t}\r \\
& \hspace{2cm} \geq 1-\exp\l-\frac{ct^2}{CN\alpha\tau^d+t}\r.
\end{align*}
Similarly,
\begin{align*}
& \bbP\l \#\lb \vy\in A_N\,:\,|\vx-\vy|\leq \tau\rb\geq -t+C_\tau(x)N\alpha(1-a\tau)\r \\
& \hspace{2cm} \geq \bbP\l \sum_{\vy\in\Omega_N} \l\xi_\vy -\bbE[\xi_\vy]\r \geq -t\r \\
& \hspace{2cm} \geq 1-\exp\l-\frac{ct^2}{N\sigma^2+t}\r \\
& \hspace{2cm} \geq 1-\exp\l-\frac{ct^2}{CN\tau^d\alpha+t}\r.
\end{align*}
Choosing $t = C_\tau(x)N\alpha\vartheta$ implies
\begin{align*}
& \bbP\l (1-a\tau-\vartheta)C_\tau(x)N\alpha\leq \#\lb y\in A_N\,:\,|\vx-\vy|\leq \tau\rb \leq (1+a\tau+\vartheta)C_\tau(\vx)N\alpha\r \\
& \hspace{4cm} \geq 1-2e^{-cN\alpha\tau^d\vartheta^2}
\end{align*}
where we restrict $\vartheta\leq 1$.
Union bounding (Fr\'echet inequality for logical conjugation) implies the first result.

Choose $\tau_0>0$ sufficiently small so that $1-a\tau_0>0$ and set $\vartheta = \frac12(1-a\tau_0)>0$.
Noticing that we can find $C_1$, $C_2$ such that $0<C_1\leq (1-a\tau-\vartheta)C(\vx)\leq (1+a\tau+\vartheta)C(\vx) \leq C_2$ we can conclude the second part of the lemma.
\end{proof}

We are left to show the analogue of Lemma~\ref{lem:app:iidAss} for the perturbed quantities $\hat{d}_{N,\eps}$, $\hat{p}_{N,\eps}$ and $\hat{\cL}_{N}$.

\begin{lemma}
\label{lem:app:AdvAss}
Under Assumptions (A1-A4) and assuming $|\vx_i-\hat{\vx}_i|\leq r$ for all $i=1,\dots, N$ there exists $\eps_0>0$, $C>c>0$ such that if $\eps\in (0,\eps_0)$, $\beta\in [\eps^2,1]$ and $r\in (0,c\sqrt{\beta}\eps]$ then
\begin{align*}
\la \hat{d}_{N,\eps}(\hat{\vx}) - C_\eta N\rho(\hat{\vx})\ra & \leq CN\sqrt{\beta} \\
\hat{d}_{N,\eps}(\hat{\vx}) & \geq cN \\
\hat{p}_{N,\eps}(\hat{\vx}) & \geq cN\beta \\
\la \frac{1}{N\eps^2} \hat{\cL}_{N}\varphi(\hat{\vx}) - \cL\varphi(\hat{\vx}) \ra & \leq C\|\varphi\|_{\Cr{3}(\bar{\Omega})} \frac{\sqrt{\beta}}{\eps} && \forall \varphi\in\Cr{3}(\bar{\Omega})
\end{align*}
for all $\hat{\vx}\in\hat{\Omega}_N\setminus\partial_{2\eps}\Omega$ with probability at least $1-CNe^{-cN\beta\eps^d}$.
\end{lemma}

\begin{proof}
%
Let $\dist(\hat{\vx},\partial\Omega)> 2\eps$.
We first consider a bound on the difference between weights $\la \Wb_{\vx,\vy} - \Wb_{\hat{\vx},\hat{\vy}}\ra$.
If $\eta$ is Lipschitz continuous we have,
\begin{align*}
\la \Wb_{\vx,\vy} - \Wb_{\hat{\vx},\hat{\vy}}\ra & = \frac{1}{\eps^d} \la \eta\l\frac{\vx-\vy}{\eps}\r - \eta\l\frac{\hat{\vx}-\hat{\vy}}{\eps}\r\ra \\
 & \leq \frac{\Lip(\eta)}{\eps^{d+1}} \la |\vx-\vy| - |\hat{\vx}-\hat{\vy}|\ra \one_{|\vx-\vy|\leq2(\eps+r)} \\
 & \leq \frac{2\Lip(\eta)r}{\eps^{d+1}} \one_{|\vx-\vy|\leq2(\eps+r)}.
\end{align*}
On the other hand, If $\eta(t) = \one_{t<1}$ then we have
\[ \la \Wb_{\vx,\vy} - \Wb_{\hat{\vx},\hat{\vy}}\ra \leq \frac{1}{\eps^d} \one_{\eps-2r\leq |\vx-\vy|\leq \eps+2r}. \]

Now, for the first inequality we have, when $\eta$ is Lipschitz continuous,
\begin{align*}
\la \hat{d}_{N,\eps}(\hat{\vx}) - d_{N,\eps}(\vx) \ra & \leq \sum_{\vy\in\Omega_N} \la \Wb_{\hat{\vx},\hat{\vy}} - \Wb_{\vx,\vy}\ra \\
 & \leq \frac{Cr}{\eps^{d+1}} \#\lb \vy\in\Omega_N\,:\, |\vx-\vy|\leq 2(\eps+r)\rb \\
 & \leq \frac{CrN(\eps+r)^d}{\eps^{d+1}} \\
 & \leq CN\sqrt{\beta}
\end{align*}
by Lemma~\ref{lem:app:BallBound} with probability at least $1-CNe^{-cN\eps^d}$.
And when $\eta(t) = \one_{t<1}$,
\begin{align*}
\la \hat{d}_{N,\eps}(\hat{\vx}) - d_{N,\eps}(\vx) \ra & \leq \frac{1}{\eps^d} \#\lb \vy\in\Omega_N\,:\, \eps-2r\leq |\vx-\vy|\leq\eps+2r\rb \\
 & \leq \frac{CN}{\eps^d} \ls \l 1+\vartheta+ a(\eps+2r)\r\l \eps+2r \r^d - \l 1-\vartheta-a(\eps-2r)\r \l \eps-2r \r^d \rs \\
 & \leq CN(\sqrt{\beta}+\vartheta)
\end{align*}
by Lemma~\ref{lem:app:BallBound} with probability at least $1-CNe^{-cN\eps^d\vartheta^2}$, choosing $\vartheta=\sqrt{\beta}$ we have, in both cases,
\[ \la \hat{d}_{N,\eps}(\hat{\vx}) - d_{N,\eps}(\vx) \ra \leq \sum_{\vy\in\Omega_N} \la \Wb_{\vx,\vy} - \Wb_{\hat{\vx},\hat{\vy}}\ra \leq CN\sqrt{\beta} \]
with probability at least $1-CNe^{-cN\beta\eps^d}$.
By the triangle inequality and Lemma~\ref{lem:app:iidAss}
\[ \la \hat{d}_{N,\eps}(\vx) - C_\eta N\rho(\vx)\ra \leq \la \hat{d}_{N,\eps}(\vx) - d_{N,\eps}(\vx)\ra + \la d_{N,\eps}(\vx) - C_\eta N\rho(\vx)\ra \leq CN\sqrt{\beta} \]
with probability at least $1-CNe^{-cN\beta\eps^d}$.

By assuming that $c$ is sufficiently small (where $r\leq c\beta\eps$) we can make $C$ in the above bound arbitrarily small.
Hence we can assume that $C< C_\eta\rho_{\min}$ and therefore the second inequality holds.

Similarly, for the third inequality when $\eta$ is Lipschitz,
\begin{align*}
\la \hat{p}_{N,\eps}(\hat{\vx}) - p_{N,\eps}(\vx) \ra & \leq \sum_{\vy\in\Gamma_N} \la \Wb_{\hat{\vx},\hat{\vy}} - \Wb_{\vx,\vy}\ra \\
 & \leq \frac{Cr}{\eps^{d+1}} \#\lb \vy\in\Gamma_N\,:\, |\vx-\vy|\leq 2(\eps+r)\rb \\
 & \leq \frac{CrN\beta(\eps+r)^d}{\eps^{d+1}} \\
 & \leq CN\beta^{\frac32}
\end{align*}
by Lemma~\ref{lem:app:BallBound} with probability at least $1-CNe^{-cN\beta\eps^d}$.
And, if $\eta = \one_{\cdot<1}$,
\[ \la \hat{p}_{N,\eps}(\hat{\vx}) - p_{N,\eps}(\vx) \ra \leq \frac{1}{\eps^d} \#\lb \vy\in\Gamma_N\,:\, \eps-2r\leq |\vx-\vy|\leq\eps+2r\rb \leq CN\beta^{\frac32} \]
by Lemma~\ref{lem:app:BallBound} with probability at least $1-CNe^{-cN\beta\eps^d}$.
Again, by assuming that $c$ is sufficiently small (where $r\leq c\beta\eps$) we can make $C$ in the above bound arbitrarily small so that
\[ \hat{p}_{N,\eps}(\hat{\vx}) \geq p_{N,\eps}(\vx) - \la \hat{p}_{N,\eps}(\hat{\vx}) - p_{N,\eps}(\vx) \ra \geq cN\beta - CN\beta^{\frac32} \geq cN\beta \]
by Lemma~\ref{lem:app:iidAss} with probability at least $1-CNe^{-cN\beta\eps^d}$.

For the final inequality we have
\begin{align*}
\la \hat{\cL}_{N} \varphi(\hat{\vx}) - \cL_{N}\varphi(\vx) \ra & \leq \la \sum_{\vy\in\Omega_N} \l \Wb_{\hat{\vx},\hat{\vy}} - \Wb_{\vx,\vy}\r \l \varphi(\hat{\vx}) - \varphi(\hat{\vy})\r \ra \\
 & \hspace{1cm} + \la \sum_{\vy\in\Omega_N} \Wb_{\vx,\vy} \l \varphi(\hat{\vx}) - \varphi(\hat{\vy}) - \varphi(\vx) + \varphi(\vy) \r \ra \\
 & \leq \|\varphi\|_{\Cr{1}(\overline{\Omega})}(\eps+2r) \sum_{\vy\in\Omega_N} \la \Wb_{\hat{\vx},\hat{\vy}} - \Wb_{\vx,\vy}\ra \\
 & \hspace{1cm} + 2r\|\varphi\|_{\Cr{1}(\overline{\Omega})} \sum_{\vy\in\Omega_N} \Wb_{\vx,\vy} \\
 & \leq CN\|\varphi\|_{\Cr{1}(\overline{\Omega})}(\eps+2r)\sqrt{\beta} + CN\|\varphi\|_{\Cr{1}(\overline{\Omega})}\sqrt{\beta}\eps \\
 & \leq CN\|\varphi\|_{\Cr{1}(\overline{\Omega})}\sqrt{\beta}\eps
\end{align*}
by Lemma~\ref{lem:app:iidAss} with probability at least $1-CNe^{-cN\beta\eps^d}$.
\end{proof}

It is now very easy to prove Theorem~\ref{thm:Setting:Robust:RobRadCorDat}.

\begin{proof}[Proof of Theorem~\ref{thm:Setting:Robust:RobRadCorDat} and Theorem~\ref{thm:MainRes:LargeDataLimit}.]
By Theorem~\ref{thm:app:Conv}, Lemmas~\ref{lem:app:LapStab}, \ref{lem:app:iidAss} and~\ref{lem:app:AdvAss}, and the Lipschitz condition on $\ell$ we have
\begin{align*}
|u(\vx) - \ell(\vx)| & \leq \frac{C\eps}{\sqrt{\beta}}\log\frac{\sqrt{\beta}}{\eps} \\
|\hat{w}(\hat{\vx}) - \ell(\hat{\vx})| & \leq \frac{C\eps}{\sqrt{\beta}}\log\frac{\sqrt{\beta}}{\eps} \\
|\hat{w}(\hat{\vx}) - \hat{u}(\hat{\vx})| & \leq C\sqrt{\beta}\eps \\
|\ell(\hat{\vx}) - \ell(\vx)| & \leq C\sqrt{\beta}\eps
\end{align*}
with probability at least $1-CNe^{-cN\beta\eps^d}$.
So by the triangle inequality
\[ |u(\vx) - \hat{u}(\hat{\vx})| \leq \frac{C\eps}{\sqrt{\beta}}\log\frac{\sqrt{\beta}}{\eps} \]
and
\[ |\ell(\vx) - \hat{u}(\hat{\vx})| \leq \frac{C\eps}{\sqrt{\beta}}\log\frac{\sqrt{\beta}}{\eps} \]
with probability at least $1-CNe^{-cN\beta\eps^d}$.
This implies that $\cR_\delta(\Omega^\prime,u,D_N)\geq r$ with probability at least $1-CNe^{-cN\beta\eps^d}$.
\end{proof}

Corollary~\ref{cor:MainRes:LapClass} is not difficult to show using the Theorem~\ref{thm:Setting:Robust:RobRadCorDat}.

\begin{proof}[Proof of Corollary~\ref{cor:MainRes:LapClass}.]
We work on the set of realisations of $\{\vx_i\}_{i=1}^\infty$ such that the conclusions of Theorem~\ref{thm:Setting:Robust:RobRadCorDat} hold; i.e. the following statements hold with probability at least $1-CNe^{-cN\beta\eps^d}$.
Let us define $\hat{v}(\hat{\vx}) = v(\hat{\vx};\hat{D}_N)$ and 
\[ \Omega_\delta = \lb \vx \, :\, \frac12 - 2\delta > \ell(\vx) \text{ or } \ell(\vx) \geq \frac12 + 2 \delta\rb. \]
By Assumption~\ref{eq:MainRes:ExAss} we have that $\Vol(\Omega_\delta^c)\leq C\delta$.
Moreover, for $\vx\in\Omega_N$ such that $\ell(\vx)\geq \frac12 + 2\delta$ we have
\[ \hat{u}(\hat{\vx}) \geq u(\vx) - \delta \geq \ell(\vx) - 2\delta \geq \frac12 \]
so $|v(\vx) - \hat{v}(\hat{\vx})|=0$.
And similarly, if $\vx\in\Omega_N$ is such that $\ell(\vx)< \frac12 - 2\delta$ we have
\[ \hat{u}(\hat{\vx}) \leq u(\vx) + \delta \leq \ell(\vx) + 2\delta < \frac12 \]
so $|v(\vx) - \hat{v}(\hat{\vx})|=0$.
Hence, for any $\vx\in\Omega_N\cap \Omega_\delta$ we have $|v(\vx) - \hat{v}(\hat{\vx})|=0$.
Therefore, $\cR_0(\Omega_\delta,v,D_N) \geq r$ as required.
\end{proof}

\section{Robust Accuracy Under Different Black-Box Adversarial Attacks}\label{Appendix:RobustAcc:BlackBox:Attack}
In this section, we report the robustness results of the GL- and $k$NN-based classifiers under the three different BB attacks, where these BB attacks attack the substitute models: logistic regression (LR), neural net (NN), and kernel classifier (Kernel). 
For Abalone classification, GL-based classifiers performs better than the corresponding $k$NN-based classifiers; in particular, the RobustGL (GL classifier with the pruned dataset that satisfies the $a$-separation condition \citep{wang2017analyzing}) outperforms all the other classifiers under three BB attacks with different maximum perturbation. Data augmentation with adversarial data can also enhance the classifiers' adversarial robustness. These adversarial defenses can also improve $k$NN-based classifiers for Abalone classification. For Halfmoon classification, GL-based classifiers again outperforms $k$NN-based classifiers consistently. Furthermore, with adversarial data augmentation, the classifiers' robustness can be significantly improved.
\begin{figure}[!ht]
\centering
\begin{tabular}{ccc}
\hskip -0.6cm\includegraphics[clip, trim=0cm 0cm 0cm 0cm,width=0.33\columnwidth]{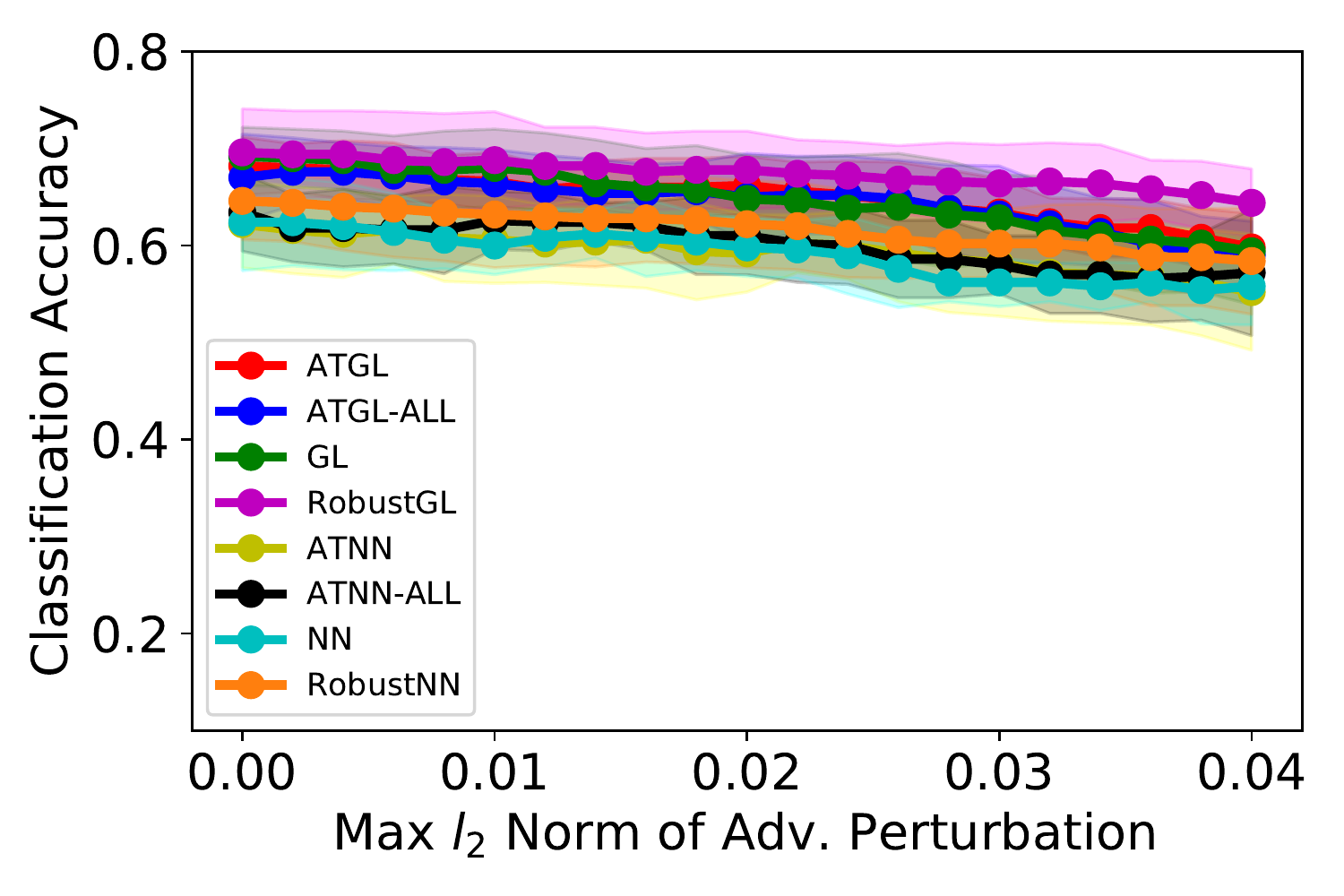}&
\hskip -0.3cm\includegraphics[clip, trim=0cm 0cm 0cm 0cm,width=0.33\columnwidth]{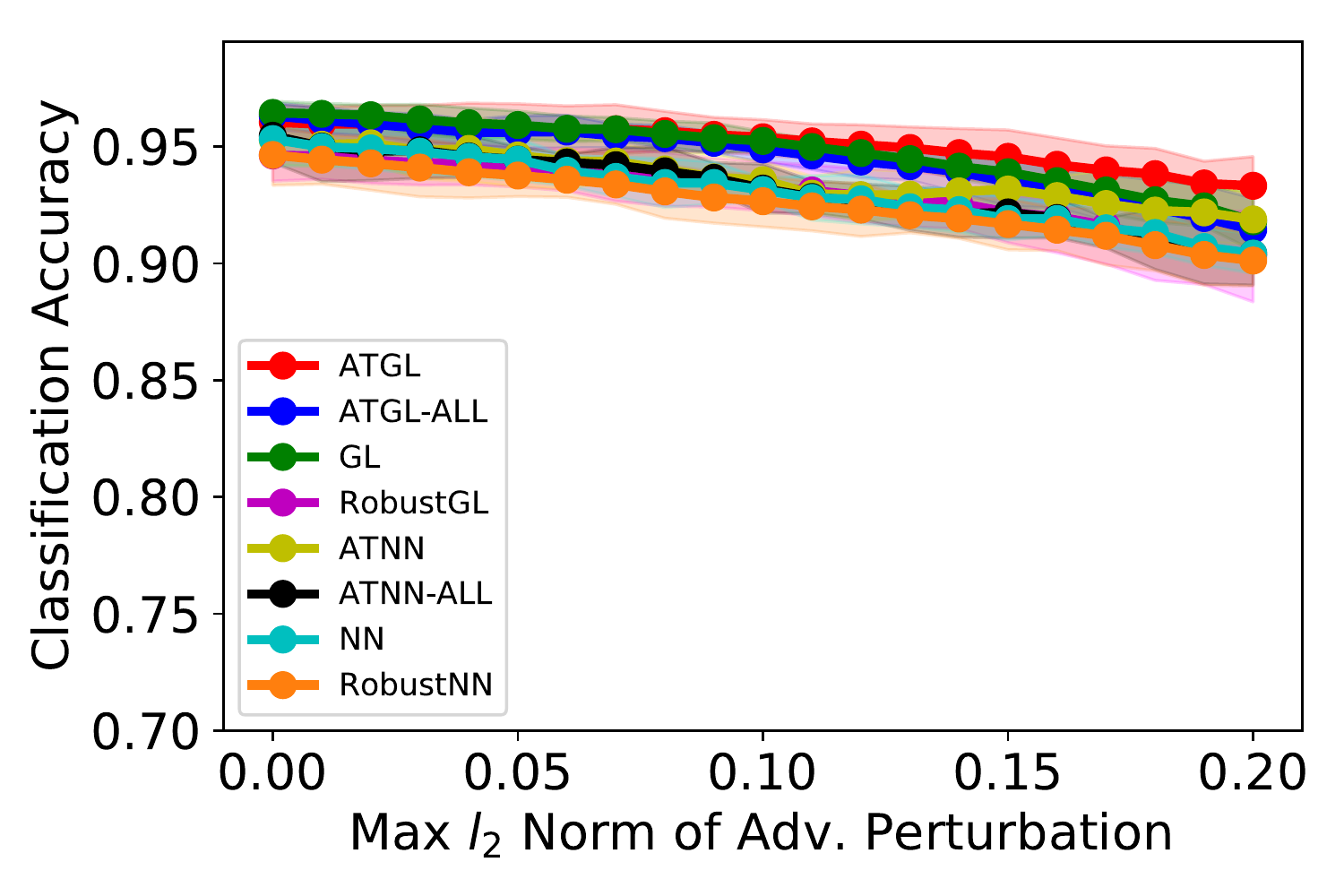}&
\hskip -0.3cm\includegraphics[clip, trim=0cm 0cm 0cm 0cm,width=0.33\columnwidth]{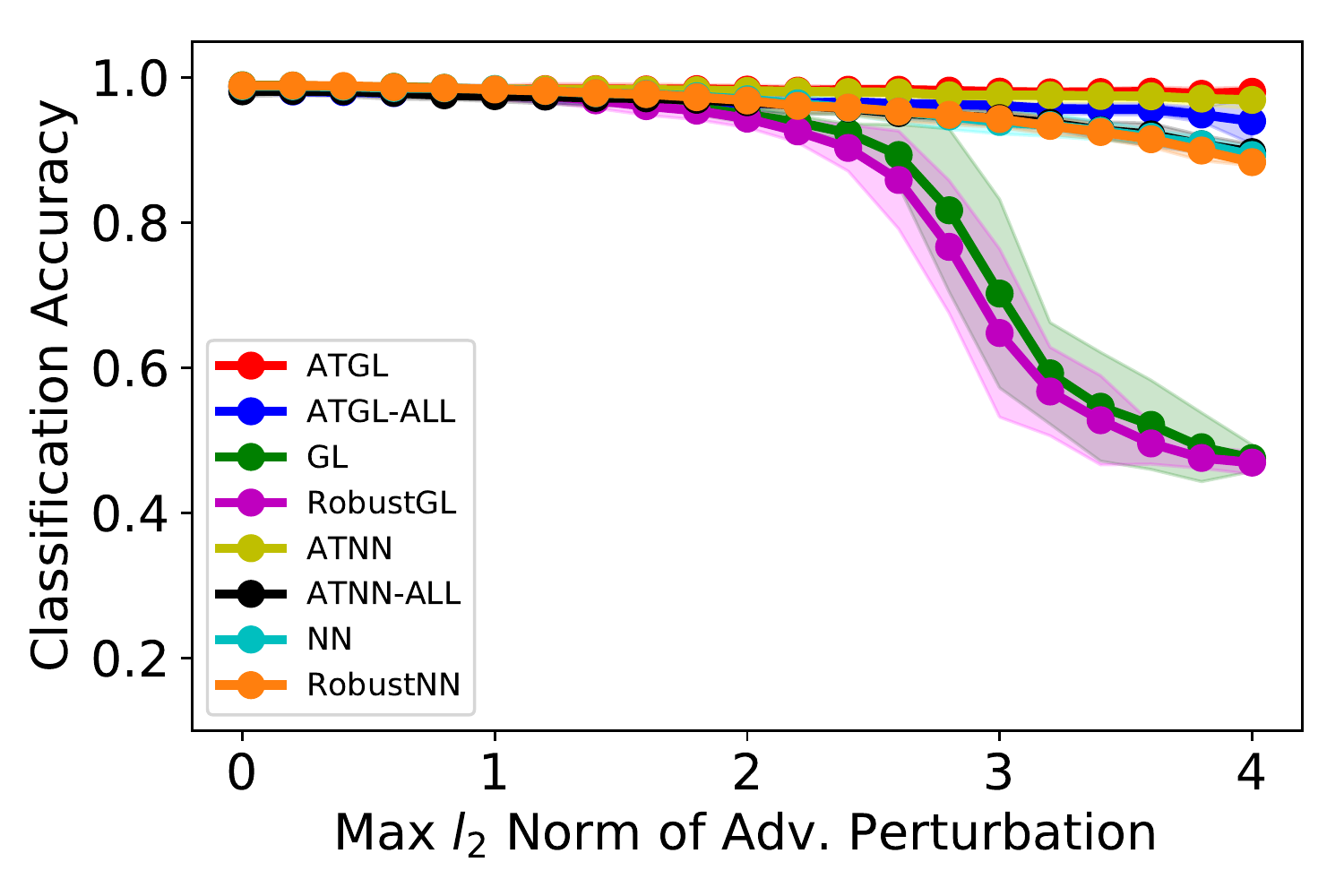}\\
{\footnotesize (a) Abalone} & {\footnotesize (b) Halfmoon} & {\footnotesize (c) MNIST 1v7} \\
\end{tabular}
\caption{Robust accuracies of GL- vs. $k$NN-based classifiers for three different datasets classification under black-box attack using logistic regression as the substitute model with different maximum perturbation in $l_2$-norm. (Best viewed on a computer screen.)}
\label{fig:acc:lr}
\end{figure}

For MNIST 1v7 classification, GL-based classifiers are not always more robust than $k$NN-based classifiers. Nevertheless, ATGL or ATGL-ALL (GL with adversarial data augmentation) gives the most robust classification under different BB attacks. 

\begin{figure}[!ht]
\centering
\begin{tabular}{ccc}
\hskip -0.6cm\includegraphics[clip, trim=0cm 0cm 0cm 0cm,width=0.33\columnwidth]{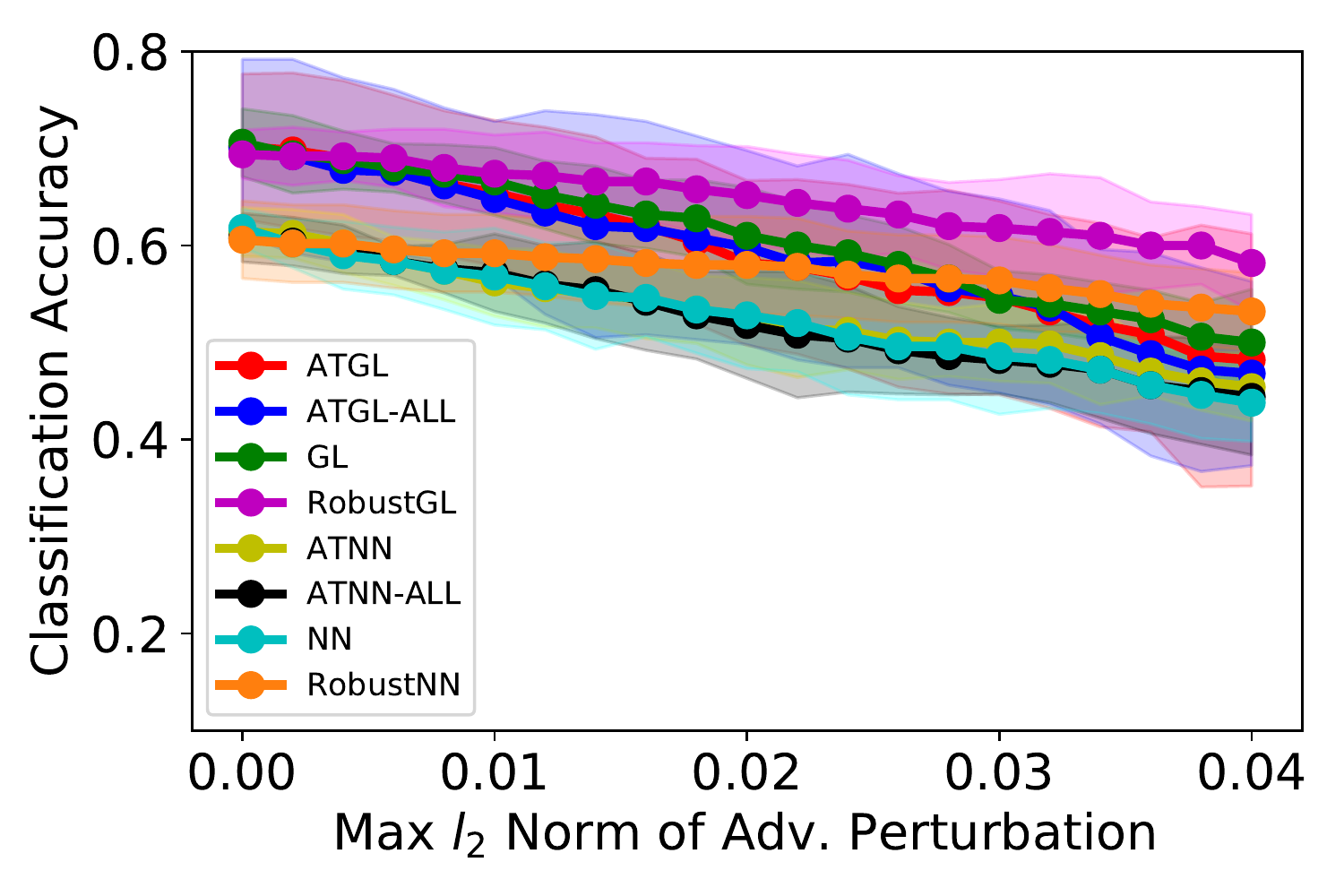}&
\hskip -0.3cm\includegraphics[clip, trim=0cm 0cm 0cm 0cm,width=0.33\columnwidth]{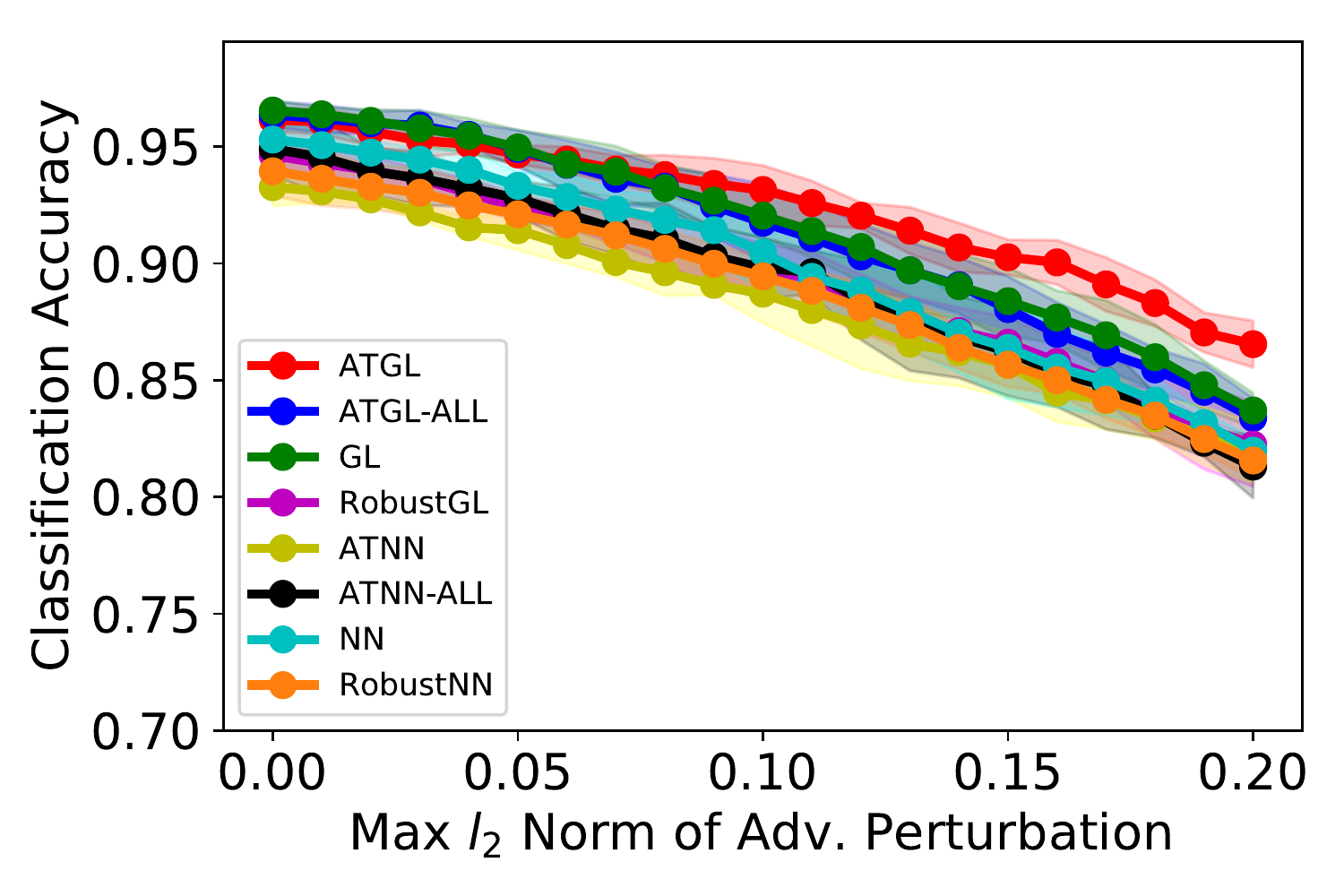}&
\hskip -0.3cm\includegraphics[clip, trim=0cm 0cm 0cm 0cm,width=0.33\columnwidth]{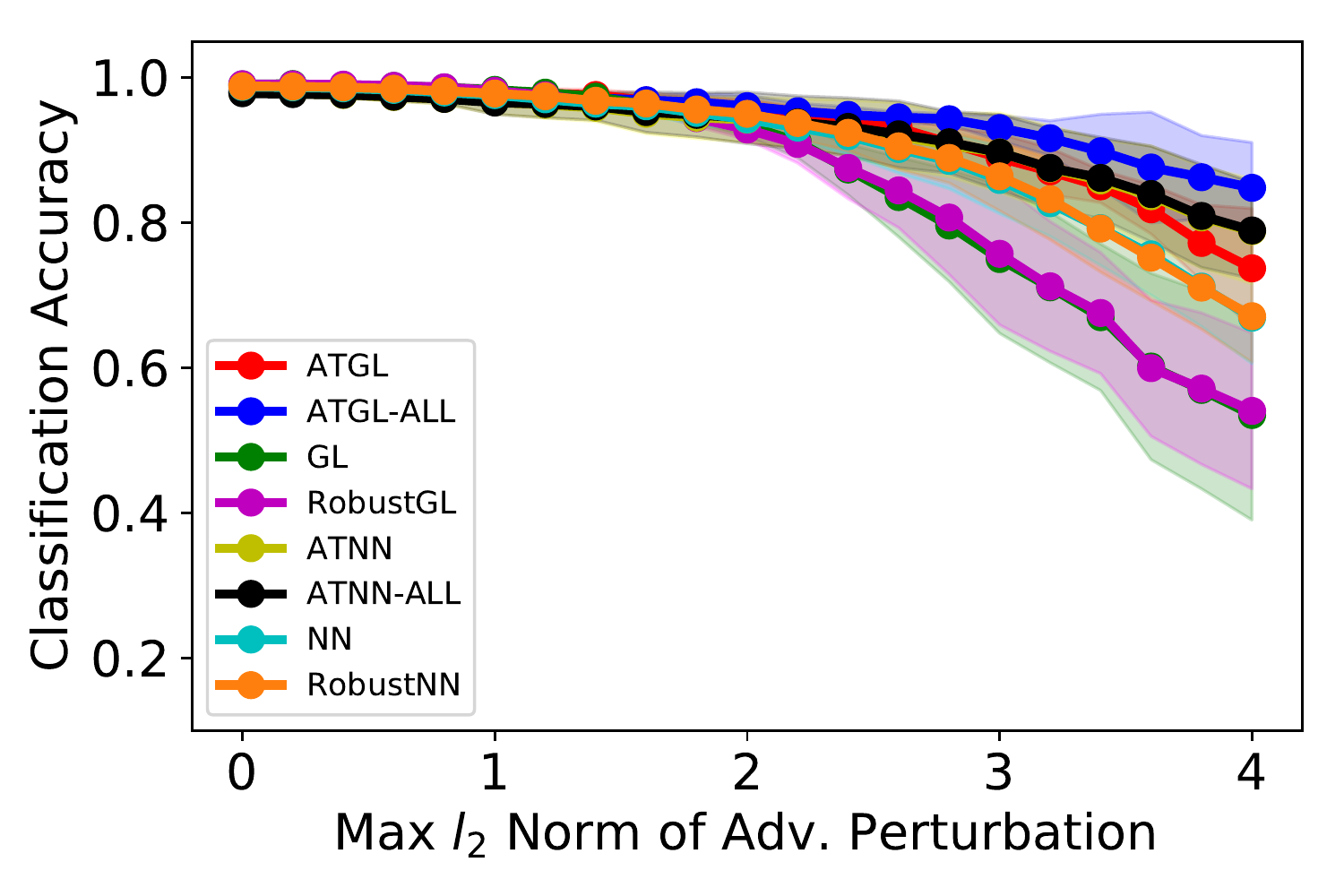}\\
{\footnotesize (a) Abalone} & {\footnotesize (b) Halfmoon} & {\footnotesize (c) MNIST 1v7} \\
\end{tabular}
\caption{Robust accuracies of GL vs. $k$NN classifiers for three different datasets classification under BB attack using kernel model as the substitute model with different maximum perturbation in $l_2$-norm. (Best viewed on a computer screen.)}
\label{fig:acc:kernel}
\end{figure}

\begin{figure}[!ht]
\centering
\begin{tabular}{ccc}
\hskip -0.6cm\includegraphics[clip, trim=0cm 0cm 0cm 0cm,width=0.33\columnwidth]{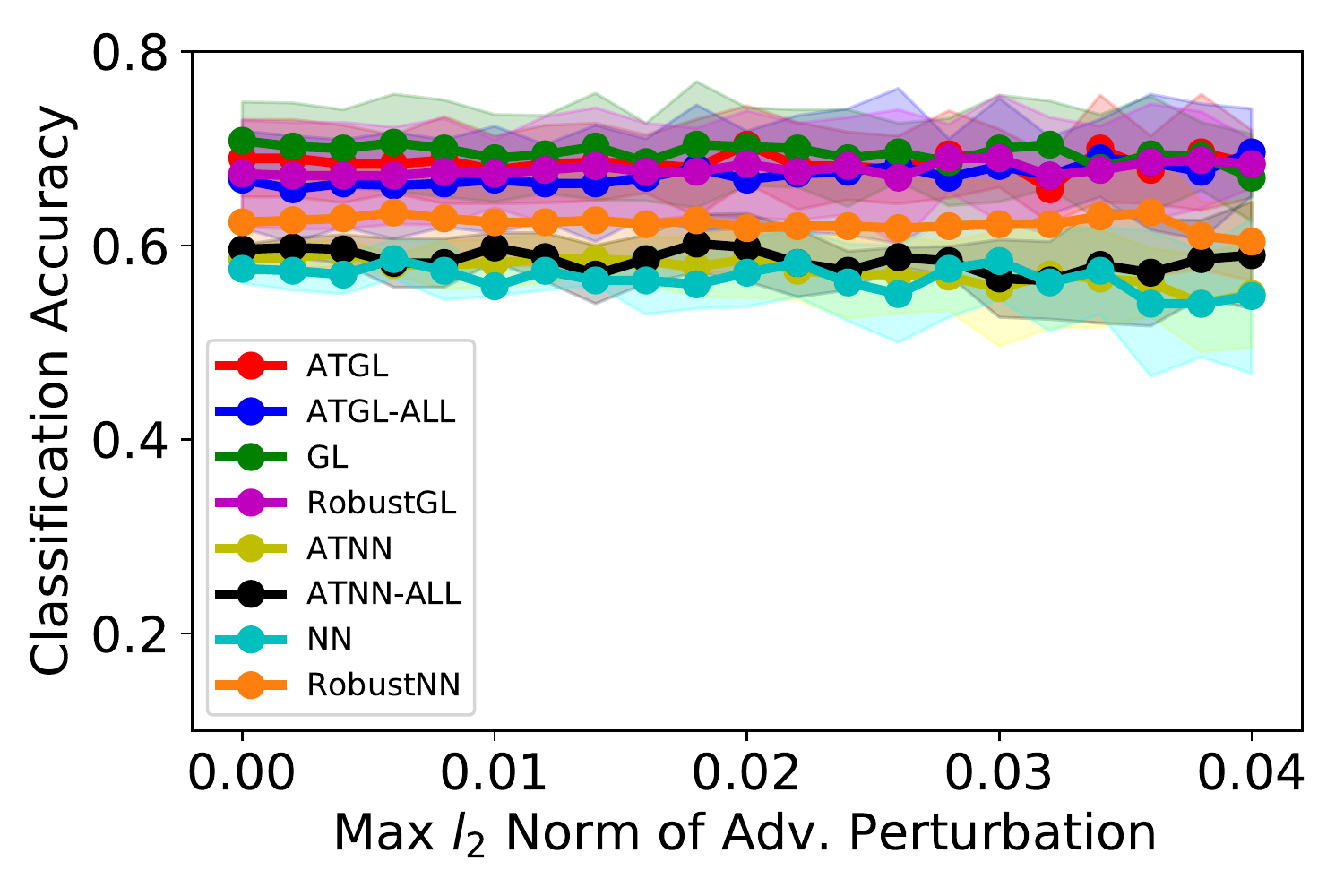}&
\hskip -0.3cm\includegraphics[clip, trim=0cm 0cm 0cm 0cm,width=0.33\columnwidth]{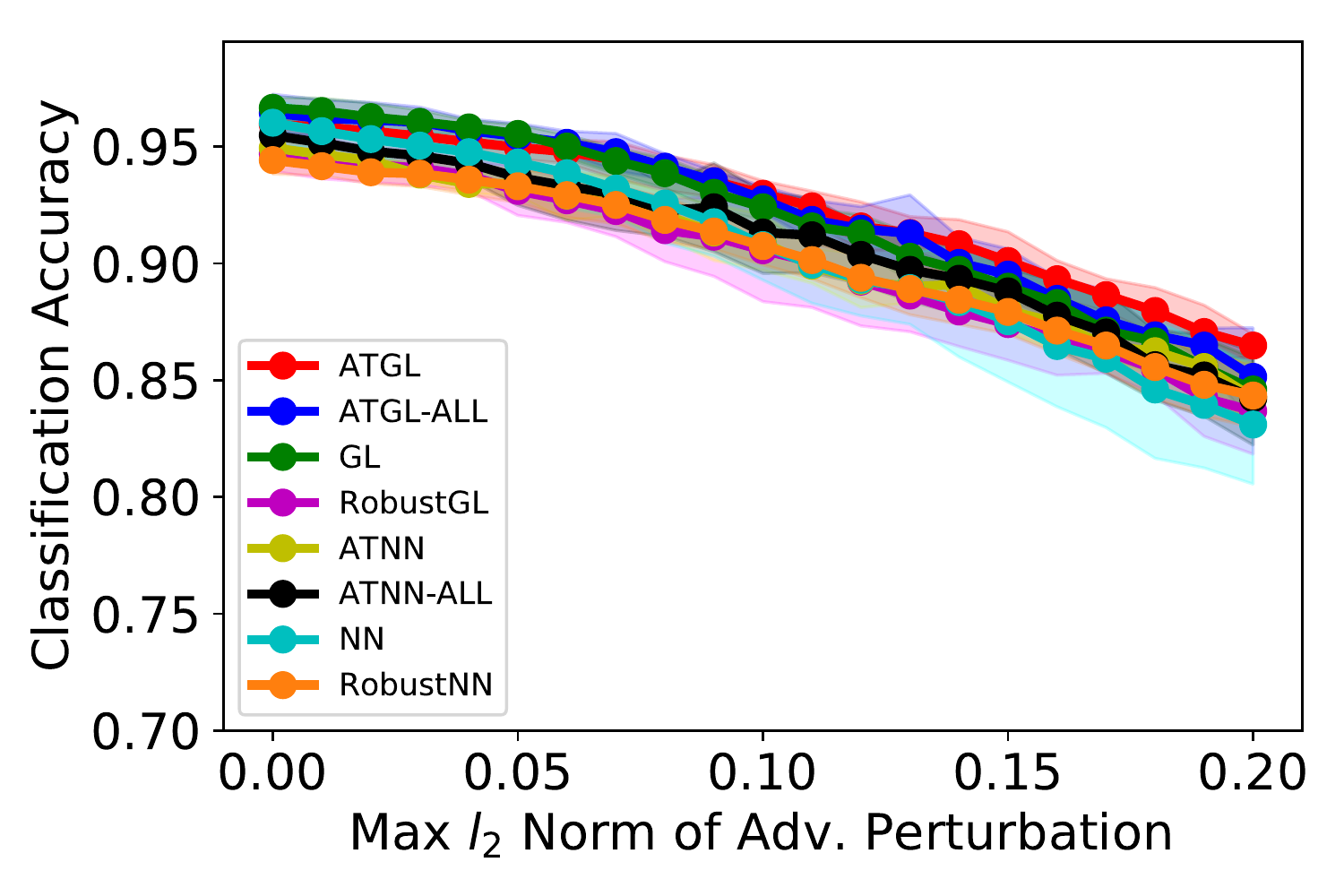}&
\hskip -0.3cm\includegraphics[clip, trim=0cm 0cm 0cm 0cm,width=0.33\columnwidth]{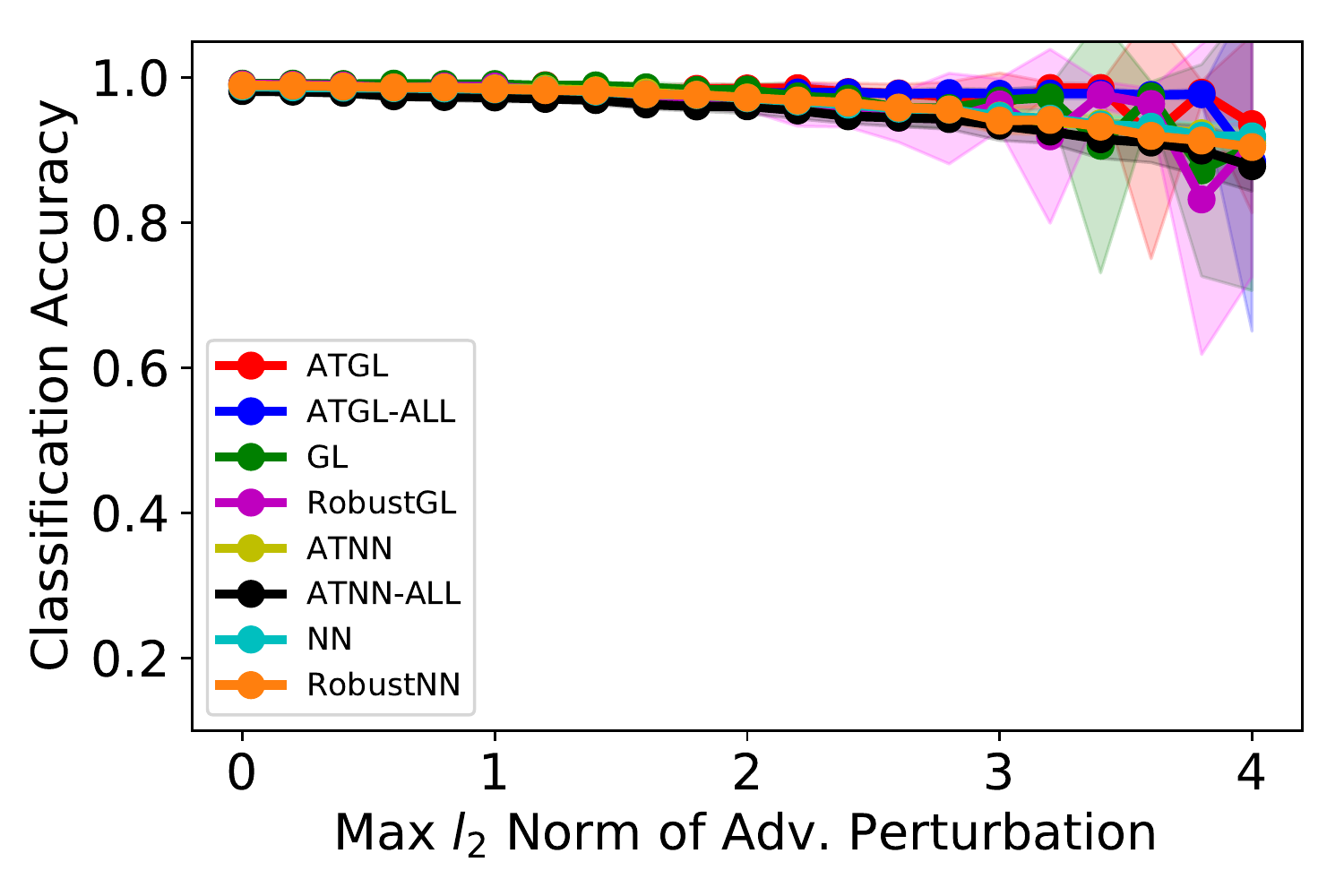}\\
{\footnotesize (a) Abalone} & {\footnotesize (b) Halfmoon} & {\footnotesize (c) MNIST 1v7} \\
\end{tabular}
\caption{Robust accuracies of GL vs. $k$NN classifiers for three different datasets classification under BB attack using neural net as the surrogate model with different maximum perturbation in $l_2$-norm. (Best viewed on a computer screen.)}
\label{fig:acc:nn}
\end{figure}

\section{Robust Accuracy with Different Number of Training Data}\label{appendix:RAcc:vs:Different-num-training-data}
We have numerically shown that as the number of training data increases, the robustness of GL-based classifiers on the Abalone dataset will increase under the WB attacks. In this section, we consider the effects of the number of training data on the classifiers' robustness under BB attacks. As shown in Fig.~\ref{fig:acc:diff:number:training-data:BBAttack:Abalone}, under the BB attack with LR or Kernel as the substitute model, both GL- and $k$NN-based classifiers become more robust as the number of training data increases. As we increase the number of training data from $100$ to $500$, the models classification accuracy can increase more than $5$\% under the BB attacks with different maximum perturbations. For BB attacks with $k$NN as the surrogate model, the GL- and $k$NN- based classifiers accuracy under adversarial attacks fluctuates, the average robust accuracy of the GL-based classifiers also has an upward trend as the number of training data increases.


\begin{figure}[!ht]
\centering
\begin{tabular}{ccc}
\hskip -0.6cm\includegraphics[clip, trim=0cm 0cm 0cm 0cm, width=0.33\columnwidth]{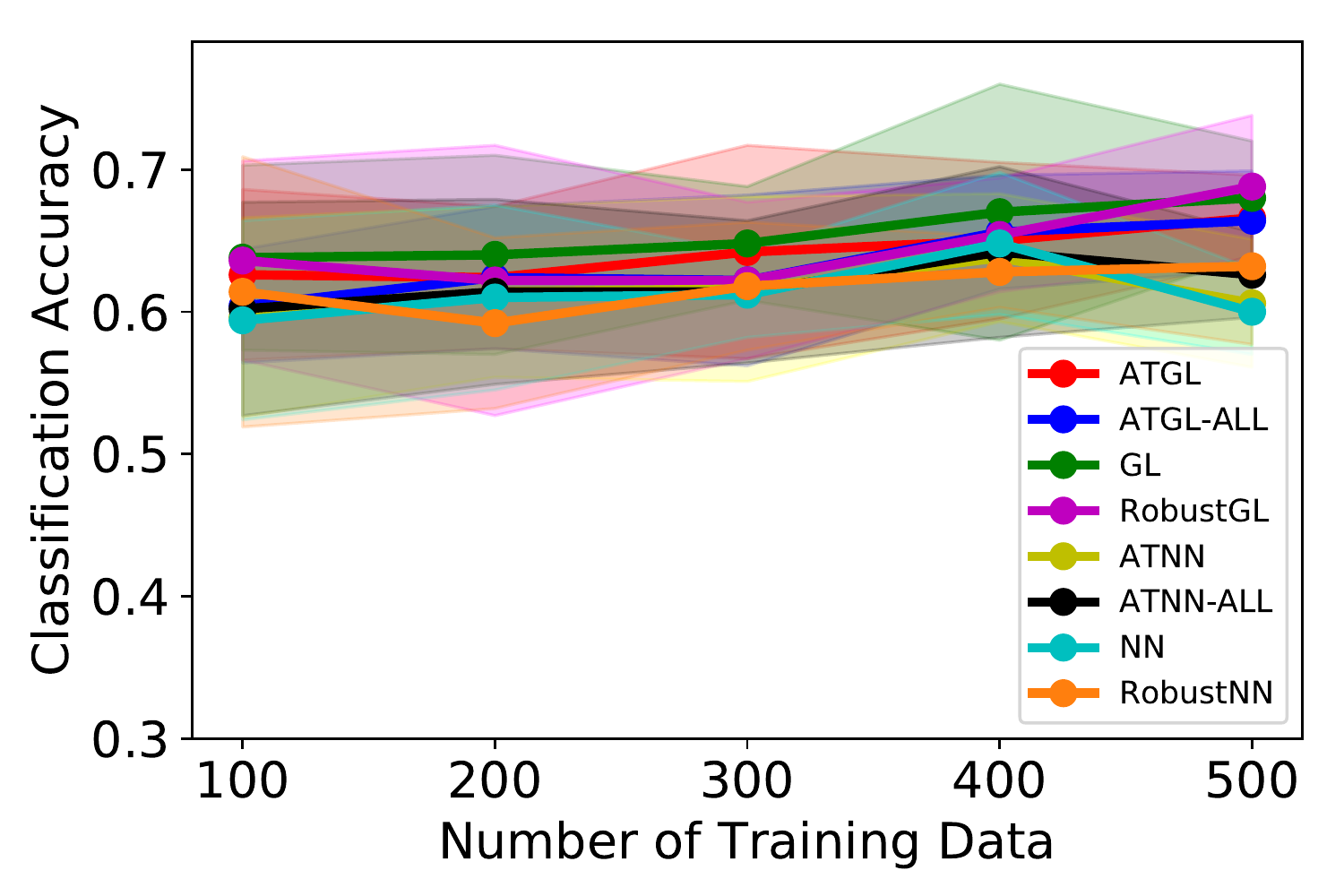}&
\hskip -0.4cm\includegraphics[clip, trim=0cm 0cm 0cm 0cm, width=0.33\columnwidth]{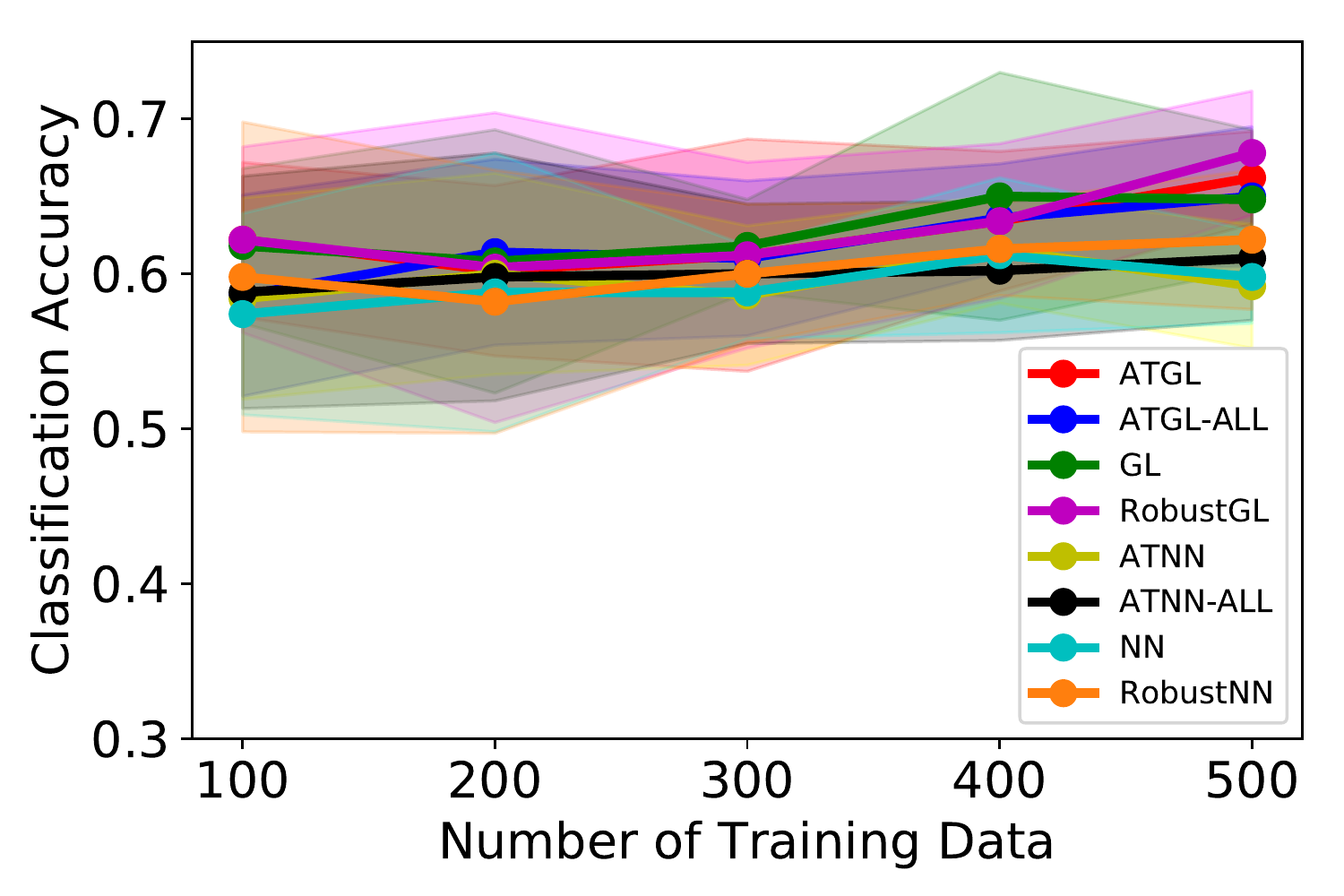}&
\hskip -0.4cm\includegraphics[clip, trim=0cm 0cm 0cm 0cm, width=0.33\columnwidth]{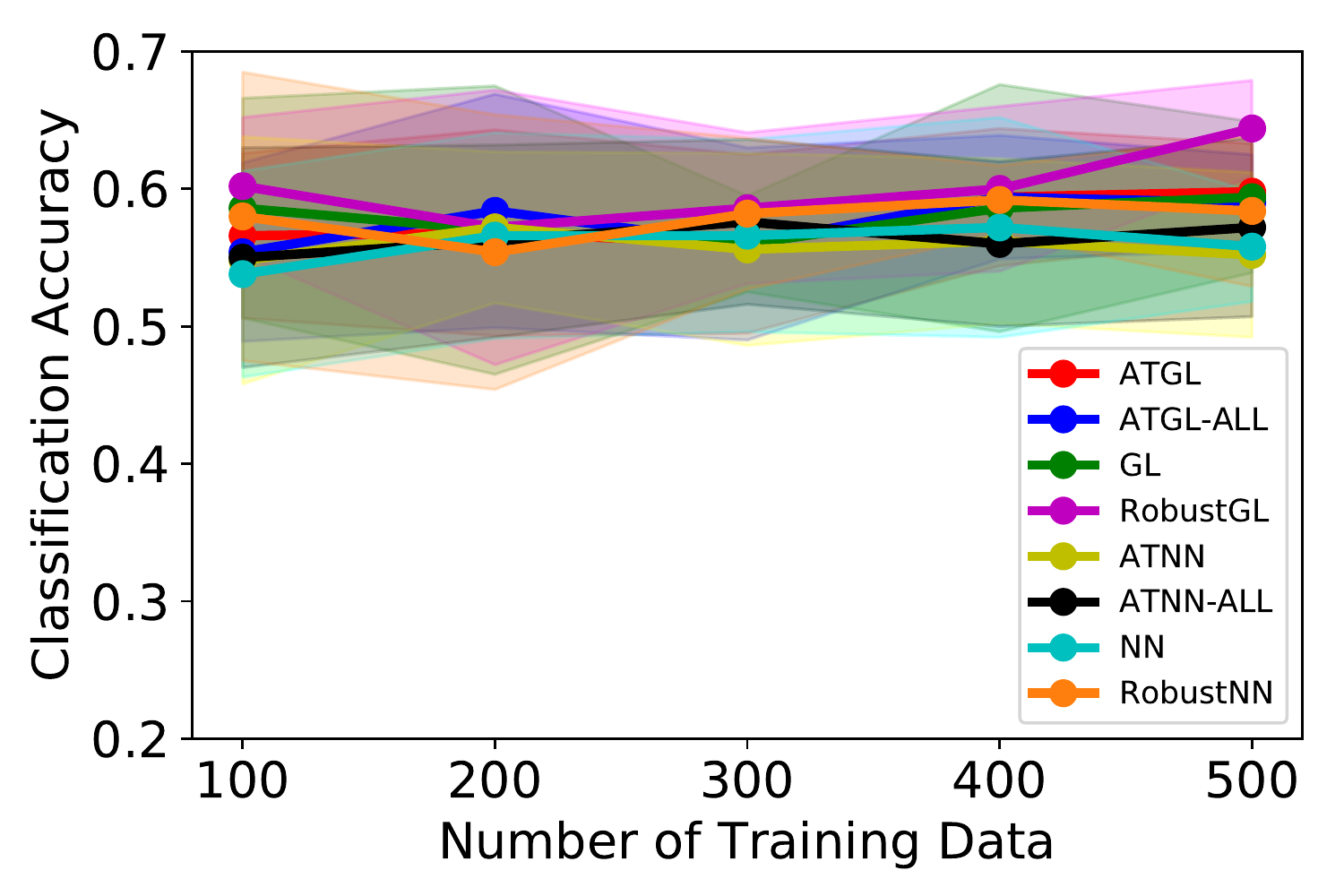}\\
LR, $r=0.01$ & LR, $r=0.02$ & LR, $r=0.04$ \\
\hskip -0.6cm\includegraphics[clip, trim=0cm 0cm 0cm 0cm, width=0.33\columnwidth]{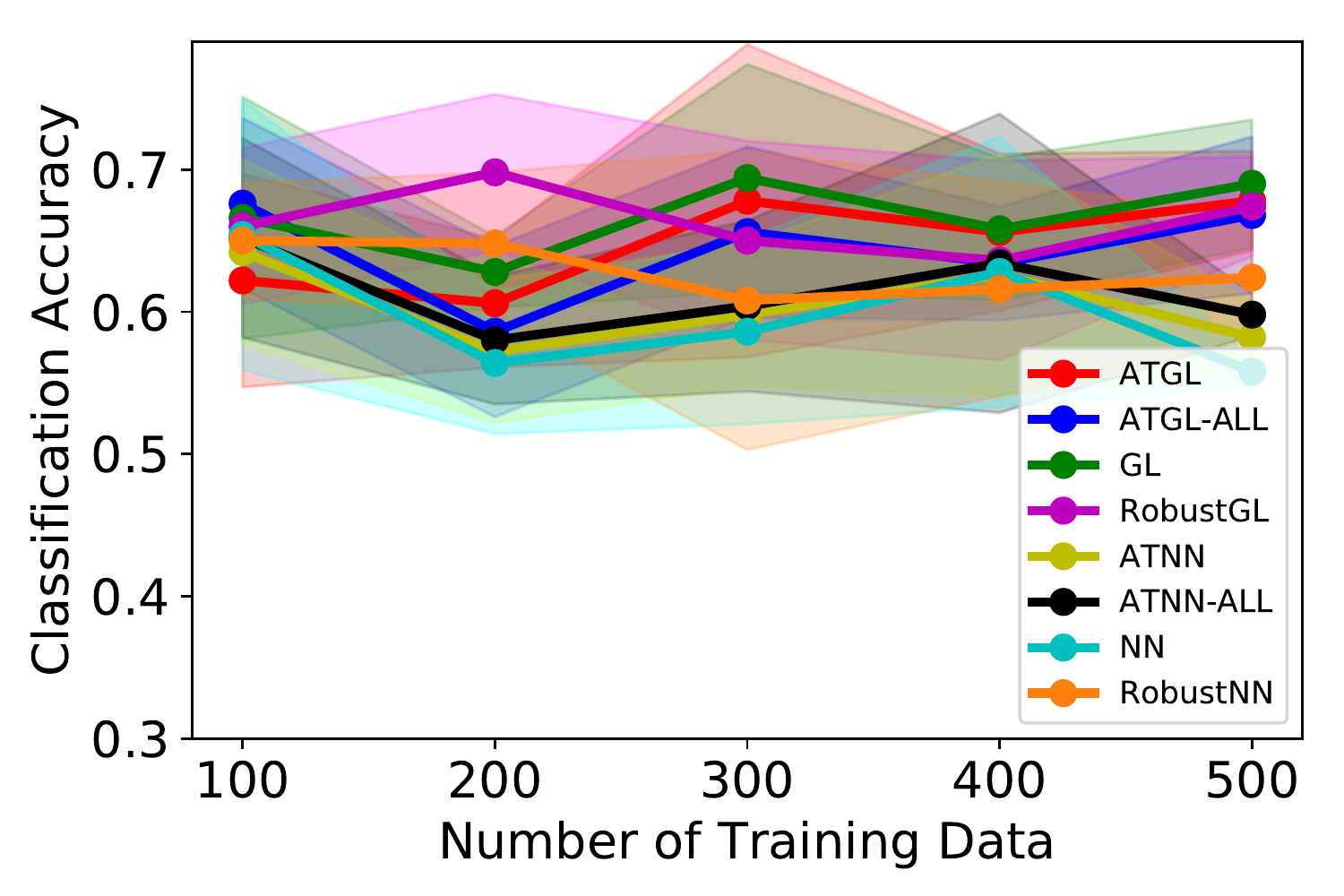}&
\hskip -0.4cm\includegraphics[clip, trim=0cm 0cm 0cm 0cm, width=0.33\columnwidth]{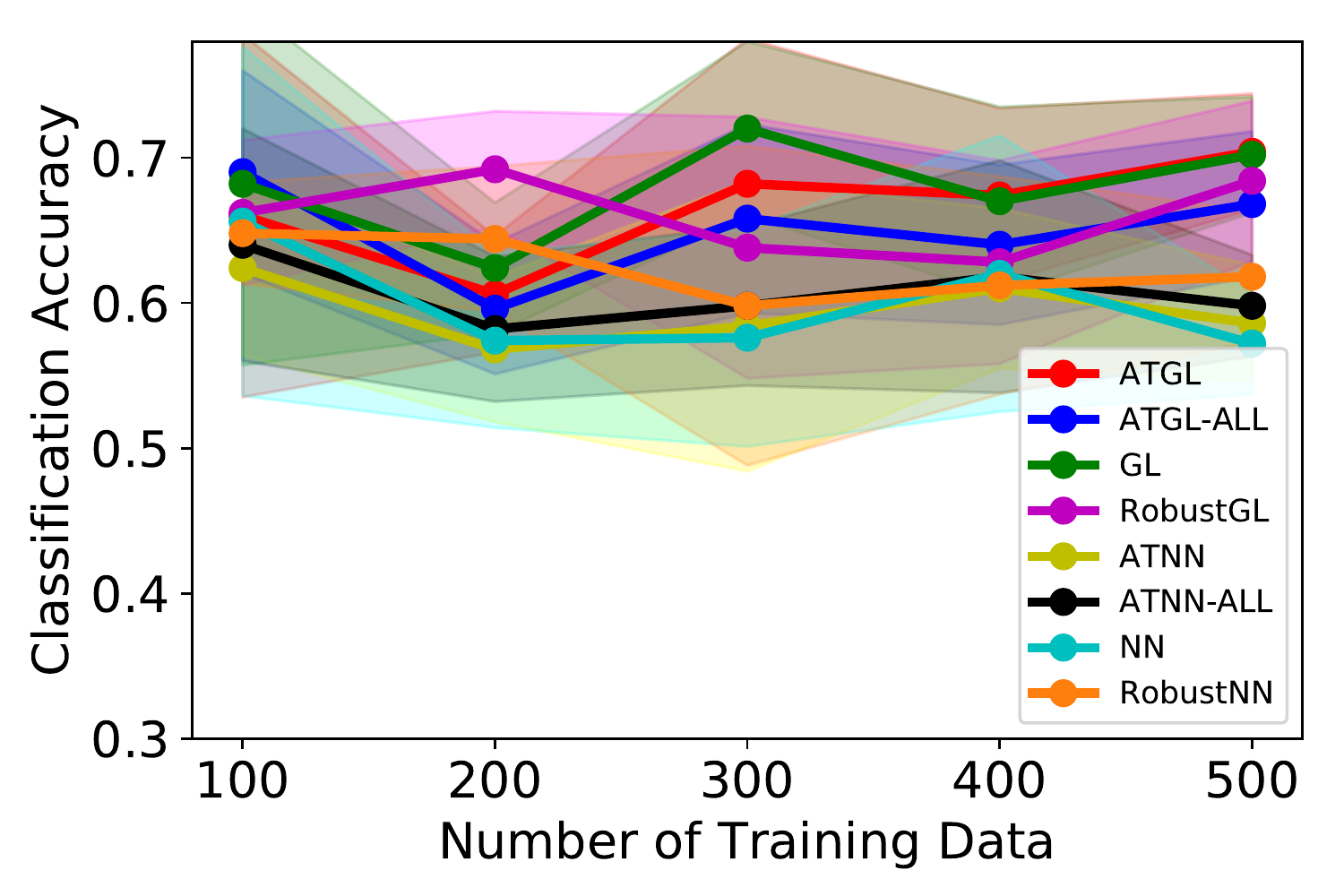}&
\hskip -0.4cm\includegraphics[clip, trim=0cm 0cm 0cm 0cm, width=0.33\columnwidth]{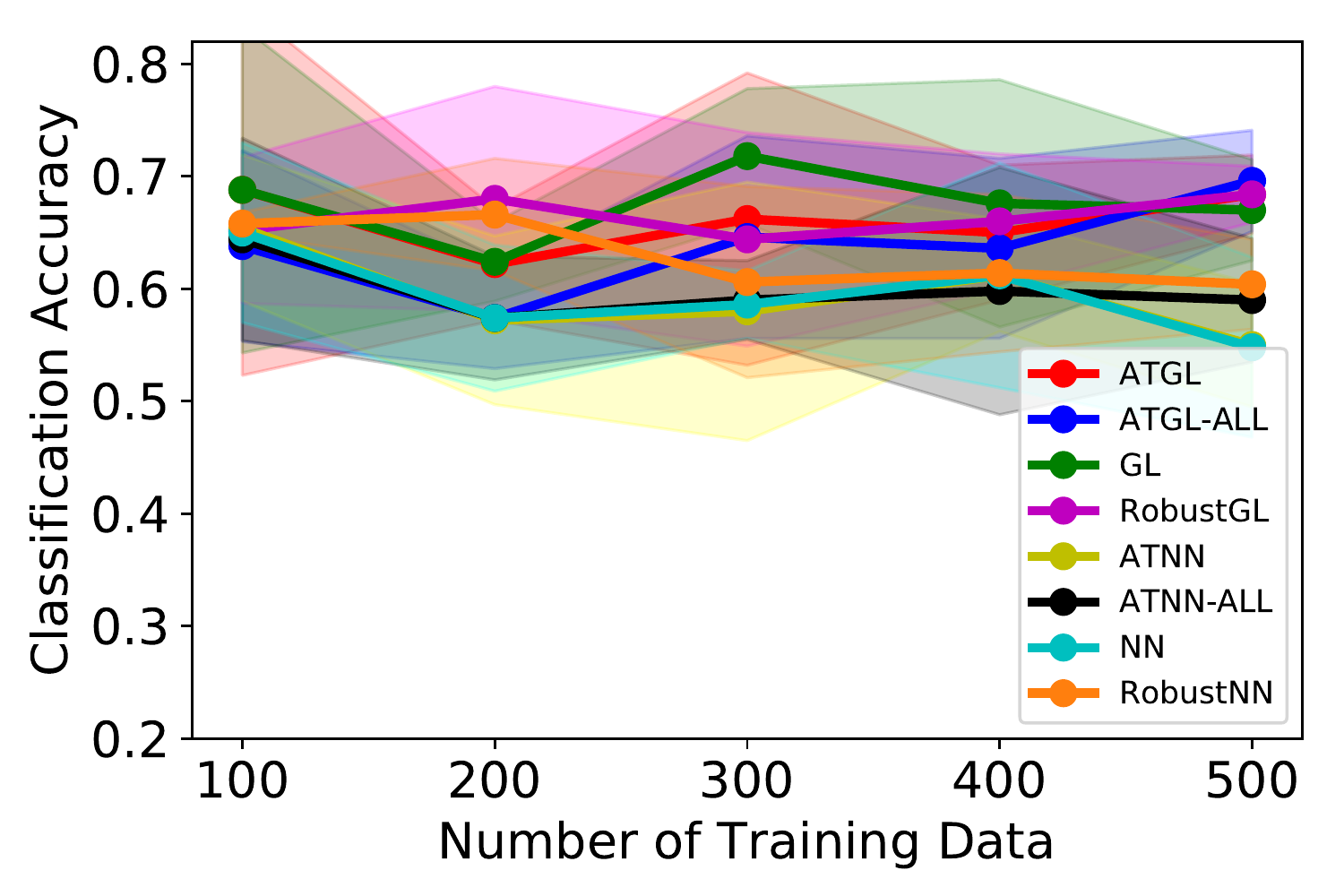}\\
NN, $r=0.01$ & NN, $r=0.02$ & NN, $r=0.04$ \\
\hskip -0.6cm\includegraphics[clip, trim=0cm 0cm 0cm 0cm, width=0.33\columnwidth]{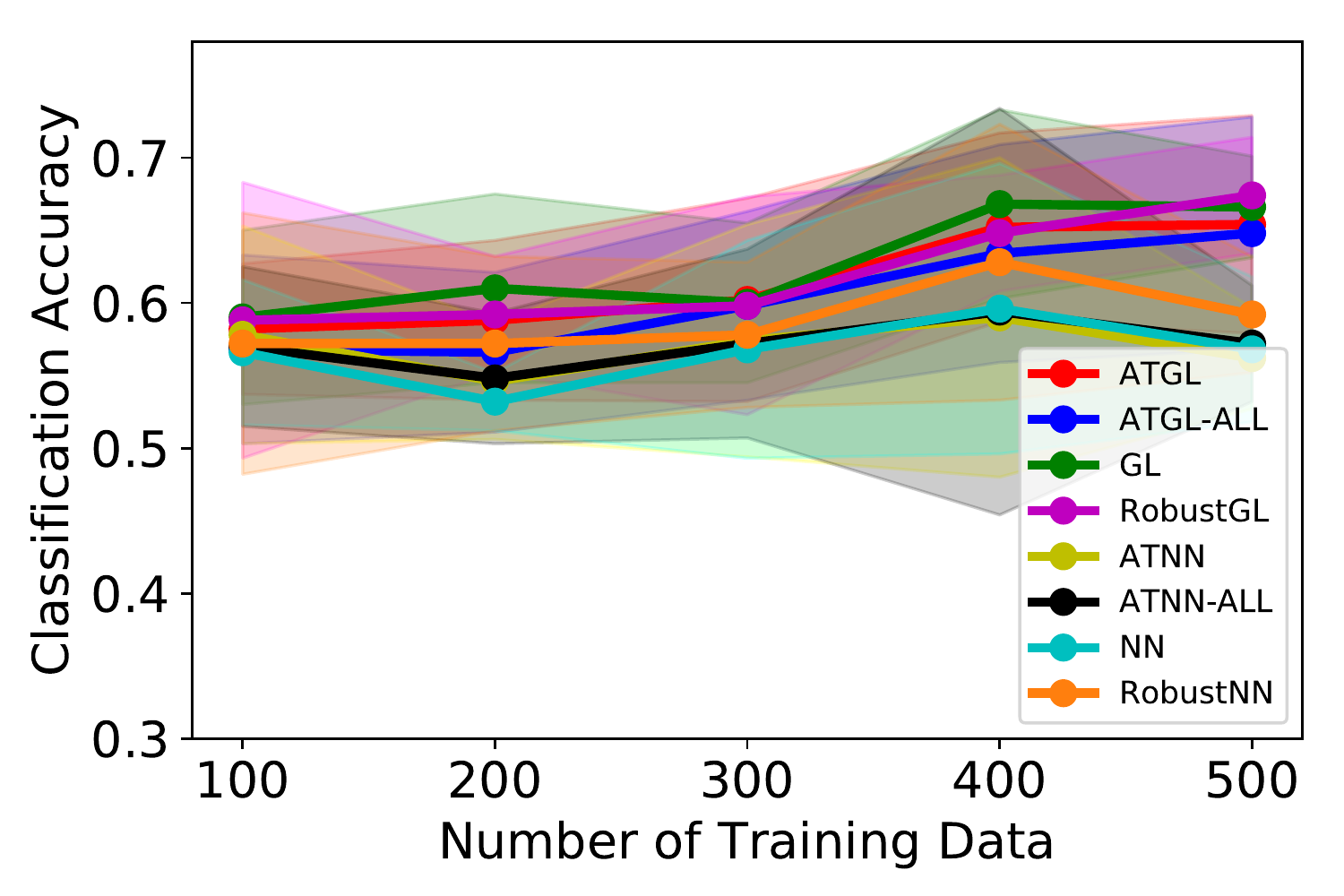}&
\hskip -0.4cm\includegraphics[clip, trim=0cm 0cm 0cm 0cm, width=0.33\columnwidth]{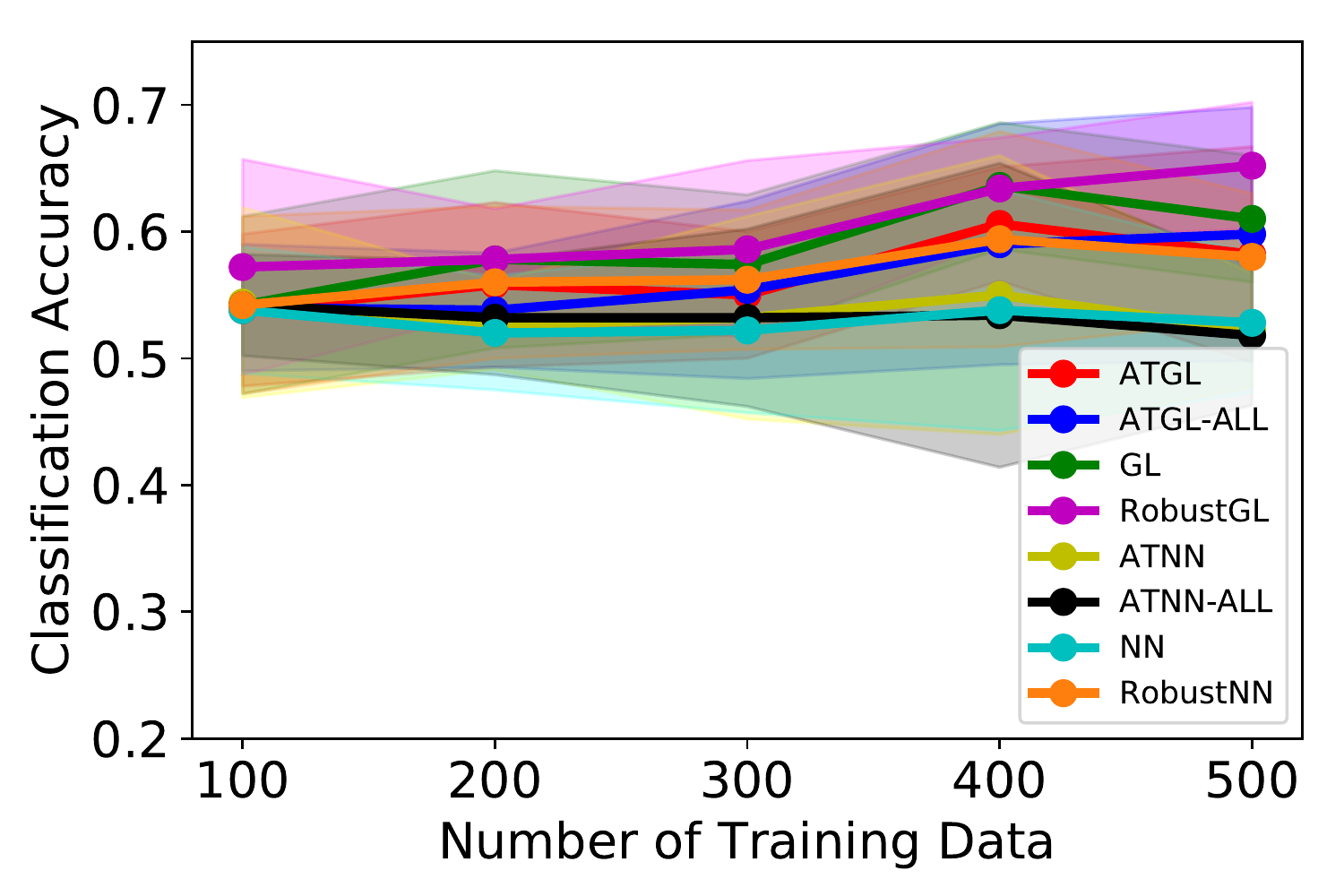}&
\hskip -0.4cm\includegraphics[clip, trim=0cm 0cm 0cm 0cm, width=0.33\columnwidth]{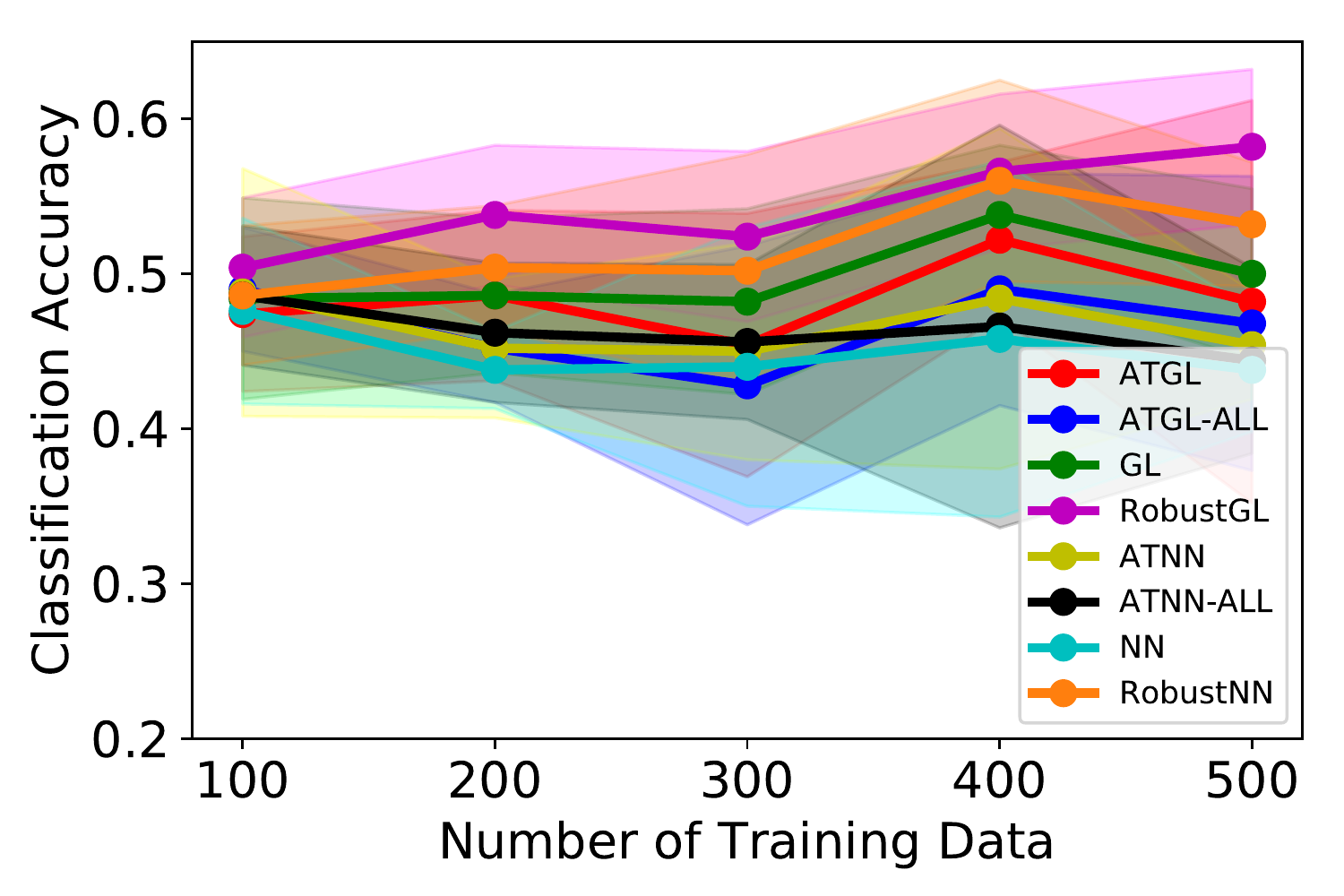}\\
Kernel, $r=0.01$ & Kernel, $r=0.02$ & Kernel, $r=0.04$ \\
\end{tabular}
\caption{Robust accuracies of GL vs. $k$NN classifiers, trained with different numbers of training data, for classifying the Abalone dataset under the BB attacks. For both LR and Kernel surrogate models, both GL and $k$NN classifiers becomes more robust as the number of training data increases. When the NN is used as the substitute model, the accuracy fluctuates as the size of training data increases. (Best viewed on a computer screen.)
}
\label{fig:acc:diff:number:training-data:BBAttack:Abalone}
\end{figure}

\section{Visualizing the Adversarial Examples of GL-Based Classifiers 
}\label{sec:Vis:BB:Attack}
In this part, we visualize the adversarial examples of the MNIST 1v7 and the Halfmoon datasets under adversarial attacks with $r=4$ and $0.2$, respectively. For the MNIST 1v7, under the strong WB attacks ($r=4$); in particular, the KSA attack, many adversarial digit $1$'s are hard for us to classify, and this explains there is a sharp accuracy drop when $r$ is large (see Fig.~\ref{fig:acc:wb-wb-kerbel} (f)). For the Halfmoon dataset, when we apply the strong WB-KSA attack, the adversarial examples enter the opposite territory, which leads to misclassification. The GL-based classifiers leverage the global geometric information to improve the classification robustness in this scenario.

\begin{figure}[!ht]
\centering
\begin{tabular}{cccc}
\hskip -1.0cm
\includegraphics[clip, trim=0cm 0cm 0cm 0cm, width=0.3\columnwidth]{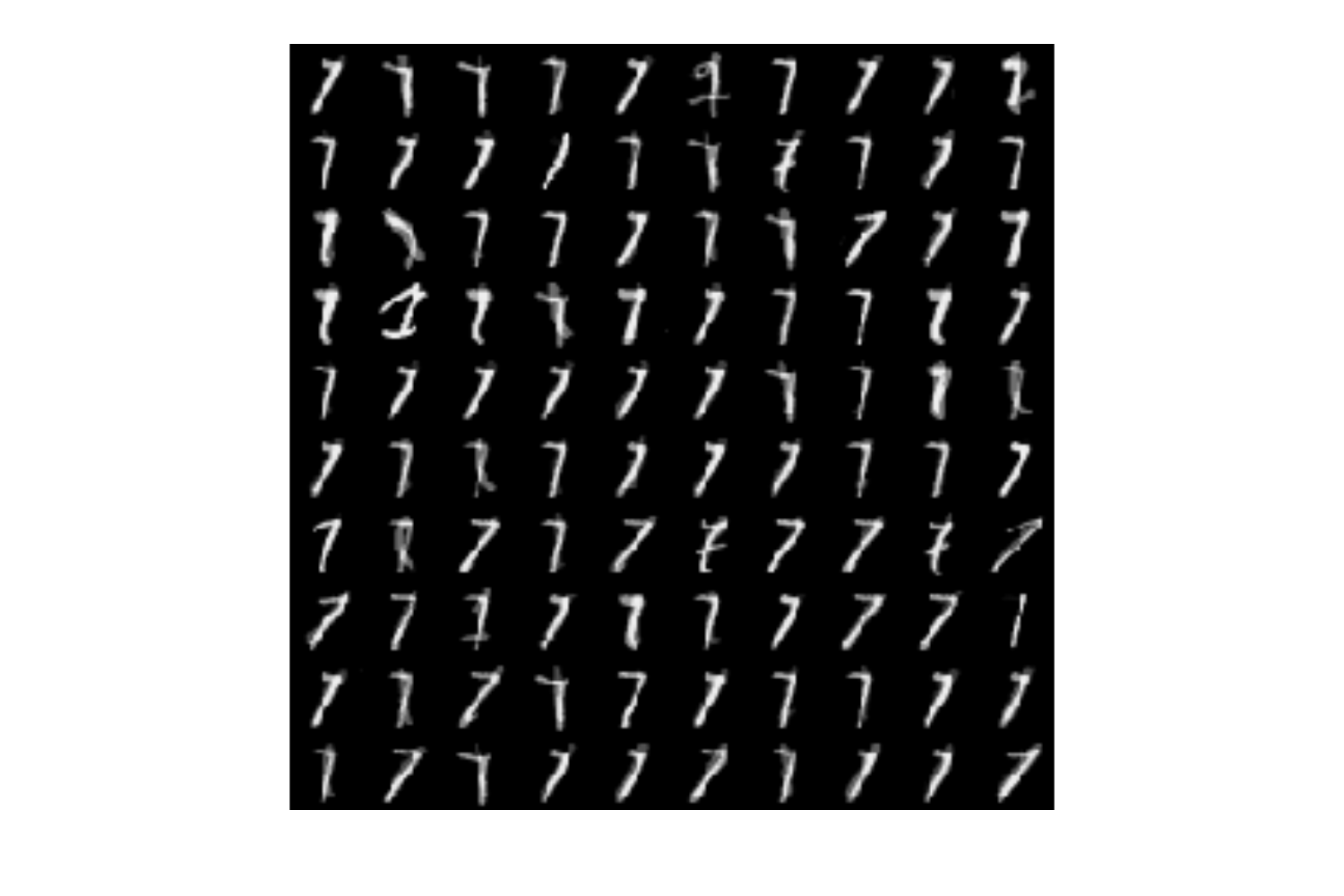}&
\hskip -1.4cm
\includegraphics[clip, trim=0cm 0cm 0cm 0cm, width=0.3\columnwidth]{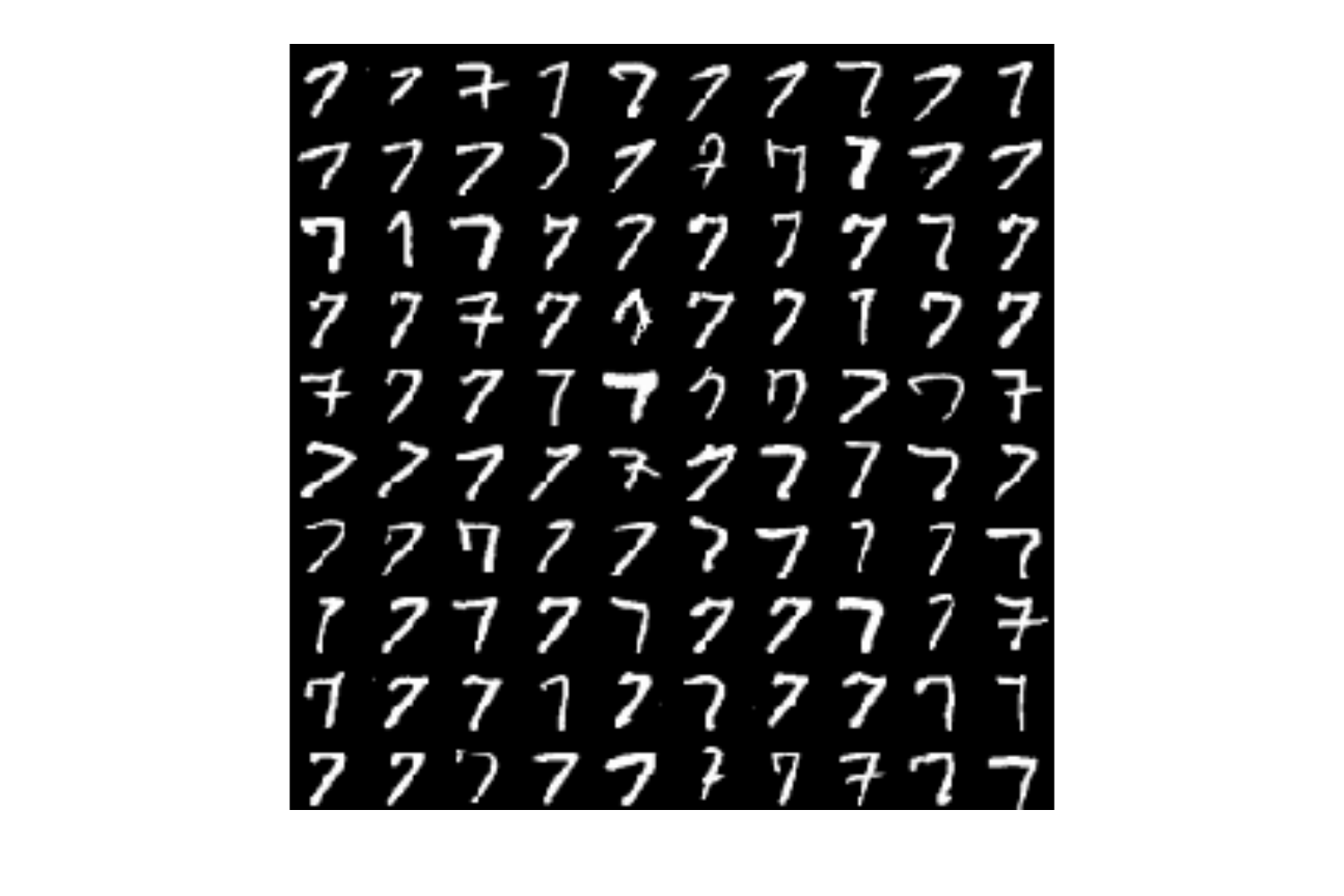}&
\hskip -1.4cm
\includegraphics[clip, trim=0cm 0cm 0cm 0cm, width=0.3\columnwidth]{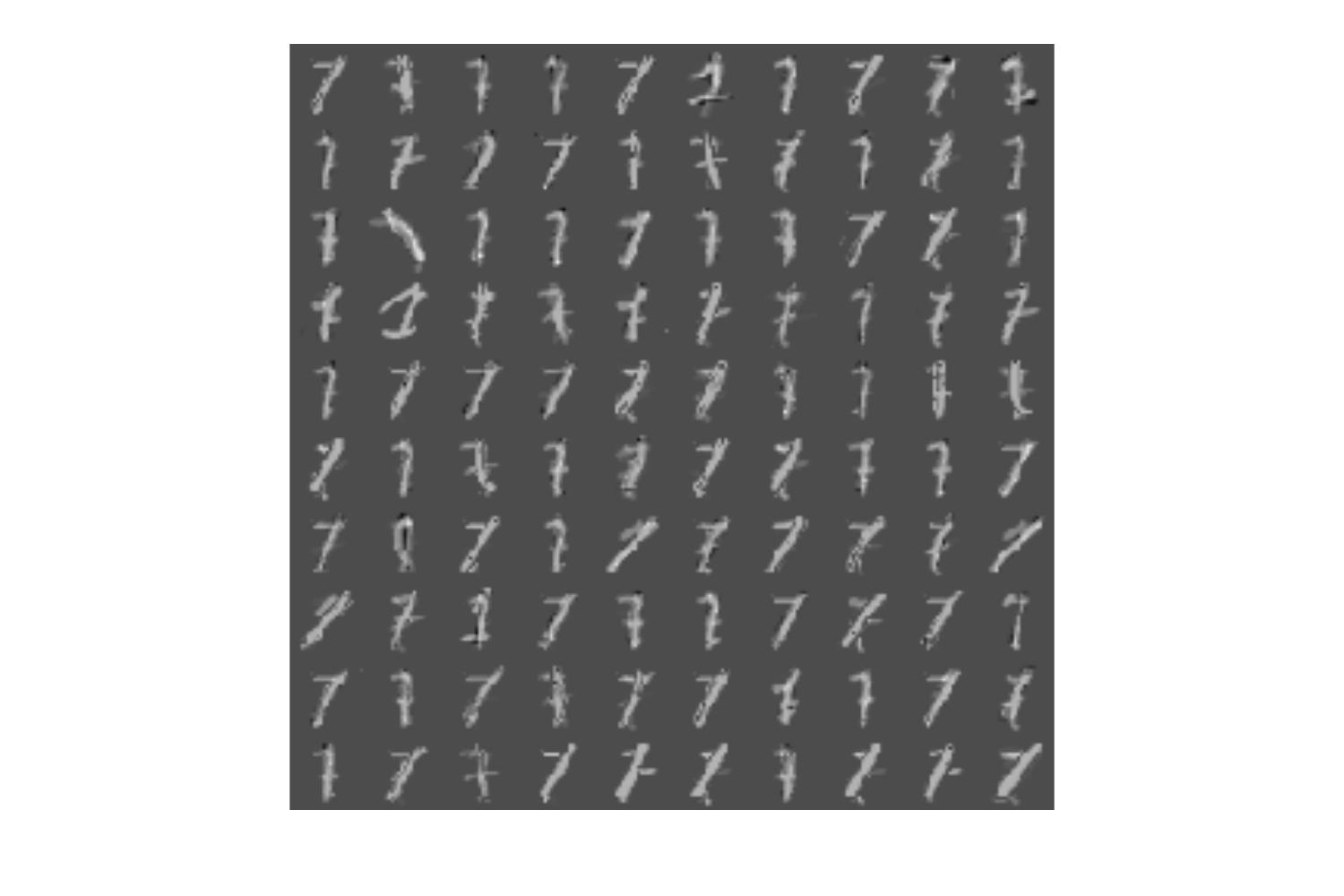}&
\hskip -1.4cm
\includegraphics[clip, trim=0cm 0cm 0cm 0cm, width=0.3\columnwidth]{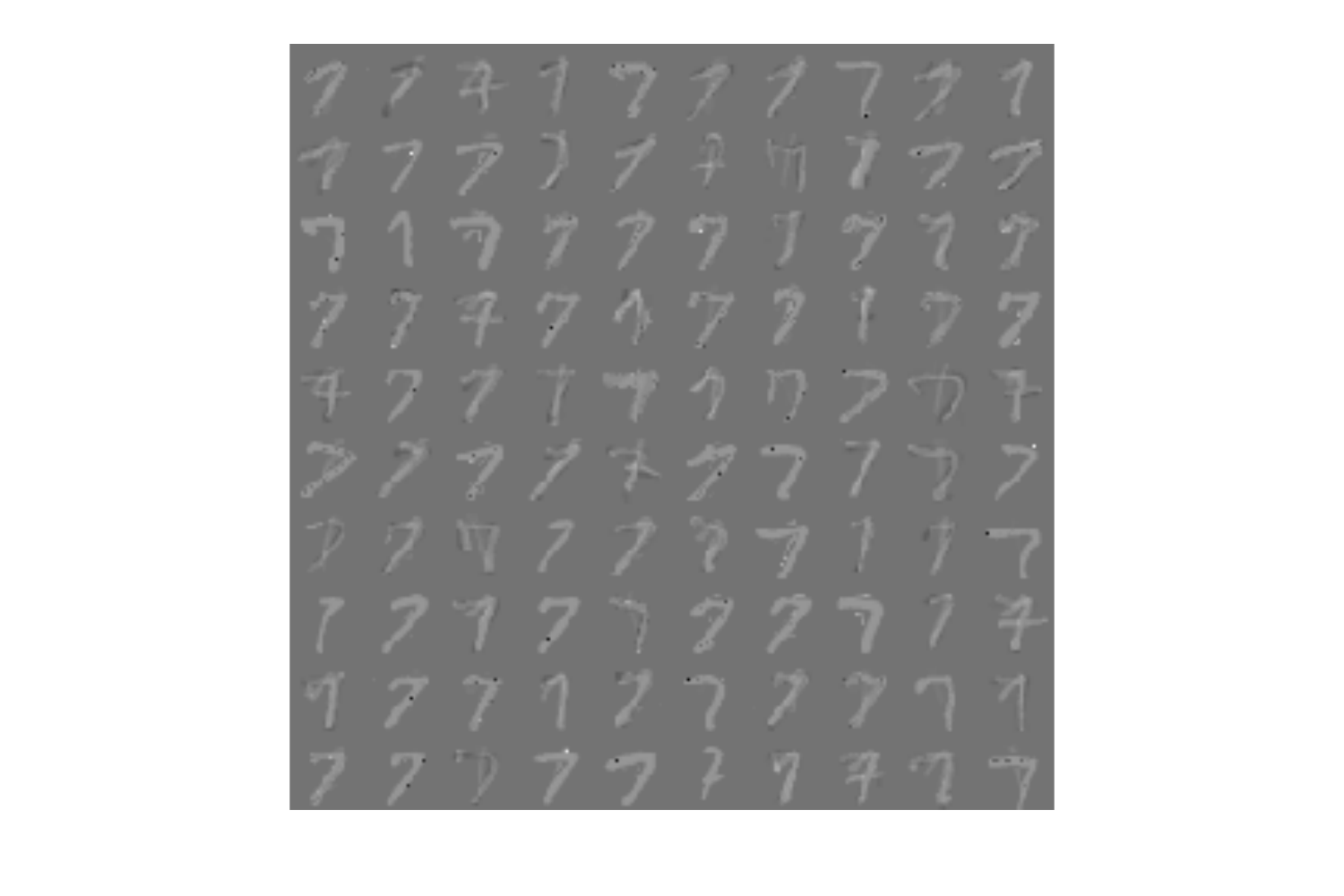}\\
\end{tabular}
\caption{Adversarial images of the MNIST generated by WB attacks ($r=4.0$). Adversarial images by both attacks are harder to classify than the clean images, and KSA attack changes the contrast of the images severely. Left (Right) two: adversarial images of digits 1 and 7 from DA (KSA).}
\label{fig:vis-MNIST:WB-Attack}
\end{figure}

\begin{figure}[!ht]
\centering
\begin{tabular}{cc}
\includegraphics[clip, trim=0cm 0cm 0cm 0cm, width=0.45\columnwidth]{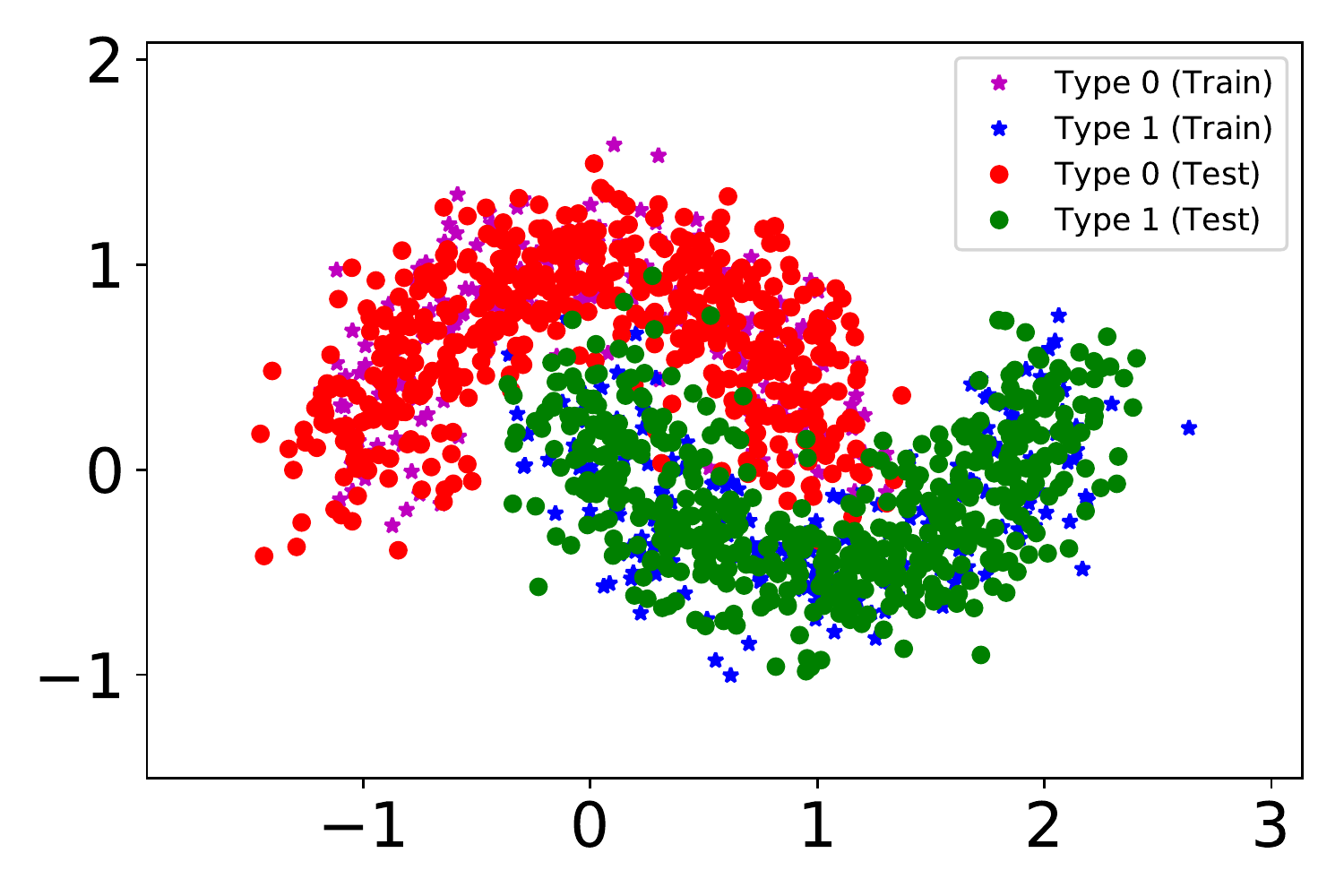}&
\includegraphics[clip, trim=0cm 0cm 0cm 0cm, width=0.45\columnwidth]{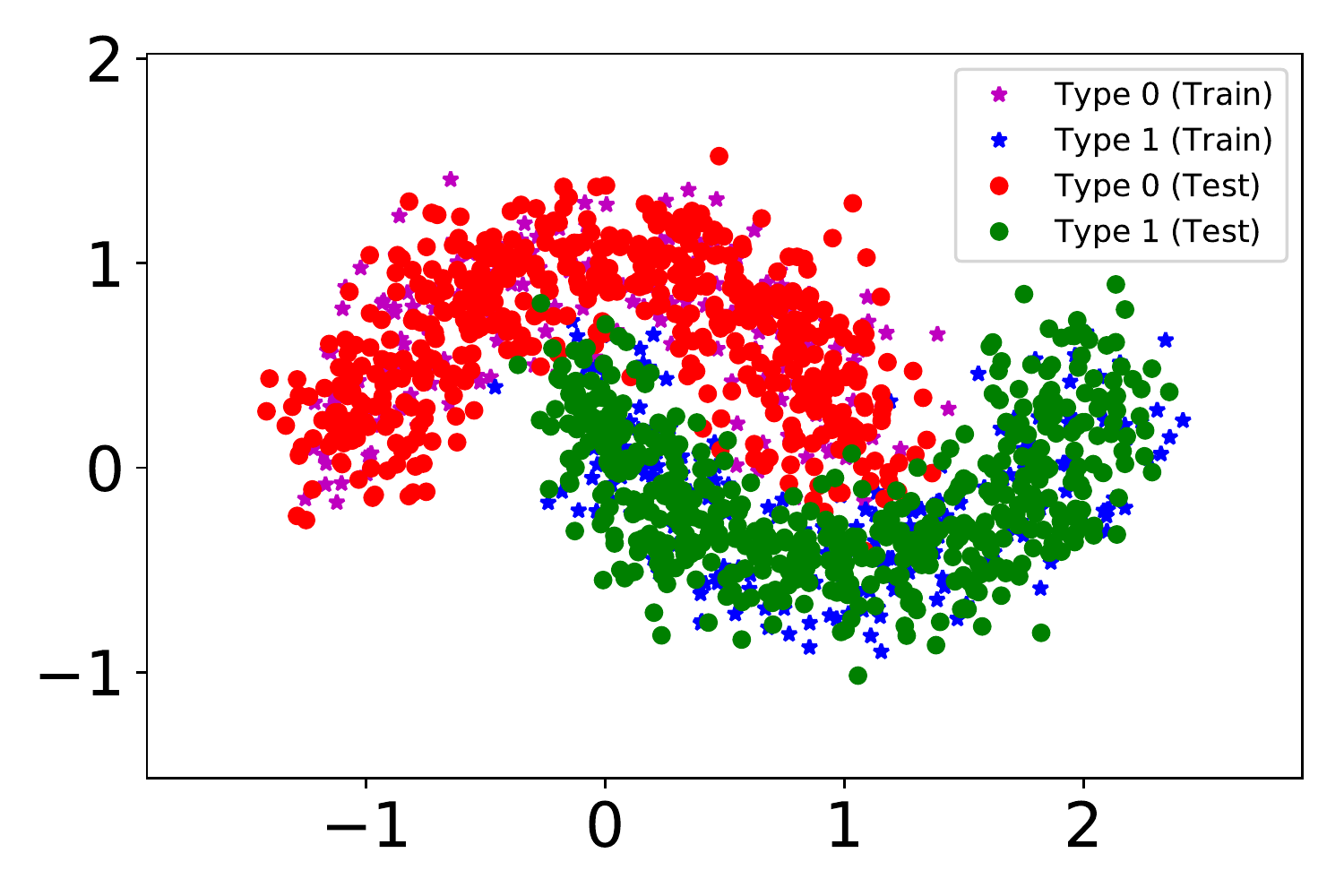}\\
\end{tabular}
\caption{Plots of training and adversarial points of the test data for the Halfmoon dataset. The adversarial points are generated by WB attacks ($r=0.2$), which have a significant overlap with the training data of different type. Left: DA; Right: KSA. (Best viewed on a computer screen.) 
}
\label{fig:vis-Halfmoon:WB-Attack}
\end{figure}

Figures~\ref{fig:vis-MNIST:BB-Attack} and \ref{fig:vis-Halfmoon:BB-Attack} show the adversarial examples of the MNIST 1v7 and the Halfmoon datasets under the BB attacks with $r=4$ and $0.2$, respectively. For MNIST 1v7, under the BB attack with three different substitute models, namely, logistic regression (LR), neural net (NN), and the kernel classifier (Kernel), the contrast of the adversarial images has been changed dramatically from the original images. The adversarial images of digit $1$ and digit $7$ are much harder to classifier even for us compared with the clean images, especially those generated by BB attacks with LR or Kernel model substitution.

\begin{figure}[!ht]
\centering
\begin{tabular}{ccc}
\hskip -1.0cm\includegraphics[clip, trim=0cm 0cm 0cm 0cm, width=0.4\columnwidth]{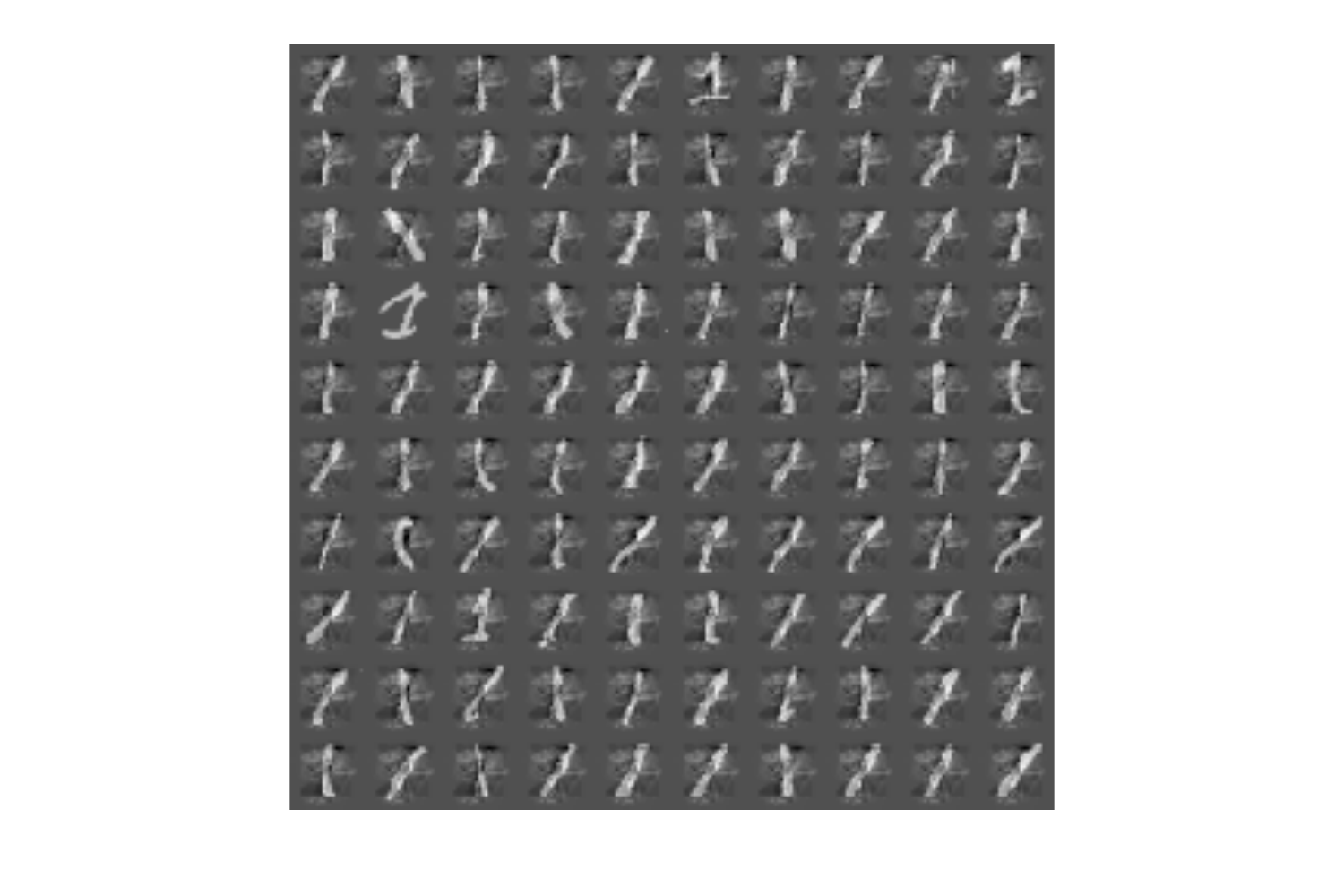}&
\hskip -1.4cm\includegraphics[clip, trim=0cm 0cm 0cm 0cm, width=0.4\columnwidth]{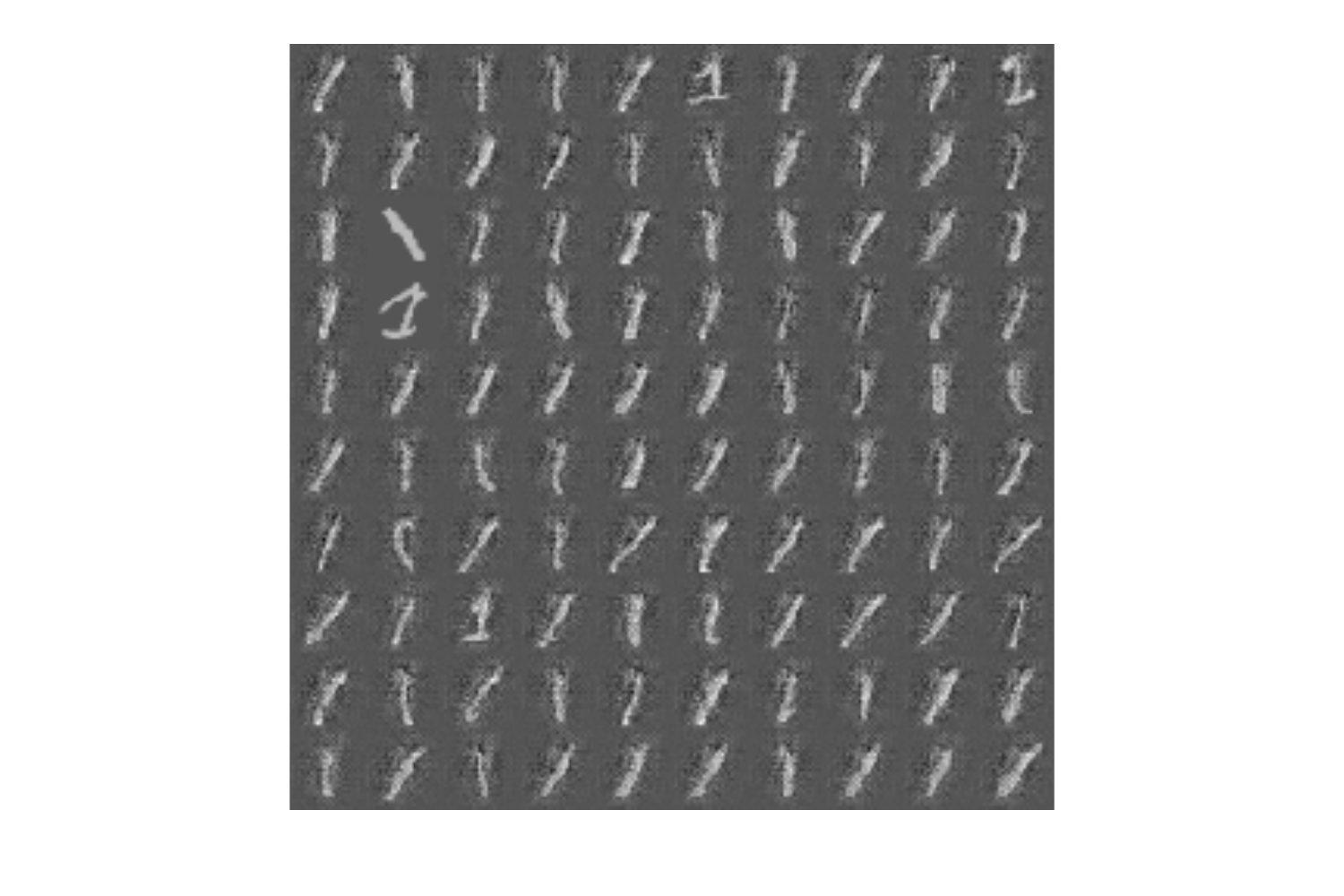}&
\hskip -1.4cm\includegraphics[clip, trim=0cm 0cm 0cm 0cm, width=0.4\columnwidth]{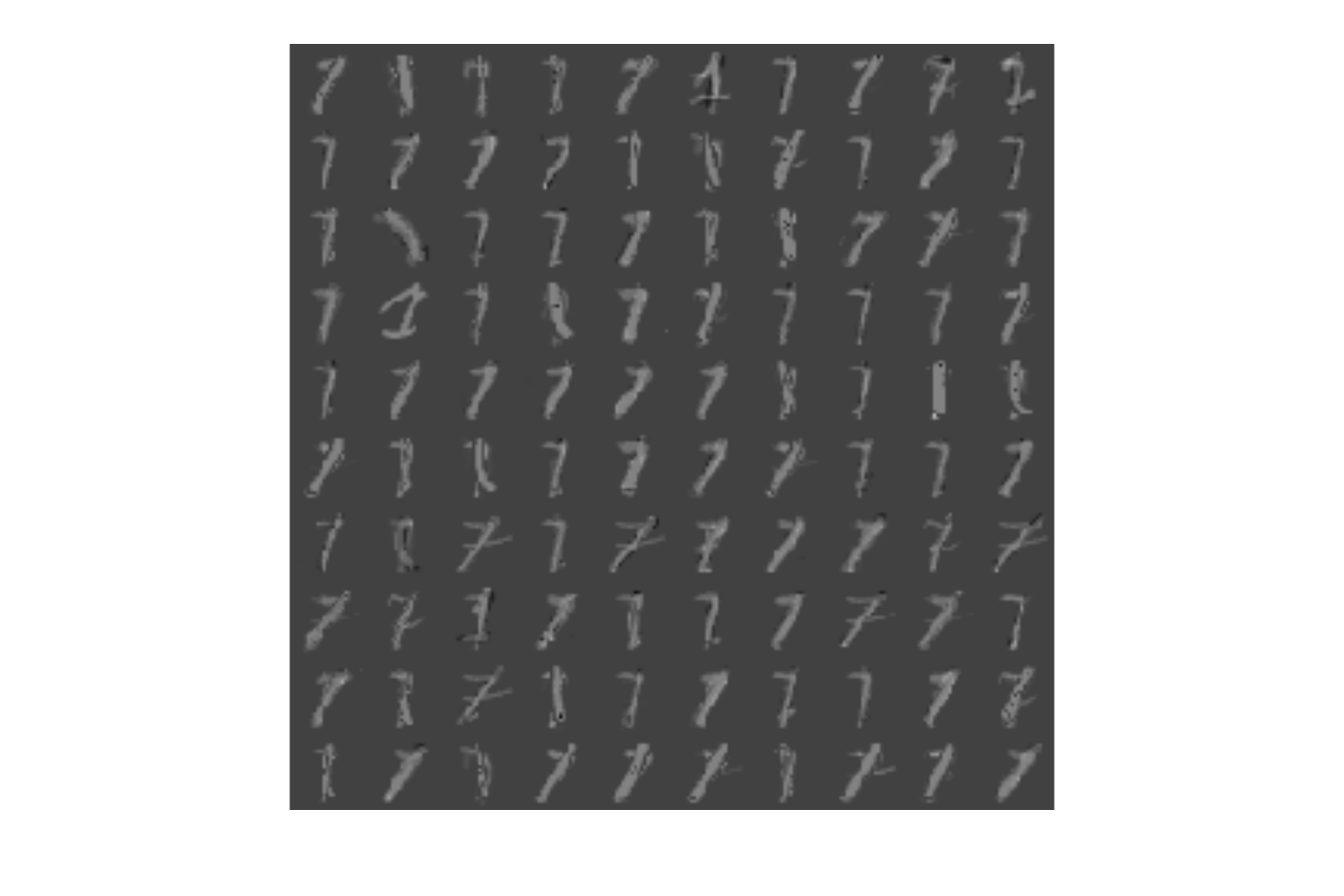}\\
\hskip -1.0cm LR (Digit $1$) & \hskip -1.3cm NN (Digit $1$) & \hskip -1.3cm Kernel (Digit $1$) \\
\hskip -1.0cm\includegraphics[clip, trim=0cm 0cm 0cm 0cm, width=0.4\columnwidth]{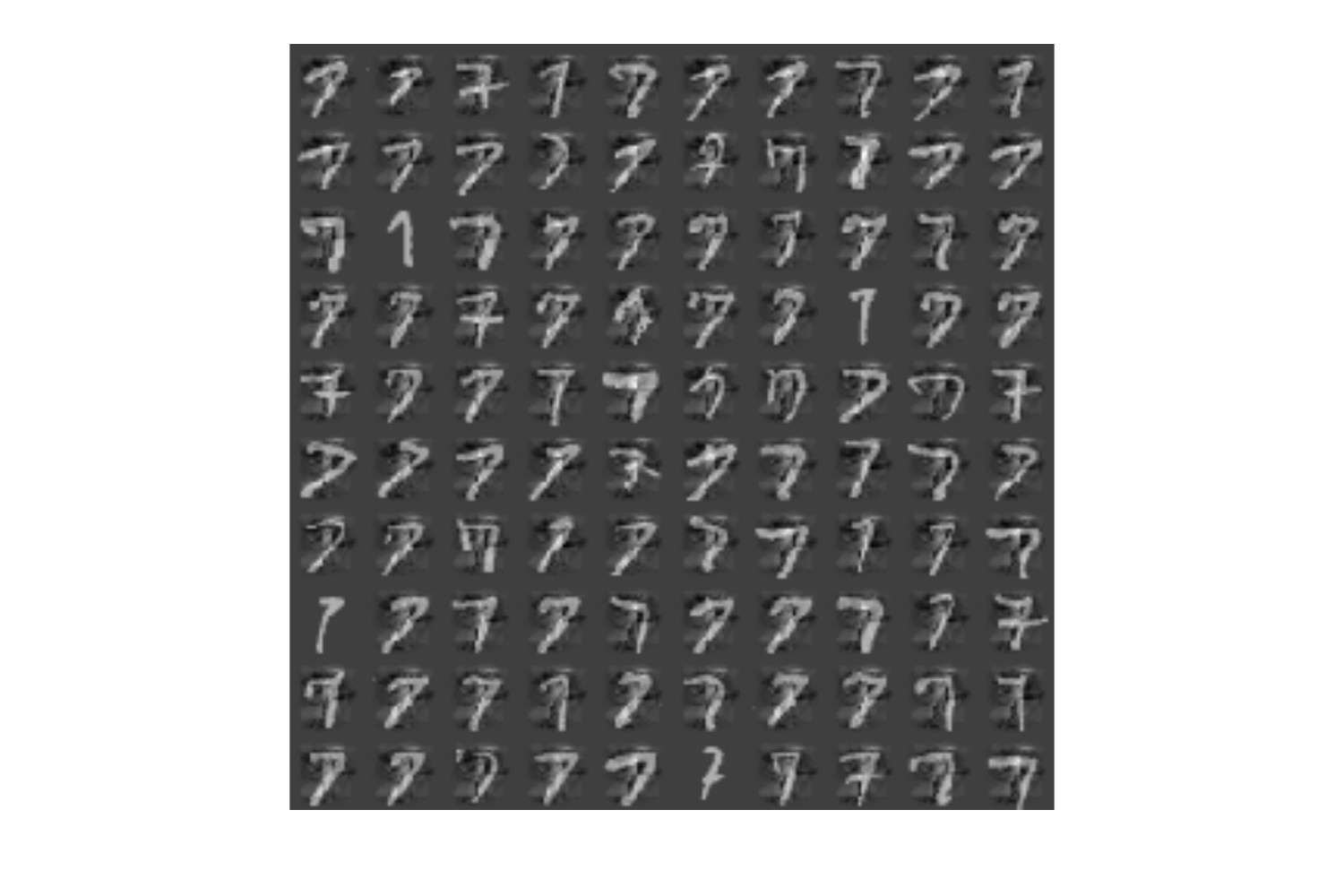}&
\hskip -1.4cm\includegraphics[clip, trim=0cm 0cm 0cm 0cm, width=0.4\columnwidth]{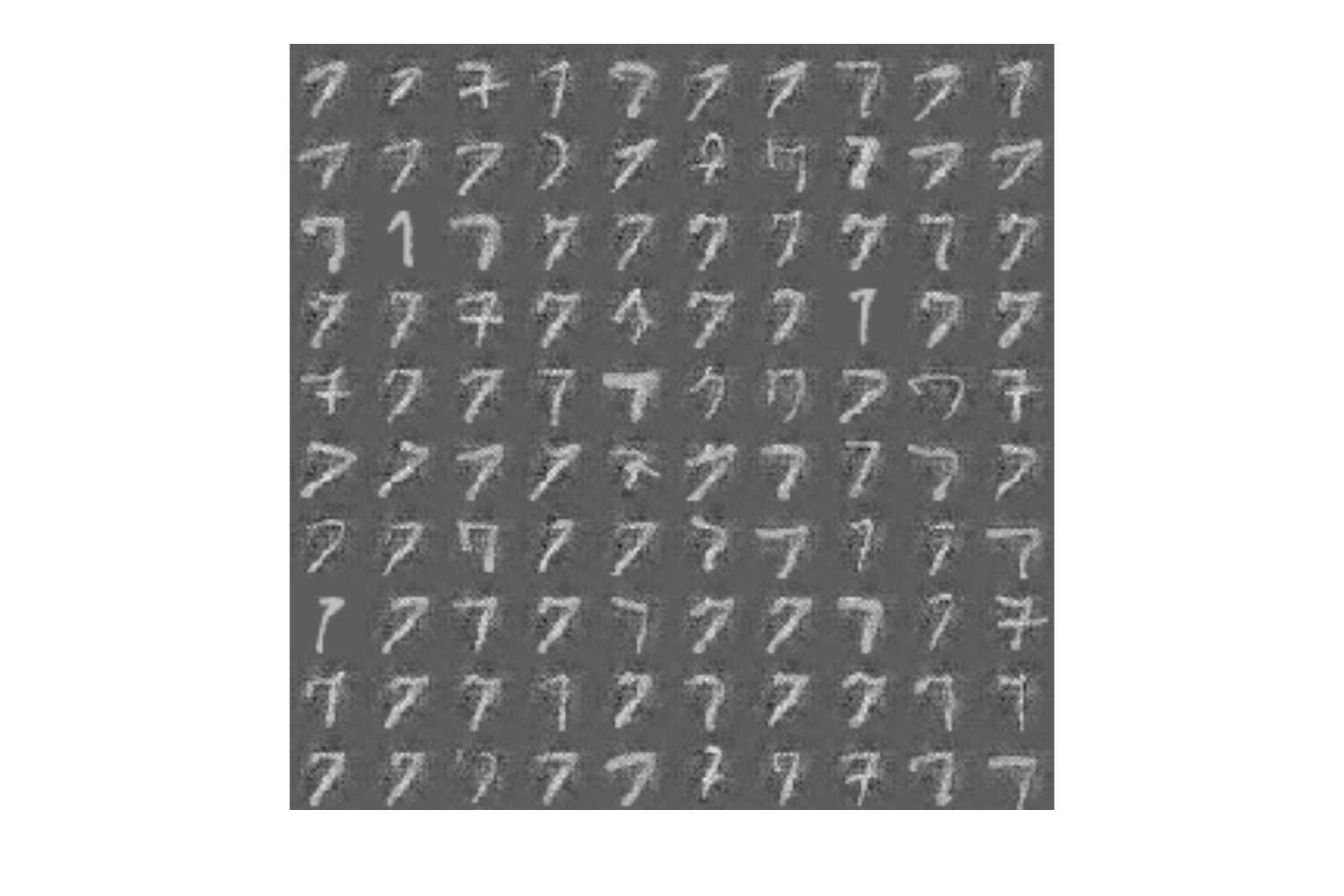}&
\hskip -1.4cm\includegraphics[clip, trim=0cm 0cm 0cm 0cm, width=0.4\columnwidth]{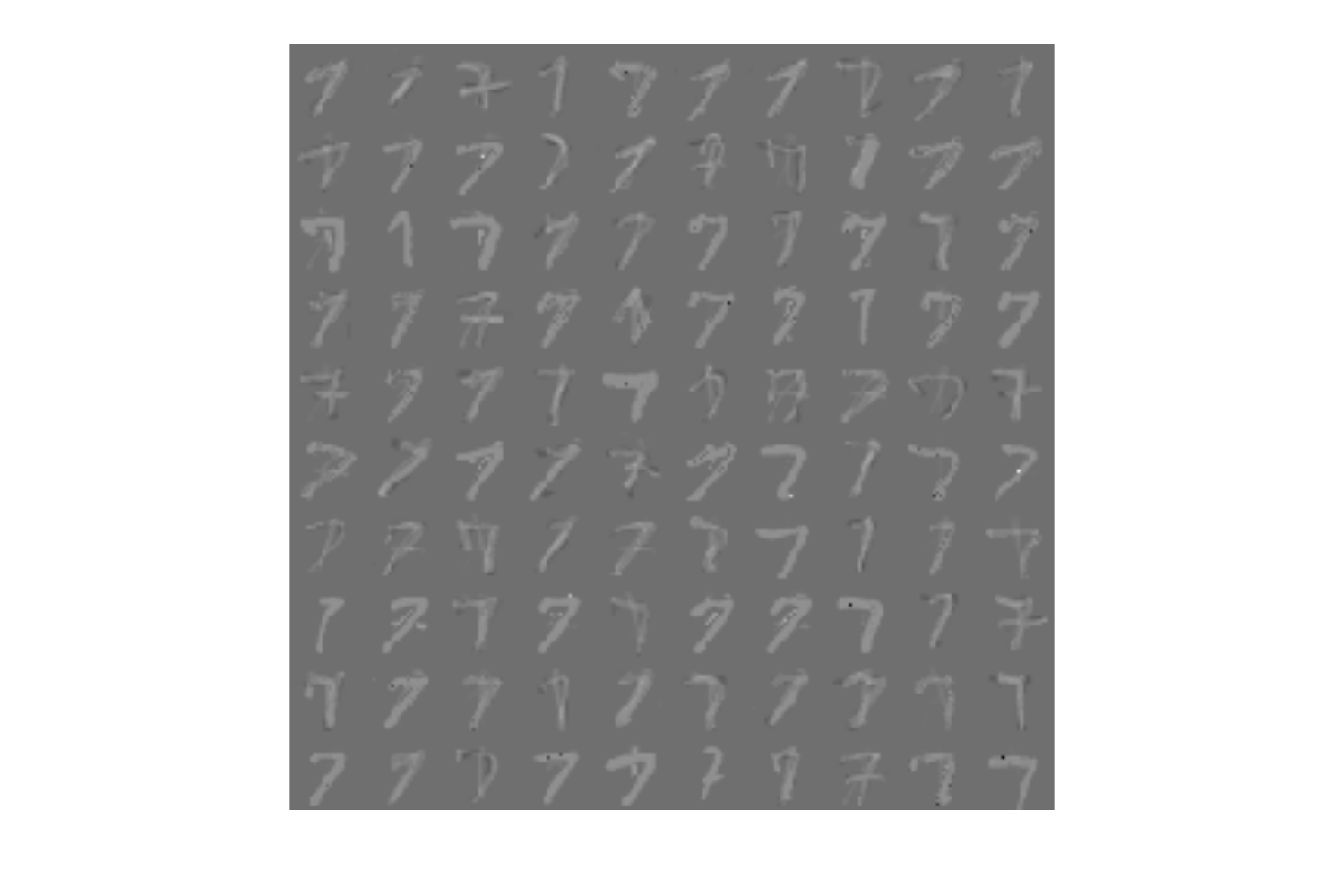}\\
\hskip -1.0cm LR (Digit $7$) & \hskip -1.3cm NN (Digit $7$) & \hskip -1.3cm Kernel (Digit $7$) \\
\end{tabular}
\caption{Adversarial images of the MNIST generated by BB attacks with $r=4.0$. Both attacks make the adversarial images harder to classifier, and change the contrast of the adversarial remarkably. (Best viewed on a computer screen.)}
\label{fig:vis-MNIST:BB-Attack}
\end{figure}

As shown in Fig.~\ref{fig:vis-Halfmoon:BB-Attack}, for the Halfmoon dataset, with a maxium perturbation $r=0.2$ the adversarial version of the test set of Type $0$ will have a significant overlap with the training set of Type $1$, and vice versa. This will degrade the classification accuracy GL and $k$NN classifiers severely, as shown in Figures~\ref{fig:acc:lr}-\ref{fig:acc:nn}.

\begin{figure}[!ht]
\centering
\begin{tabular}{ccc}
\hskip -0.8cm\includegraphics[clip, trim=0cm 0cm 0cm 0cm, width=0.33\columnwidth]{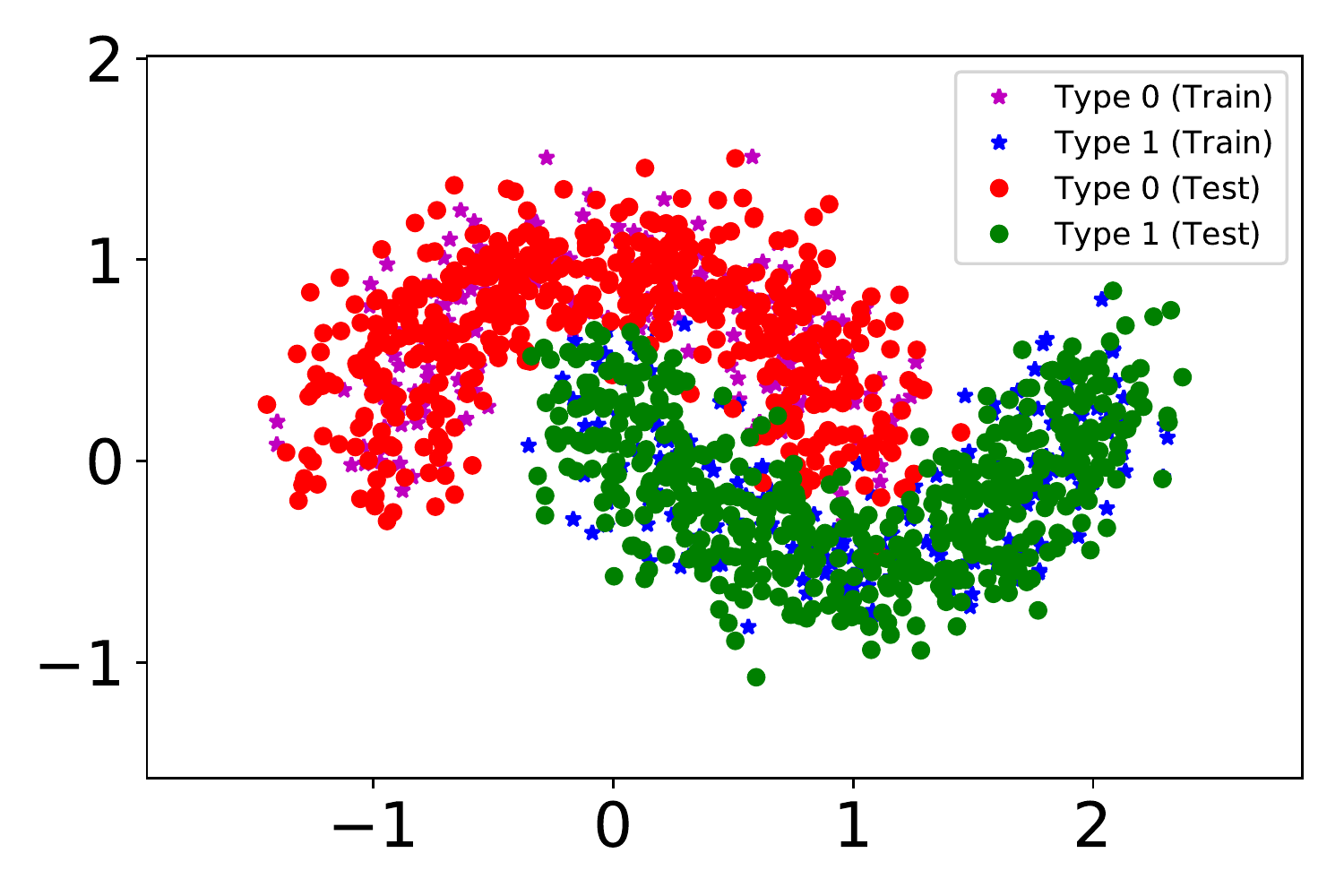}&
\includegraphics[clip, trim=0cm 0cm 0cm 0cm, width=0.33\columnwidth]{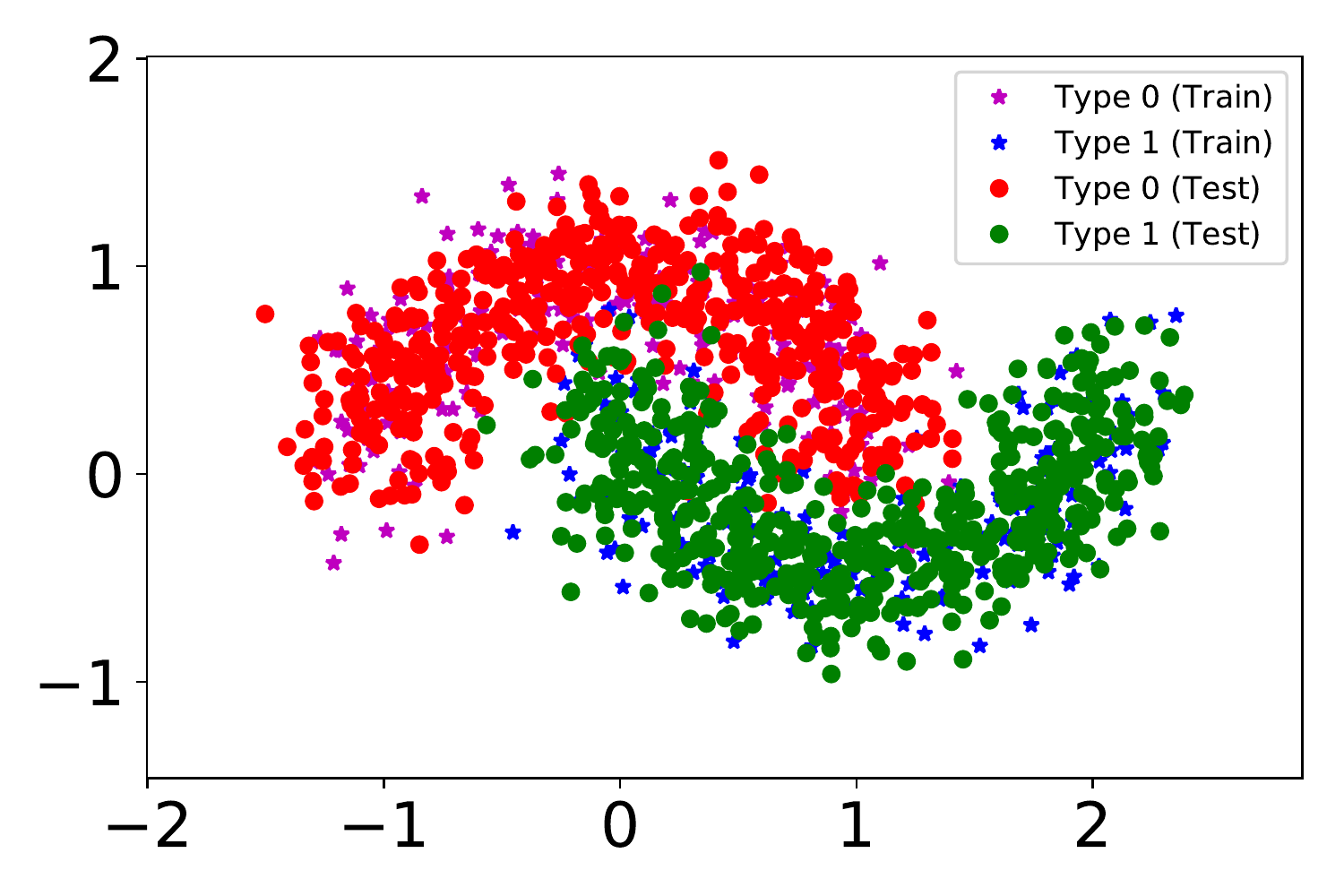}&
\includegraphics[clip, trim=0cm 0cm 0cm 0cm, width=0.33\columnwidth]{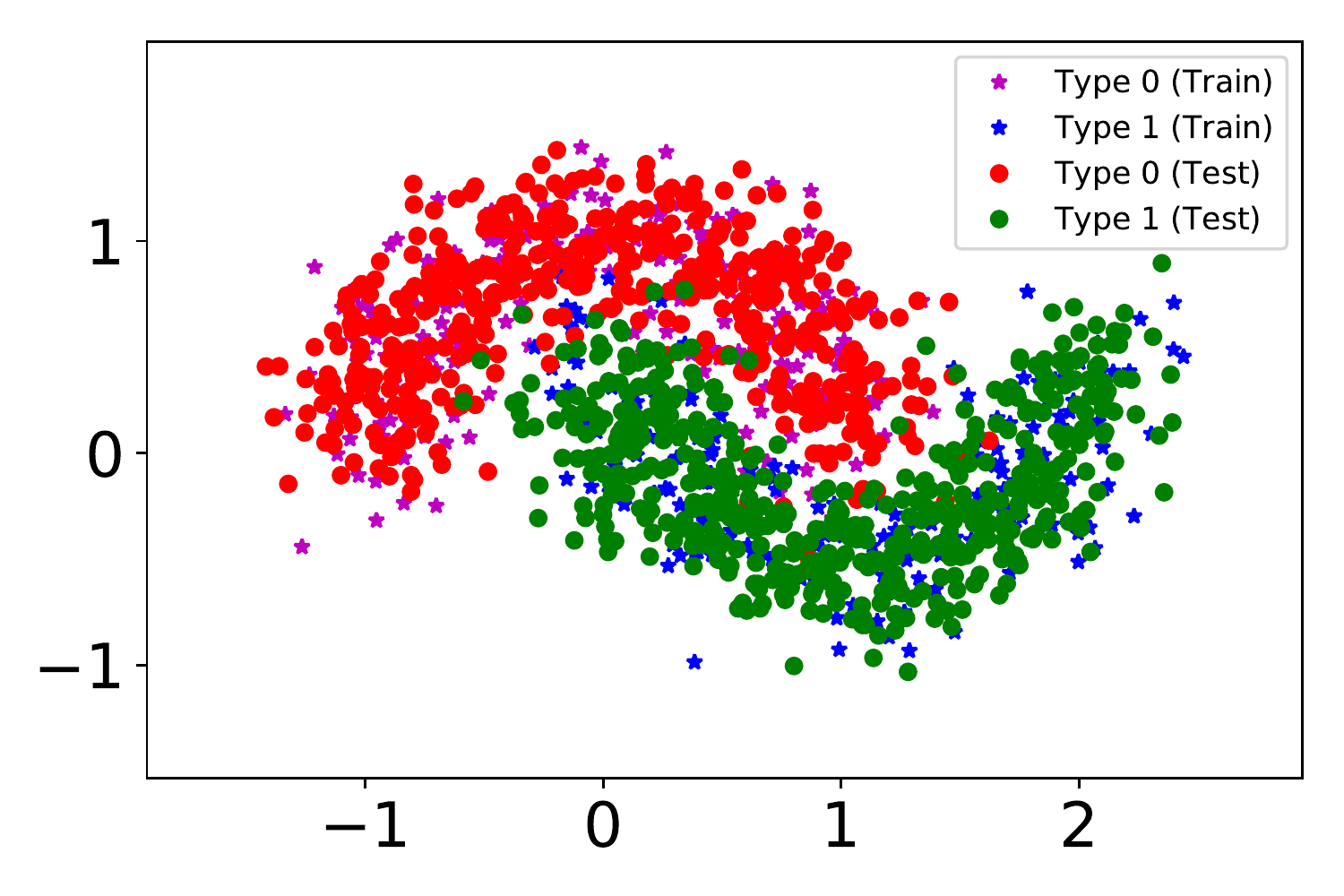}\\
\hskip -0.8cm LR & NN & Kernel\\
\end{tabular}
\caption{Plots of training and adversarial points of the test data for the Halfmoon dataset. The adversarial points are generated by BB attacks $r=0.2$. The adversarial points have a significant overlap with the training data of different type. (Best viewed on a computer screen.)}
\label{fig:vis-Halfmoon:BB-Attack}
\end{figure}

\end{document}